\documentclass[11pt]{article}

\usepackage[margin=1in]{geometry}

\usepackage{lipsum}

\usepackage[utf8]{inputenc} 
\usepackage[T1]{fontenc}    
\usepackage{tgtermes}

\usepackage{url}            
\usepackage{booktabs}       
\usepackage{amsfonts}       
\usepackage{nicefrac}       
\usepackage{microtype}      

\usepackage[english]{babel}

\usepackage{amsmath}
\usepackage{enumitem}
\usepackage{amssymb}
\usepackage{graphicx}
\usepackage{bm}
\usepackage{theorem}
\usepackage[colorlinks=true, allcolors=blue]{hyperref}

\usepackage{mathtools}

\usepackage[title]{appendix}
\usepackage{comment}
\usepackage{amsfonts}
\usepackage{dsfont}
\usepackage{hyperref}
\usepackage{lipsum}
\usepackage{dsfont,subcaption,float}
\usepackage{algorithm}
\usepackage{algorithmic}

\usepackage{thm-restate}
\usepackage[size=small,labelfont=bf]{caption}

\usepackage{xcolor}

\newtheorem{theorem}{Theorem}[section]
\newtheorem{example}{Example}
\newtheorem{condition}{Condition}
\newtheorem{corollary}{Corollary}
\newtheorem{lemma}[theorem]{Lemma}
\newtheorem{remark}{Remark}
\newtheorem{proposition}[theorem]{Proposition}
\newtheorem{assumption}{Assumption}
\newtheorem{definition}{Definition}
\newtheorem{question}{Open Question}
\newenvironment{proof}{\paragraph{\it Proof.}}{\hfill$\square$}

\newcommand{\mC}{\mathcal{C}}	
\newcommand{\mS}{\mathcal{S}}	
\newcommand{\mA}{\mathcal{A}}	
\newcommand{\mO}{\mathcal{O}}	
\newcommand{\mH}{\mathcal{H}}	
\newcommand{\mT}{\mathcal{T}}	
\newcommand{\mM}{\mathcal{M}}	
\newcommand{\tmM}{\widetilde{\mathcal{M}}}	

\newcommand{\mX}{\mathcal{X}}	
\newcommand{\mF}{\mathcal{F}}	
\newcommand{\mG}{\mathcal{G}}	
\newcommand{\mR}{\mathbb{R}}	
\newcommand{\mD}{\mathcal{D}}	

\newcommand{\tPP}{\tilde{\mathds{P}}}    
\newcommand{\PP}{\mathds{P}}    

\newcommand{\Eps}{\mathcal{E}}

\newcommand{\Exs}{\mathbb{E}}

\allowdisplaybreaks

\newif\ifdraft
\drafttrue  

\title{\bf{\LARGE{RL for Latent MDPs: Regret Guarantees and a Lower Bound}}}

\usepackage{authblk}
\author[1]{Jeongyeol Kwon}
\author[2]{Yonathan Efroni} 
\author[1]{Constantine Caramanis}
\author[3]{Shie Mannor}
\affil[1]{Department of Electrical and Computer Engineering, University of Texas at Austin}
\affil[2]{Microsoft Research, New York}
\affil[3]{Department of Electrical Engineering, Technion}

\begin{document}
\maketitle

\begin{abstract}
    In this work, we consider the regret minimization problem for reinforcement learning in latent Markov Decision Processes (LMDP). In an LMDP, an MDP is randomly drawn from a set of $M$ possible MDPs at the beginning of the interaction, but the identity of the chosen MDP is not revealed to the agent. We first show that a general instance of LMDPs requires at least $\Omega((SA)^M)$ episodes to even approximate the optimal policy. Then, we consider sufficient assumptions under which learning good policies requires polynomial number of episodes. We show that the key link is a notion of separation between the MDP system dynamics. With sufficient separation, we provide an efficient algorithm with local guarantee, {\it i.e.,} providing a sublinear regret guarantee when we are given a good initialization. Finally, if we are given standard statistical sufficiency assumptions common in the Predictive State Representation (PSR) literature (e.g., \cite{boots2011online}) and a reachability assumption, we show that the need for initialization can be removed. 
\end{abstract}

\section{Introduction}
Reinforcement Learning (RL) \cite{sutton2018reinforcement} is a central problem in artificial intelligence which tackles the problem of sequential decision making in an unknown dynamic environment. The agent interacts with the environment by receiving the feedback on its actions in the form of a state-dependent reward and new observation. The goal of the agent is to find a policy that maximizes the long-term reward from the interaction. 

Partially observable Markov decision processes (POMDPs) \cite{smallwood1973optimal} give a general framework to describe partially observable sequential decision problems. In POMDPs, the underlying dynamics satisfy the Markovian property, but the observations give only partial information on the identity of the underlying states. With the generality of this framework comes a high computational and statistical price to pay: POMDPs are hard, primarily because optimal policies depend on the entire history of the process. But for many important problems, this full generality can be overkill, and in particular, does not have a way to leverage special structure. We are interested in settings where the hidden or latent (unobserved) variables have slow dynamics or are even static in each episode. This model is important for diverse applications, from serving a user in a dynamic web application \cite{hallak2015contextual}, to medical decision making \cite{steimle2018multi}, to transfer learning in different RL tasks \cite{brunskill2013sample}. 
Yet, as we explain below, even this area remains little understood, and challenges abound.

Thus, in this work, we consider reinforcement learning (RL) for a special type of POMDP which we call a latent Markov decision process (LMDPs). LMDPs consist of some (perhaps large) number $M$ of MDPs with joint state space $\mS$ and actions $\mA$. In episodic LMDPs with finite time-horizon $H$, the static latent (hidden) variable that selects one of $M$ MDPs is randomly chosen at the beginning of each episode, yet is not revealed to the agent. The agent then interacts with the chosen MDP throughout the episode (see Definition \ref{definition:lmdp} for the formal description). 

The LMDP framework has previously been introduced under many different names, {\it e.g.,} hidden-model MDP \cite{chades2012momdps}, Multitask RL \cite{brunskill2013sample}, Contextual MDP \cite{hallak2015contextual}, Multi-modal Markov decision process \cite{steimle2018multi} and Concurrent MDP \cite{buchholz2019computation}. Learning in LMDPs is a challenging problem due to the unobservability of latent contexts. For instance, the exact planning problem is P-SPACE hard \cite{steimle2018multi}, inheriting the hardness of planning from the general POMDP framework. Nevertheless, the lack of dynamics of the latent variables, offers some hope. As an example, if the number of contexts $M$ is bounded, then the planning problem can be at least approximately solved ({\it e.g.,} by point-based value iteration (PBVI) \cite{pineau2006anytime}, or mixed integer programming (MIP) \cite{steimle2018multi}). 

The most closely related work studying LMDPs is in the context of multitask RL~\cite{taylor2009transfer, brunskill2013sample, liu2016pac, hallak2015contextual}. In this line of work, a common approach is to cluster trajectories according to different contexts, an approach that guided us in designing the algorithms in Section~\ref{subsection:scenario3}. However, previous work requires very long time-horizon $H \gg SA$ in order to guarantee that every state-action pair can be visited multiple times in a single episode. In contrast, we consider a significantly shorter time-horizon that scales poly-logarithmic with the number of states, {\it i.e.,} $H = poly \log (MSA)$. This short time-horizon results in a significant difference in learning strategy even when we get a feedback on the true context at the end of episode. We refer the readers to Section \ref{sec:related_work} for additional discussion on related work.

\paragraph{Main Results.} To the best of our knowledge, none of the previous literature has obtained sample complexity guarantees or studied regret bounds in the LMDP setting. This paper addresses precisely this problem. We ask the following: 
\begin{center}
    \emph{Is there a sample efficient RL algorithm for LMDPs, with sublinear regret?}    
\end{center}
The answer turns out to be not so simple. Our results comprise a first impossibility result, followed by positive algorithmic results under additional assumptions. Specifically:
\begin{itemize}
    \item First, we find that for a general LMDP, polynomial sample complexity cannot be attained without further assumptions. That is, to find an approximately optimal policy we need at least $\Omega\left( (SA)^M \right)$ samples, {\it i.e.,} at least exponential in the number of contexts $M$ (Section \ref{subsec:lower_bound}). This lower bound even applies to instances with deterministic MDPs. 

\item We find that there are several natural assumptions under which optimal policies can be learned with polynomial sample complexity. Similarly to mixture problems without dynamics, the key link is a notion of separation between the MDPs. With sufficient separation, we show that there is a planning-oracle efficient RL algorithm with polynomial sample complexity. A critical development is adapting the principle of optimism as in UCB, but to the partially observed setting where value-iteration cannot be directly applied, and thus neither can the UCRL algorithm for MDPs. 
\item Finally, under additional statistical sufficiency assumptions that are common in the Predictive State Representation (PSR) literature (e.g., \cite{boots2011online}) and a reachability assumption, we show that the need for initialization can be entirely removed. 

\item Finally, we perform an empirical evaluation of the suggested algorithms on toy problem (Section \ref{section:experiments}), while focusing on the importance of the made assumptions. 
\end{itemize}


\subsection{Related Work}
\label{sec:related_work}
Due to the vast volume of literature on the RL, we only review the most closely related research to our problem.

\paragraph{Previous Study on LMDPs} As mentioned earlier, LMDPs have been previously introduced with different names. In the work of~\cite{chades2012momdps, steimle2018multi, buchholz2019computation}, the authors study the planning problem in LMDPs, when the true parameters of the model is given. The authors in~\cite{steimle2018multi} have shown that, as for POMDPs~\cite{papadimitriou1987complexity}, it is P-SPACE hard to find an exact optimal policy, and NP-hard to find an optimal memoryless policy of an LMDP. On the positive side, several heuristics are proposed for practical uses of finding the optimal memoryless policy \cite{steimle2018multi, buchholz2019computation}. 

LMDP has been studied in the context of multitask RL~\cite{taylor2009transfer, brunskill2013sample, liu2016pac, hallak2015contextual}. In this line of work, a common approach is to cluster trajectories according to different contexts under some separation assumption, an approach that guided us in designing the algorithms in Section~\ref{subsection:scenario3}. However, in this line of work, the authors assume very long time-horizon such that they can visit every state-action pair multiple times in a single episode. In order to satisfy such assumption, the time-horizon must be at least $H \ge \Omega(SA)$. In contrast, we consider a significantly shorter time-horizon that scales poly-logarithmic with the number of states, {\it i.e.,} $H = poly \log (MSA)$. This short time-horizon results in a significant difference in learning strategy even when we get a feedback on the true context at the end of episode. 

\paragraph{Approximate Planning in POMDPs} The study of learning in partially observable domains has a long history. Unlike in MDPs, finding the optimal policy for a POMDP is P-SPACE hard even with known parameters \cite{smallwood1973optimal}. Even finding a memoryless policy is known to be NP-hard \cite{littman1994memoryless}. Due to the computational intractability of exact planning, various approximate algorithms and heuristics within a policy class of interest \cite{jaakkola1995reinforcement, ng2000pegasus, smith2004heuristic, spaan2005perseus, pineau2006anytime, li2011finding}. Since LMDP is a special case of POMDP, any of these methods can be applied to solve LMDP. We will assume that the planning-oracle achieves some approximation guarantees with respect to maximum long-term rewards obtained by the optimal policy. We show that when the context is identifiable in hindsight, then we can quickly perform as good as the policy obtained by the planning-oracle with true parameters.

\paragraph{Spectral Methods for POMDPs} Previous studies of partially-observed decision problems assumed the number of observations is larger than the number of hidden states, as well as, that a set of single observations forms {\it sufficient statistics} to learn the hidden structure \cite{azizzadenesheli2016reinforcement, guo2016pac}. With such assumptions, one can apply tensor-decomposition methods by constructing multi-view models \cite{anandkumar2012method, anandkumar2014tensor} and recovering POMDP parameters under uniformly-ergodic (or stationary) assumption on the environment \cite{azizzadenesheli2016reinforcement, guo2016pac}. Our work is differentiated from the mentioned works in two aspects. First, for LMDPs, the observation space is {\it smaller} than the hidden space. Therefore constructing a multi-view model with a set of single observations is not enough to learn hidden structures of the system. Second, we are not aware of any natural conditions for tensor-decomposition methods to be applicable for learning LMDPs. Therefore, we do not pursue the application of tensor-methods in this work. 

\paragraph{Predictive State Representation} Since the introduction of PSR \cite{littman2002predictive, singh2004predictive}, it has become one major alternative to POMDPs for modeling partially-observed environments. The philosophy of PSR is to express the internal state only with a set of observable experiments, or {\it tests}. Various techniques have been developed for learning PSR with statistical consistency and global optimality guarantees \cite{hsu2012spectral, boots2011closing, hefny2015supervised}. We use the PSR framework to get an initial estimate of the system. However, the sample complexity of learning PSR is quite high \cite{jiang2016improving} (see also our finite-sample analysis of spectral learning technique in Section \ref{appendix:spectral_learning_psr}). Therefore, we only learn PSR up to some desired accuracy and convert it to an LMDP parameter by clustering of trajectories (Section \ref{Appendix:clustering_with_psrs}) to warm-start the optimal policy learning. 

\paragraph{Other Related Work} In a relatively well-studied setting of contextual decision processes (CDPs), the context is always given as a side information {\it at the beginning} of the episode \cite{jiang2017contextual, modi2018markov}. This makes the decision problem a fully observed decision problem. LMDP is different since the context is hidden. The main challenge comes from the partial observability which results in significant differences in terms of analysis from CDPs. Another line of work on decision making with latent contexts considers the problem of latent bandits \cite{maillard2014latent, gentile2014online, gentile2017context}. It would be interesting to understand whether any previous results on latent bandits can be extended to latent MDPs. Another line of research on theoretical studies with partially observability considers the environment with rich observations \cite{krishnamurthy2016pac, dann2018oracle, du2019provably}. Rich observation setting assumes that any observation happens from only one internal state which removes the necessity to consider histories, and thus the nature of the problem is different from our setting. Recent work considers a sample-efficient algorithm for undercomplete POMDPs \cite{jin2020sample}, {\it i.e.,} when the observation space is larger than the hidden state space and a set of single observations is statistically sufficient to learn hidden structures. In contrast, our problem is a special case of POMDPs where the observation space is smaller than the hidden state space.

\section{Preliminaries}
In Section \ref{subsec:setup_latentmdp} we define LMDPs precisely, as well as the notion of regret relevant to this paper. Then in Section \ref{subsec:review_psr} we discuss predictive state representations, which we need for our results in Section \ref{subsection:scenario3}.

\subsection{Problem Setup: Latent MDPs}
\label{subsec:setup_latentmdp}
We consider the following Latent Markov decision processes (LMDP) in an episodic reinforcement learning with finite-time horizon $H$. 
\begin{definition}[Latent Markov Decision Process (LMDP)]
    \label{definition:lmdp}
    Suppose the set of MDPs $\mathcal{M}$ with joint state space $\mathcal{S}$ and joint action space $\mathcal{A}$ in a finite time horizon $H$. Let $M = |\mathcal{M}|$, $S = |\mS|$ and $A = |\mA|$. Each MDP $\mM_m \in \mathcal{M}$ is a tuple $(\mathcal{S}, \mathcal{A}, T_m, R_m, \nu_m)$ where $T_m: \mathcal{S}\times\mathcal{A}\times\mathcal{S} \rightarrow [0,1]$ a transition probability measures that maps a state-action pair and a next state to a probability, $R_m: \mathcal{S}\times\mathcal{A}\times \{0,1\} \rightarrow [0, 1]$ a probability measure for rewards that maps a state-action pair and a binary reward to a probability, and $\nu_m$ is an initial state distribution. Let $w_1, ..., w_M$ be the mixing weights of LMDPs such that at the start of every episode, one MDP $\mM_m \in \mathcal{M}$ is randomly chosen with probability $w_m$. 
\end{definition}
We assume for simplicity that one of MDPs is uniformly chosen at the start of each episode, {\it i.e.,} $w_1 = ... = w_M = 1/M$. The goal of the problem is to find a (possibly non-Markovian) policy $\pi$ in a policy class $\Pi$ that maximizes the expected return:
\begin{align*}
    V_{\mM}^* := \max_{\pi \in \Pi} \sum_{m=1}^M w_m \Exs_{m}^\pi \left[ \sum_{t=1}^H r_t \right],
\end{align*}
where $\Exs_{m}^\pi [\cdot]$ is expectation taken over the $m^{th}$ MDP with a policy $\pi$. The policy $\pi: (\mS, \mA, \{0,1\})^* \times \mS \rightarrow \mA$ maps a history to an action. When parameters of the model $\mM$ is given, we use sufficient statistics of histories, a.k.a. belief states, using the following recursive formulation:
\begin{align*}
    b_1(m) &= \frac{w_m \nu_m (s_1)}{\sum_m w_m \nu_m(s_1)}, \\
    b_{t+1}(m) &= \frac{b_{t}(m) T_m (s_{t+1} | s_t,a_t) R_m(r_t| s_t, a_t)}{\sum_m b_{t}(m) T_m (s_{t+1} | s_t,a_t) R_m(r_t|s_t, a_t)}.
\end{align*}

We define the notion of regret in this work. Suppose we have a planning-oracle that guarantees the following approximation guarantee:
\begin{align}
    \label{eq:approximate_cond}
    V_{\mM}^\pi \ge \rho_1 V_{\mM}^* - \rho_2,
\end{align}
where $\pi$ is a returned policy when $\mM$ is given to the planning-oracle, and $\rho_1, \rho_2$ are multiplicative and additive approximation constants respectively such that $0 < \rho_1 \le 1, 0 \le \rho_2$. 
\begin{example}
    Point-based value-iteration (PBVI) \cite{pineau2006anytime} with discretization level $\epsilon_d$ in the belief space over MDPs returns $\epsilon_d H^2$ additive approximate policy. That is, the equation \eqref{eq:approximate_cond} is satisfied with $\rho_1 = 1$ and $\rho_2 = \epsilon_d H^2$ with the policy class $\Pi$ a set of all possible history-dependent policies. Running-time of PBVI algorithm is $O\left( \epsilon_d^{-O(M)} HSA \right)$. 
\end{example}
\begin{example} 
    The optimal memoryless policy gives at least $1/M$-approximation to the optimal history-dependent policy, i.e., $\rho_1 = 1/M$ and $\rho_2 = 0$ with the policy class $\Pi$ a set of all possible history-dependent policies. We may restrict $\Pi$ to deterministic memoryless policies, {\it i.e.,} $\Pi$ is the class of policies which are mappings from state to actions. The mixed-integer programming (MIP) formulation in \cite{steimle2018multi} may be used to find the optimal memoryless policy. In this case, we have $\rho_1 = 1, \rho_2 = 0$.  
\end{example}
We define the regret as the comparison to the best approximation guarantee that the planning-oracle can achieve:
\begin{align}
    \label{eq:regret_definition}
    Regret(K) = \sum_{k=1}^K (\rho_1 V_{\mM}^* - \rho_2) - V_{\mM}^{\pi_k},
\end{align}
where $\pi_k$ is a policy executed in the $k^{th}$ episode.

\subsection{Predictive State Representation (PSR)}
In general, a partially observable dynamical system can be viewed as a model that generates a sequence of observations from observation space $\mathcal{O}$ with (controlled) actions from action space $\mathcal{A}$. A predictive state representation (PSR) is a compact description of a dynamical system with a set of observable experiments, or {\it tests} \cite{singh2004predictive}. Specifically, let a test of length $t$ is a sequence of action-observation pairs given as $\tau = a_1^{\tau} o_1^{\tau} o_2^{\tau} ... a_t^{\tau} o_t^{\tau}$.  A {\it history} $h = a_1^{h} o_1^{h} a_2^h o_2^{h} ... a_t^h o_t^{h}$ is a sequence of action-observation pairs that has been generated prior to a given time. A {\it prediction} $\PP(\tau | h) = \PP(o_{1:t}^{\tau} | h || \bm{do} \ a_{1:t}^\tau)$ denotes the probability of seeing the test sequence from a given history, given that we intervene to take actions $a_1^{\tau} a_2^{\tau} ... a_t^\tau$. In latent MDPs, the observation space can be considered as a pair of next-states and rewards, i.e., $\mathcal{O} = \mS \times \{0,1\}$ and $o_t = (s_{t+1}, r_t)$. 

As we work with a special class of POMDPs, we customize the PSR formulation for LMDPs. The set of histories consists of a subset of histories that ends with different states, i.e., $\mH = \bigcup_s \mH_s$, where each element $h \in \mH_s$ is a short sequence of state-action-rewards of length at most $l = O (1)$ and ends with state $s$: 
\begin{align*}
    h = s_1^h a_1^h r_1^h s_2^h ... s_{l-1}^h a_{l-1}^h r_{l-1}^h s = (s,a,r)_{1:l-1}^h s.
\end{align*}
We define $\PP_m^{\pi} (\mH_s)$ a vector of probability where each coordinate is a probability of sampling each history in $\mH_s$ in $m^{th}$ MDP with a policy $\pi$. Likewise, each element in tests $\tau \in \mT$ is a short sequence of action-reward-next states of length at most $l$:
\begin{align*}
    \tau = a_1^\tau r_1^\tau s_2^\tau ... a_l^\tau r_l^\tau s_{l+1}^\tau = (a,r,s')_{1:l}^{\tau}. 
\end{align*}
We denote $\PP_m(\mT|s)$ as a vector of probability where each coordinate is a success probability of each test in $m^{th}$ MDP starting from a state $s$. That is,
\begin{align*}
    \PP_m(\mT|s)_i = \PP_m(\tau_i |s) = \PP_m(r_1^{\tau_i} s_2^{\tau_i} ... r_l^{\tau_i} s_{l+1}^{\tau_i} | s || \bm{do} \ a_1^{\tau_i} ... a_l^{\tau_i}).
\end{align*}

\subsubsection{Spectral Learning of PSRs in LMDPs}
\label{subsec:review_psr}
In spectral learning, we build a set of observable matrices that contains the (joint) probabilities of histories and tests, and then we can extract parameters from these matrices by performing singular value decomposition (SVD) and regressions \cite{boots2011closing}. In order to apply spectral learning techniques, we need the following technical conditions on {\it statistical sufficiency} of histories and tests. Specifically, the first condition is a rank degeneracy condition on sufficient tests:
\begin{condition}[Full-Rank Condition for Tests]
    \label{condition:rank_test}
    For all $s \in \mS$, for the test set $\mT$ that starts from state $s$, let $L_s = [\PP_1(\mT | s) | \PP_2(\mT | s) | ... | \PP_M(\mT | s)]$. Then $\sigma_{M} (L_s) \ge \sigma_\tau$ for all $s \in \mS$ with some $\sigma_\tau > 0$. 
\end{condition}
Another technical condition for spectral learning method is a rank non-degeneracy condition for sufficient histories:
\begin{condition}[Full-Rank Condition for Histories]
    \label{condition:rank_history}
    For all $s \in \mS$, for the history set $\mH_s$ that ends with state $s$ with a sampling policy $\pi$, let $H_s = [\PP_1^{\pi}(\mH_s) | \PP_2^{\pi}(\mH_s) | ... | \PP_M^{\pi} (\mH_s)]^\top$. Then $\sigma_{M} (L_s H_s) \ge \PP^\pi(\text{end state} = s) \cdot \sigma_h$ for all $s \in \mS$ with some $\sigma_h > 0$. 
\end{condition}
Here $\PP^\pi(\text{end state} = s)$ is a probability of sampling a history ending with $s$. Condition \ref{condition:rank_history} implies that we can sample a set of short trajectories with some given sampling policy $\pi$. Along with the rank condition for tests, pairs of histories and tests can be thought as many short snap-shots of long trajectories obtained by external experts or some exploration policy (e.g., random policy in uniformly ergodic MDPs).

Conditions \ref{condition:rank_test} and \ref{condition:rank_history} ensure that a set of tests and histories are statistically sufficient \cite{boots2011closing}. Following the notations in \cite{boots2011closing}, let $P_{\mT, \mH_s} = L_s H_s$ and $P_{\mT, (s',r) a, \mH_s} = L_{s'} D_{(s',r),a,s} H_s$ be empirical counterparts respectively, where $D_{(s',r),a,s} = diag(\PP_1(s',r|a,s), ..., \PP_M(s',r|a,s))$. Since these matrices can be estimated from observable sequences, we can construct the empirical matrices and apply the spectral learning methods. The goal of spectral learning algorithm is to output PSR parameters $(b_{1,s}, B_{(s',r), a, s}, b_{\infty,s})$ which are used to compute $\hat{\PP}^\pi(\tau | h)$, the estimated probability of any future observations (or tests $\tau$) conditioned any histories $h$ with some sampling policy $\pi$. We describe a detailed spectral learning procedure in Appendix \ref{appendix:spectral_learning_psr}.

\subsection{Notation}
We denote the underlying LMDP with true parameters as $\mM^*$. With slight abuse in notations, we denote $l_1$ distance between two probability distributions $\mD_1$ and $\mD_2$ on a random variable $X$ conditioned on event $E$ as
\begin{align*}
    \| (\PP_{X \sim \mD_1} - \PP_{X \sim \mD_2}) (X | E)\|_1 = \sum_{X \in \mX} |\PP_{X \sim \mD_1} (X | E) - \PP_{X \sim \mD_2} (X | E)|,
\end{align*}
where $\mX$ is a support of $X$. When we do not condition on any event, we omit the conditioning on $E$. When we measure a transition or reward probability at a state-action pair $(s,a)$, we use $T$ or $R$ instead of $\PP$. We refer $\PP_m$ the probability of any event measured in the $m^{th}$ context (or in $m^{th}$ MDP). In particular, $\PP_m(s',r|s,a) = T_m(s'|s,a) R_m(r|s,a)$. For instance, $l_1$ distance between transition probabilities to $s'$ at $(s,a)$ in $m_1^{th}$ and $m_2^{th}$ MDP will be denoted as
\begin{align*}
    \|(T_{m_1} - T_{m_2}) (s' | s,a)\|_1 = \sum_{s' \in \mS} |T_{m_1} (s'|s,a) - T_{m_2} (s'|s,a)|.
\end{align*}
If we use $\PP$ without any subscript, it is a probability of an event measured outside of the context, {\it i.e.}, $\PP(\cdot) = \sum_{m=1}^M w_m \PP_m(\cdot)$. If the probability of an event depends on a policy $\pi$, we add superscript $\pi$ to $\PP$. Similarly, $\Exs_{m} [\cdot]$ is expectation taken over the $m^{th}$ context and $\pi$ is added as superscript if the expectation depends on $\pi$. We use $\hat{\cdot}$ to denote any estimated quantities. $a \lesssim b$ implies $a$ is less than $b$ up to some constant and logarithmic factors. We use $\preceq$ for a coordinate-wise inequality for vectors. When the norm $\| \cdot \|$ is used without subscript, we mean $l_2$-norm for vectors and operator norm for matrices. We interchangeably use $o$, an observation, to replace a pair of next-state and immediate reward $(s', r)$ to simplify the notation. We occasionally express a length $t > 0$ history $(s_1, a_1, r_1, ..., s_{t-1}, a_{t-1}, r_{t-1}, s_t)$ compactly as $((s,a,r)_{1:t-1}, s_t)$.


\section{Main Results}
In this section, we first investigate sample complexity of the LMDP framework, and obtain a hardness result for the general case. We then consider sample- and computationally efficient algorithms under additional assumptions. 

\subsection{Fundamental Limits of Learning General LMDPs}
\label{subsec:lower_bound}
We first study the fundamental limits of the problem. In particular, we are interested in whether we can learn the optimal policy after interacting with the LMDP for a number of episodes polynomial in the problem parameters. We prove a worst-case lower bound, exhibiting an instance of LMDP that requires at least $\Omega \left((SA)^M \right)$ episodes:
\begin{theorem}[Lower Bound]
\label{theorem:lower_bound}
    There exists an LMDP such that for finding an $\epsilon$-optimal policy $\pi_\epsilon$ for which ${V_{\mathcal{M}}^{\pi_\epsilon} \geq V_{\mathcal{M}}^*- \epsilon}$, we need at least $\Omega\left( (SA/M)^M / \epsilon^2 \right)$ episodes. 
\end{theorem}
The hard instance consists of MDPs with deterministic transitions and possibly stochastic rewards, indicating an exponential lower bound in the number of contexts even for the easiest types of LMDPs. The example is constructed such that, in the absence of knowing true contexts, all wrong action sequences of length $M$ cannot provide any information with zero reward, whereas the only correct action sequence gets a total reward 1 under one specific context. Nevertheless, we note here that Theorem \ref{theorem:lower_bound} does {\it not} suggest an exponential lower bound in $H$ when the number of contexts $M$ is fixed. A construction of the lower bound example is given in Appendix \ref{appendix:lower_bound}.

Theorem \ref{theorem:lower_bound} prevents a design of efficient algorithms with growing number of contexts. To the best of our knowledge this is the first lower bound of its kind for LMDPs. In the following subsections, we investigate natural assumptions which help us to develop an efficient algorithm when only polynomial number of episodes are available.

\subsection{The Critical First Step: Contexts in Hindsight}
\label{subsec:true_context}

\begin{algorithm}[t]
    \caption{Latent Upper Confidence Reinforcement Learning (L-UCRL)}
    \label{algorithm:lucrl}
    Initialize visit counts $N_m(s,a), N(m)$ and empirical estimates of an LMDP $(\hat{T}_m, \hat{R}_m, \hat{\nu}_m)$ properly.
    \begin{algorithmic}[1]
        \FOR{each $k^{th}$ episode}
            \STATE Construct optimistic model $\tmM$ with empirical estimates using Lemma \ref{lemma:optimistic_model}
            \STATE Get (approximately) optimal policy $\pi_k$ for $\tmM$
            \STATE Play policy $\pi_k$ and get the trajectory $\tau = (s_1, a_1, r_1, ..., s_H, a_H, r_H)$
            \STATE Get an estimated belief over contexts $\hat{b}$ at the end of episode with either Algorithm \ref{algorithm:access_true_context} (when contexts are given), or Algorithm \ref{algorithm:estimate_belief} (when we infer contexts)
            \STATE Update empirical parameters:
            \FOR{$m=1,...,M$}
                \FOR{$t=1,...,H$}
                    \STATE $N_m(s_{t+1} | a_t, s_t) \leftarrow N_m(s_{t+1} | a_t, s_t) + \hat{b}(m)$
                    \STATE $N_m(r_t | s_t, a_t) \leftarrow N_m(r_t | s_t, a_t) + \hat{b}(m)$
                \ENDFOR
                \STATE $N_m(s_1) \leftarrow N_m(s_1) + \hat{b}(m)$
                \FOR{all $(s', r, a, s) \in \mS \times \{0,1\} \times \mA \times \mS$}
                    \STATE Let $N_m(s, a) = \max \left(1, \sum_{x \in \mS} N_m (x | s, a) \right)$.
                    \STATE $\widehat{T}_m(s' | s, a) = \frac{N_m (s' |a, s)} {N_m(s, a)} $, $\ \widehat{R}_m(r | s, a) = \frac{N_m (r |s,a)} {N_m(s, a)} $
                    \STATE $\widehat{\nu}_m(s) = \frac{N_m (s)} {\max \left(1, \sum_{x \in \mS} N_m (x) \right)}$
                \ENDFOR
            \ENDFOR
        \ENDFOR
    \end{algorithmic}
\end{algorithm}

Suppose the true context of the underlying MDP is revealed to the agent at the end of each episode.  
We do not require any assumptions on the environments in this scenario. Note that this scenario is different from fully observable settings ({\it i.e.,} knowing the true context at the beginning of an episode). In the latter scenario, we would simply have $M$-decoupled RL problems in standard MDPs. While this can be considered as a ``warm-up'' for the sequel, it is motivated by real-world examples. Moreover, the key technical insight here will prove important for the sequel as well. 

Knowing contexts in hindsight allows us to construct a confidence set for parameters:
\begin{align}
    \label{eq:construct_confidence_set}
    \mC &= \{\mM \ | \ \|(T_m - \hat{T}_m) (s' | s,a)\|_1 \le \sqrt{c_T / N_m(s,a)}, \nonumber \\
    &\ \|(R_m - \hat{R}_m) (r|s,a)\|_1 \le \sqrt{c_R / N_m(s,a)}, \nonumber \\
    &\ \|(\nu_m - \hat{\nu}_m) (s)\|_1 \le \sqrt{c_\nu / N(m)}, \ \ \forall m,s,a\},
\end{align}
where $N_m(s,a)$ is the number of times each state-action pair $(s,a)$ in $m^{th}$ MDP is visited, and $N(m)$ is the number of episodes we interact with the $m^{th}$ MDP. With properly set constants (depending on problem parameters) $c_T, c_R, c_\nu > 0$ for the confidence intervals, $\mM^* \in \mC$ with high probability for all $K$ episodes. 

With the construction of confidence sets, it is then natural to try to design an optimistic RL algorithm, as in UCRL~\cite{jaksch2010near}. An obvious optimistic value in light of~\eqref{eq:construct_confidence_set} is  $\max_{\pi, \mathcal{M}\in \mathcal{C}} V_\mM^\pi.$ However, solving this optimization problem is more general than solving an LMDP. In fully observable settings, we could replace the complex optimization problem by adding a proper exploration bonus to obtain an optimistic value function~\cite{azar2017minimax}. 

In partially observable environments, the notion of value iteration is only defined in terms of belief-states and not the observed states. For this reason, existing techniques solely based on the value-iteration cannot be directly applied for LMDPs. Yet, we find that proper analysis of the Bellman update rule over the belief state reveals that an empirical LMDP with properly adjusted {\it hidden} rewards is optimistic:


\begin{proposition}
    \label{lemma:optimistic_model}
    We construct an optimistic LMDP $\tmM$ whose parameters are given such that:
    \begin{align*}
        &\widetilde{T}_m(s'|s,a) = \hat{T}_m (s'| s,a), \ \widetilde{R}_m^{obs} (r| s,a) = \hat{R}_m(r| s,a), \\
        &\widetilde{R}_m^{hid} (s,a) = H \min \left(1, \sqrt{5 (c_R + c_T)/N_m(s,a)} \right), \\
        &\widetilde{\nu}_m(s) = \hat{\nu}_m(s), \ \widetilde{R}_{init}^{hid} (m) = \min \left(1, \sqrt{c_\nu / N(m)} \right),
    \end{align*}
    where $\widetilde{R}_{init}^{hid} (m)$ is an initial hidden reward given when starting an episode with a context $m$, and $\widetilde{R}_m^{obs} (s,a)$ is a probability measure of an observable immediate reward $r$ whereas $\widetilde{R}_m^{hid} (s,a)$ is a hidden immediate reward (that is not visible to the agent) for a state-action pair $(s,a)$ in a context $m$. Then for any policy $\pi$, the expected long-term reward is optimistic, {\it i.e.,} $V_{\tmM}^\pi \ge V_{\mM^*}^\pi$.
\end{proposition}
To establish this result we make use of the `alpha vector' representation~\cite{smallwood1973optimal} of the value function of general POMDPs. Utilizing this representation we establish that each policy of the constructed LMDP has optimistic value, namely, it is not smaller than its value on the true LMDP. Detailed analysis is deferred to Appendix~\ref{appendix:description_optimistic_alpha}.

With the optimistic model constructed in Lemma~\ref{lemma:optimistic_model}, the planning-oracle efficient algorithm based on the optimism principle is straightforward. The resulting latent upper confidence reinforcement learning (L-UCRL) algorithm is summarized in Algorithm \ref{algorithm:lucrl}. We note that most existing planning algorithms can incorporate the hidden-reward structure without changes. For instance, the PBVI algorithm \cite{pineau2006anytime} can be executed as it is in the planning step. Hence in each episode, we can build one optimistic model from the Lemma \ref{lemma:optimistic_model}, and call the planning-oracle to get a policy to execute for the episode. Then we simply run the policy and update model parameters in a straight-forward manner. The algorithm can be efficiently implemented as long as some efficient (approximate) planning algorithms are available.

Based on the established optimism in Lemma~\ref{lemma:optimistic_model} and by carefully bounding the on-policy errors we arrive to the following regret guarantee of L-UCRL.
\begin{theorem}
    \label{theorem:scenario1_regret_bound}
    The regret of the Algorithm \ref{algorithm:lucrl} is bounded by:
    \vspace{-0.15cm}
    \begin{align*}
        Regret(K) \le \sum_{k=1}^K (V_{\tmM_k}^{\pi_k} - V_{\mM^*}^{\pi_k}) \lesssim HS\sqrt{MAN},
    \end{align*}
    \vspace{-0.2cm}
    where $N = HK$, {\it i.e.,} total number of taken actions.
\end{theorem}
Proof of Theorem~\ref{theorem:scenario1_regret_bound} is given in Appendix \ref{appendix:lucrl_regret_bound}. The central result of this section, Theorem~\ref{theorem:scenario1_regret_bound} leads to the following observation: a polynomial sample complexity is possible for the LMDP model assuming the context of the underlying MDP is supplied at the end of each episode. In the next sections we explore ways to relax this assumption, while still supplying with a polynomial sample complexity guarantee.

\begin{algorithm}[t]
    \caption{Access to True Contexts}
    \label{algorithm:access_true_context}
    \textbf{Input:} Get a true context $m^*$ in hindsight. \\
    \textbf{Output :} Return an encoded belief $\hat{b}$ over contexts:
    \vspace{-0.3cm}
    \begin{align*}
        \hat{b}(m) = \Big\{\begin{array}{lr}
        1, & \text{for } m = m^* \\
        0, & \text{for } m \neq m^*
        \end{array}
    \end{align*}
    \vspace{-0.3cm}
\end{algorithm}

\subsection{When we can Infer Contexts?}
\label{subsection:scenario2}
Without explicit access to the true context at the end of an episode, it is natural to estimate the context from the sampled trajectory. One {\it sufficient} condition that ensures such well-separatedness between MDPs is the following:
\begin{assumption}[$\delta$-Strongly Separated MDPs]
    \label{assump:delta_separation}
    For all $m$, $m_1, m_2 \in [M]$ such that $m_1 \neq m_2$, for all $(s,a) \in \mS \times \mA$, $l_1$ distance between probability of observations $o = (s', r)$ of two different MDPs in LMDP is at least $\delta > 0$, {\it i.e.,} $\| (\PP_{m_1} - \PP_{m_2}) (o | s,a) \|_1 \ge \delta$ for some constant $\delta > 0$. 
\end{assumption}

In order to reliably infer the true contexts the seperatedness between MDPs alone is not sufficient, since we need to estimate the contexts from the current empirical estimates of LMDPs. In order to reliably estimate the context from empirical estimate of LMDPs, we need a well-initialized empirical transition model of the LMDP:
\begin{align}
    &\| (\hat{T}_{m} - T_m) (s' | s,a) \|_1, \ \| (\hat{\nu}_{m} - \nu_m) (s) \|_1, \nonumber \\
    & \| (\hat{R}_{m} - R_m) (r| s,a) \|_1 \le \epsilon_{init}, \qquad \forall (s,a), \label{eq:initialization_condition}
\end{align}
for some initialization error $\epsilon_{init} > 0$. Note that while the initialization error is relatively small, it can be still not good enough to obtain a near-optimal policy (i.e., it will result in a linear regret). We can consider as if the state-action pairs are visited at least $N_0$ times such that $N_0 = c_T / \epsilon_{init}^2$ in each context in a fully observable setting.

Once the initialization is given along with separation between MDPs, we can modify Algorithm \ref{algorithm:lucrl} to update the empirical estimate of LMDP using the estimated belief over contexts computed in Algorithm \ref{algorithm:estimate_belief}. Note that when we update the model parameters, we increase the visit count of state-action pair $(s,a)$ at $m^{th}$ MDP by $\hat{b}(m)$. With Assumption \ref{assump:delta_separation}, it approximately adds a count for the correctly estimated context, but even without Assumption \ref{assump:delta_separation}, the update steps can still be applied. In fact, this is equivalent to an implementation of the so-called (online) expectation-maximization (EM) algorithm \cite{cappe2009line} for latent MDPs. Thus Algorithm \ref{algorithm:lucrl} with Algorithm \ref{algorithm:estimate_belief} essentially results in combining L-UCRL and the EM algorithm. 

In terms of performance guarantees, using Algorithm \ref{algorithm:estimate_belief} as a sub-routine for L-UCRL gives the same order of regret as in Theorem \ref{theorem:scenario1_regret_bound} as long as the true context can be almost correctly inferred with high probability for all $K$ episodes:
\begin{theorem}
    \label{theorem:scenario2_regret_bounds}
    Suppose Assumption \ref{assump:delta_separation} holds with $H > C \cdot \delta^{-4} \log^2(1/\alpha) \log(N/\eta)$ for some absolute constants $C, \delta > 0$, and a parameter $\alpha > 0$ such that $\alpha \ln(1/\alpha) \le \delta^2 / (200 S)$. If the initialization parameters satisfy equation \eqref{eq:initialization_condition} with some initialization error $\epsilon_{init} \le \delta^2 / (200 \ln(1/\alpha))$, then with probability at least $1 - \eta$, the regret of Algorithm \ref{algorithm:lucrl} is bounded by:
    \begin{align*}
        Regret(K) \lesssim HS \sqrt{MAN}.
    \end{align*}
\end{theorem}
The proof of Theorem \ref{theorem:scenario2_regret_bounds} is given in Appendix \ref{appendix:ucrl_em_regret}. While currently the provable guarantees are given only for well-separated LMDPs, we empirically evaluate Algorithm \ref{algorithm:lucrl} as a function of separations and initialization~(see Figure \ref{fig:em_lucrl}). 

\begin{algorithm}[t]
    \caption{Inference of Contexts}
    \label{algorithm:estimate_belief}
    \textbf{Input:} Trajectory $\tau = (s_1, a_1, r_1, ..., s_H, a_H, r_H)$ \\
    \textbf{Output:} Return an estimate of belief over contexts $\hat{b}$:
    \vspace{-0.2cm}
    \begin{align*}
        \hat{p}_m (\tau) &= \Pi_{t=1}^H ( \alpha + (1 - 2\alpha S) \hat{\PP}_m(s_{t+1},r_t|s_t,a_t)), \nonumber \\
        \hat{b}(m) &= \frac{\hat{p}_m (\tau)}{\sum_{m=1}^M \hat{p}_m (\tau)}.
    \end{align*}
    \vspace{-0.2cm}
\end{algorithm}

An interesting consequence of the Assumption \ref{assump:delta_separation} is that the length of episode can be logarithmic in the number of problem parameters. With much longer time-horizons $H \ge \Omega(S^2 A / \delta^2)$, \cite{brunskill2013sample, hallak2015contextual} assumed similar $\delta$-separation only for some $(s,a)$ pairs. While Assumption~\ref{assump:delta_separation} requires a stronger assumption of $\delta$-separation for all state-actions, the requirement on the time-horizon can be significantly weaker. For a slightly more general separation condition, see Appendix \ref{appendix:well_separated_mdps}. 


\subsection{Learning LMDPs without Initialization}
\label{subsection:scenario3}

Finally, we discuss efficient initialization with some additional assumptions. Clustering trajectories is the cornerstone of all our technical results, as this allows us to estimate the parameters of each hidden MDP and then apply the techniques of Section \ref{subsec:true_context}. The challenge is how to cluster when we have short trajectories, and no good initialization. 


The key is again in Assumption \ref{assump:delta_separation}. 
In Section \ref{subsection:scenario2}, we use a good initialization to obtain accurate estimates of the belief states. These can then be clustered, thanks to Assumption \ref{assump:delta_separation}, allowing us to obtain the true label in hindsight. Without initialization, we cannot accurately compute the belief state, so this avenue is blocked. Instead, our key idea is to leverage a predictive state representation (PSR) of the POMDP dynamics, and then show that Assumption \ref{assump:delta_separation} also allows us to cluster in this space. 

Algorithm~\ref{algo:clustering_trajectories_pseudo} gives our approach. We first explain the high-level idea, and subsequently detail some of the more subtle points. Suppose we have PSR parameters allowing us to estimate $\PP(o | h \| \textbf{do } a)$, (the probabilities of any future observations $o = (s',r)$ given a history $h$ and intervening action $a$) to within accuracy $o(\delta)$. We then show that we can again apply Assumption \ref{assump:delta_separation}, to (almost) perfectly cluster the MDPs by true context at the end of the episode. 
After we collect transition probabilities at all states near the end of episode, we can construct a full transition model for each MDP.



Learning the PSR to sufficient accuracy requires an additional assumption. We show that the following standard assumption on statistical sufficiency of histories and tests, is sufficient for our purposes (see also Section \ref{subsec:review_psr} and Appendix \ref{appendix:spectral_learning_psr}):
\begin{assumption}[Sufficient Tests/Histories]
    \label{assumption:sufficient_set}
    Let $\mT$ and $\mH$ be the set of all possible tests and histories of length $l=O(1)$ respectively, with a given sampling policy $\pi$ ({\it e.g.,} uniformly random policy) for histories $\mH$. $\mT$ and $\mH$ satisfy Condition \ref{condition:rank_test} and \ref{condition:rank_history} respectively.
\end{assumption}
While the worst-case instance may require $l \ge M$ to satisfy the full-rank conditions, we assume that the length of sufficient tests/histories is $l = O(1)$. In fact, $l = 1$ has been (implicitly) the common assumption in the literature on learning POMDPs \cite{hsu2012spectral, azizzadenesheli2016reinforcement, guo2016pac, jin2020sample}. Empirically, we observe that the more MDPs differ, the more easily they satisfy Assumption \ref{assumption:sufficient_set}. See Figure \ref{fig:psr_clustering}. 
At this point, we are not aware whether sample-efficient learning is possible with only Assumption \ref{assump:delta_separation}. 

Though the main idea and key assumption are above, a few important details and technical assumptions remain to complete this story. The primary guarantee still required is that we have access to an exploration policy with sufficient mixing, to guarantee we can collect all required information to perform the PSR-based clustering. The following assumption ensures that additional $\tilde{O}(M/\alpha_2)$ sample trajectories obtained with the exploration policy $\pi$ can provide $M$ clusters of estimated one-step predictions $\PP_m(o | s, a)$ for every state $s$ and intervening action $a$. 
\begin{assumption}[Reachability of States]
    \label{assumption:reachability}
    There exists a priori known exploration policy $\pi$ such that, for all $m \in [M]$ and $s \in \mS$,  we have ${\PP_m^{\pi} (s_{H-1} = s) \geq \alpha_2}$ for some $\alpha_2 > 0$. 
\end{assumption}
A subtle point here is that we still have an ambiguity issue in the ordering of contexts (or labels) assigned in different states, which prevents us from recovering the full model for each context. In Appendix \ref{Appendix:clustering_with_psrs}, we describe an approach that resolves this ambiguity assuming the MDP is connected.


\begin{algorithm}[t]
    \caption{(Informal) Recovery of LMDP parameters}
    \label{algo:clustering_trajectories_pseudo}

    \begin{algorithmic}
        \STATE Learn PSR parameters up to precision $o(\delta)$
        \STATE Get clusters $\{\hat{T}_m(\cdot|s,a), \hat{R}_m(\cdot|s,a)\}_{(s,a) \in \mS \times \mA, m\in [M]}$ with learned PSR parameters
        \STATE Build each MDP model by correctly assigning contexts to estimated transition and reward probabilities
        \STATE {\bf Return} Well-initialized model $\{\hat{T}_m, \hat{R}_m\}_{m\in [M]}$
    \end{algorithmic}
\end{algorithm}

We conclude this section with an end-to-end guarantee.
\begin{theorem}
    \label{theorem:final_result}
    Let Assumption \ref{assumption:sufficient_set} hold for an LMDP instance with a sampling policy $\pi$. Furthermore, assume the LMDP satisfies Assumptions \ref{assump:delta_separation} and \ref{assumption:reachability}. We learn the PSR parameters with $n_0$ short trajectories of length $2l+1$ where
    \begin{align*}
        n_0 = poly(A^l, S, \epsilon_c^{-1}, \sigma_h^{-1}, \sigma_\tau^{-1}, \alpha_2^{-1}, \alpha_3^{-1}, H, M),
    \end{align*}
    where $\epsilon_c < \min(\delta, \epsilon_{init})$ is a desired accuracy for estimated predictions, and $\alpha_3 > 0$ is a parameter related to the connectivity of MDPs (see Assumption \ref{assumption:connectivity} in Appendix \ref{Appendix:clustering_with_psrs}). Let the number of additional episodes with time-horizon $H \ge C \cdot \delta^{-4} \log^2 (1/\alpha) \log(N/\eta)$ (as in Theorem \ref{theorem:scenario2_regret_bounds}) to be used for the clustering be
    \begin{align*}
        n_1 = C_1 \cdot MA \log(MS) / (\alpha_2\alpha_3),
    \end{align*}
    with some absolute constant $C_1 > 0$. Then with probability at least $2/3$, Algorithm \ref{algo:clustering_trajectories_pseudo} (see the full algorithm described in Algorithm \ref{algo:clustering_trajectories}) returns a good initialization of LMDP parameters that satisfies the initialization condition \eqref{eq:initialization_condition}.
\end{theorem}
Theorem \ref{theorem:final_result} completes the entire pipeline for learning in latent MDPs: we initialize the parameters by the estimated PSR and clustering (see Appendix \ref{appendix:unsupervised_algorithm}) up to {\it some} accuracy, and then we run L-UCRL to refine the model and policy up to {\it arbitrary} accuracy (Algorithm \ref{algorithm:lucrl}). Note that the $2/3$ probability guarantee can be boosted to arbitrarily high precision $1 - \eta$ by repeating Algorithm \ref{algo:clustering_trajectories}  $O(\log(1/\eta))$ times, and selecting a model via majority vote. We mention that we have not optimized polynomial factors as our focus is to avoid the exponential lower bound with additional assumptions. The proof of Theorem \ref{theorem:final_result} is given in Appendix \ref{appendix:proof_clustering_with_psr}.



\section{Experiments} \label{section: Experiments}
\label{section:experiments}
In this section, we evaluate the proposed algorithm on synthetic data. Our first two experiments illustrate the performance of L-UCRL (Algorithm \ref{algorithm:lucrl}) 
for various levels of separation and quality of initialization. Then, we empirically study the performance of the PSR-Clustering algorithm for randomly generated LMDPs for different levels of separation and time-horizon. 

\begin{figure}[t]
    \centering
    \begin{subfigure}{0.5\textwidth}
        \centering
        \includegraphics[width=\textwidth]{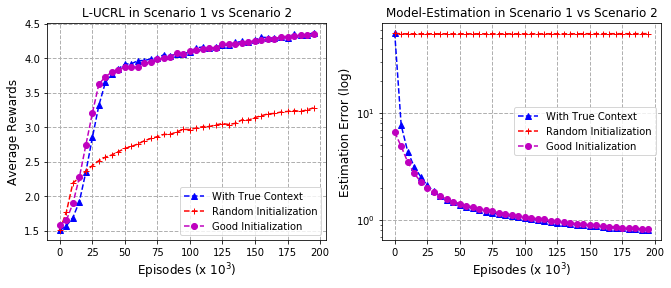}
        \caption{True context vs EM with poor/good initialization}
        \label{fig:context_vs_em}
    \end{subfigure}
    \hfill
    \begin{subfigure}{0.45\textwidth}
        \centering
        \includegraphics[width=\textwidth]{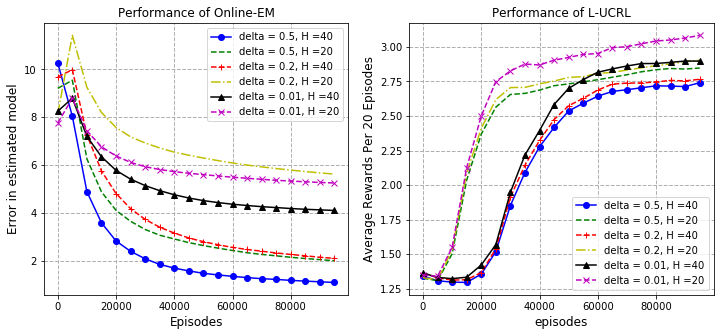}
        \caption{EM and L-UCRL with good initialization}
        \label{fig:em_lucrl}
    \end{subfigure}
    \caption{(a) We compare performance of L-UCRL (Algorithm \ref{algorithm:lucrl}) when true contexts are revealed in hindsight (Algorithm \ref{algorithm:access_true_context}) and when we infer true contexts with the EM algorithm (Algorithm \ref{algorithm:estimate_belief}). Without proper initialization, EM may converge to a local optimum which in turn results in sub-optimal policy. (b) Performance of EM + L-UCRL under different level of separation and horizon. The larger $\delta$ and the longer $H$, MDPs are better separated which results in a fast convergence of EM. For smaller $\delta$ or $H$, the convergence speed of EM slows down. L-UCRL finds the optimal policy for all different levels of $\delta$ and $H$.}
\end{figure}

\subsection{The Value of True Contexts in Hindsight}
We first study the importance of getting true contexts in hindsight for the approach analyzed in this work, by comparing Algorithm \ref{algorithm:lucrl} when using Algorithm \ref{algorithm:access_true_context} or \ref{algorithm:estimate_belief} as a sub-routine. We generate random instances of LMDPs of size $M = 7, S = 15, A = 3$ and set the time-horizon $H = 30$. The reward distribution is set to be 0 for most state-action pairs. We compare when we give a true context to the algorithm (Algorithm \ref{algorithm:access_true_context}) and when we infer a context with random initialization or good initialization (Algorithm \ref{algorithm:estimate_belief}). Note that in the latter case, it is equivalent to running the EM algorithm for the model estimation. 

We measure the model estimation error by simply summing over the  $l_1$ differences in probabilities of reward and transitions:
\begin{align*}
    error := \min_{\sigma \in \text{Perm}_M} \sum_{(m,s,a)} \|(\PP_m - \hat{\PP}_{\sigma(m)}) (s',r| s,a)\|_1,
\end{align*}
where $\text{Perm}_M$ denotes all length $M$ permutation sequences. The performance of the policy is measured by averaging the total rewards over the last thousand episodes. For the planning algorithm, we find that the Q-MDP heuristic \cite{littman1995learning} shows good performance. The measured errors are averaged over 10 independent experiments.

The experimental results are given in Figure \ref{fig:em_lucrl}. When the true context is given at the end of episode (with Algorithm \ref{algorithm:access_true_context}), L-UCRL converges to the optimal policy as our theory suggests. On the other hand, if the true context is not given (with Algorithm \ref{algorithm:estimate_belief}), the quality of initialization becomes crucial; when the model is poorly initialized, the estimated model converges to a local optimum which leads to a sub-optimal policy. When the model is well-initialized, L-UCRL performs as well as when true contexts are given in hindsight.

\subsection{Performance of L-UCRL with Good Initialization}
In our second experiment, we focus on the performance of L-UCRL (Algorithm \ref{algorithm:lucrl}) along with Algorithm \ref{algorithm:estimate_belief} under different levels of separation ($\delta$ in Assumption \ref{assump:delta_separation}) when approximately good model parameters are given. For various levels of $\delta$, we generate the parameters for transition probabilities randomly while keeping the distance between different MDPs to satisfy $\delta \le \|(T_{m_1} - T_{m_2}) (s'|s,a) \|_1 \le 2\delta$ for $m_1 \neq m_2$. As in the previous section, we test the algorithms on random instances of LMDPs of size $M = 7, S = 15, A = 3$. 

We show the error in the estimated model and average long-term rewards in Figure \ref{fig:em_lucrl}. When the separation is sufficient (larger $\delta$ or $H$), 
the estimated model converges fast to the true parameters. When the separation gets small (smaller $\delta$ or $H$), 
the convergence speed gets slower. This type of transition in the convergence speed of EM (the update of model parameters with Algorithm \ref{algorithm:estimate_belief}) is observed both in theory and practice when the overlap between mixture components gets larger ({\it e.g.,} \cite{kwon2020minimax}). On the other hand, the policy steadily improves regardless of the level of separation. We conjecture that this is because the optimal policy would only need the model to be accurate in total-variation distance, not in the actual estimated parameters.

\subsection{Initialization with PSR and Clustering}
\begin{figure}[t]
    \centering
    \includegraphics[width=0.9\textwidth]{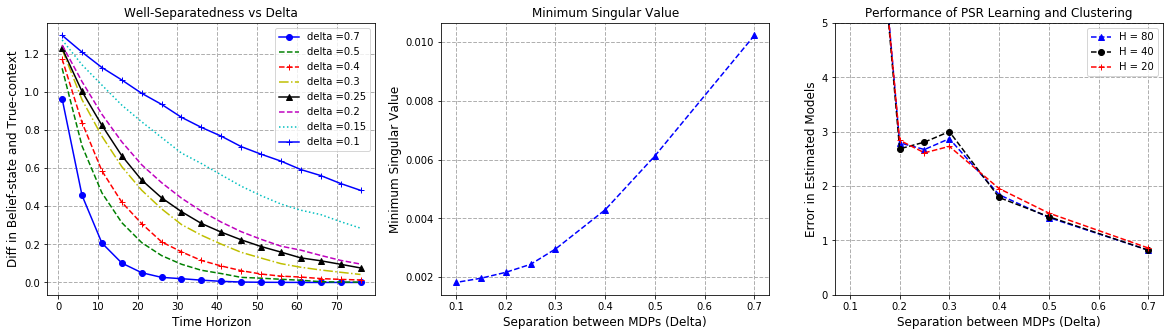}
    \caption{Performance of PSR learning and Clustering based algorithm for different levels of separation. {\bf Left:} Convergence of belief state for different levels of $\delta$ and $H$. {\bf Middle:} $M^{th}$ singular value of sufficient histories/tests matrix for various $\delta$. {\bf Right:} Accuracy of estimated model produced by PSR-Clustering algorithm under various $\delta$ and $H$. Failed when $\delta < 0.2$ due to small $M^{th}$ singular value of $P_{\mT, \mH}$.}
    \label{fig:psr_clustering}
\end{figure}

In the third experiment, we evaluate the initialization algorithm (Algorithm \ref{algo:clustering_trajectories_pseudo}) for randomly generated LMDP instances. Since PSR learning requires a (relatively) large number of short sample trajectories, we evaluate this step on smaller instances with $S = 7, A = 2, M = 3$. The LMDP instances are generated similarly as in the second experiment with different levels of $\delta$ and $H$. The reward and initial distributions are set the same across all MDPs. 

To learn the parameters of PSR, we run $10^6$ episodes with $H = 4$.
We assume histories and tests of length 1 are statistically sufficient 
with the uniformly random policy. In the clustering step, we run an additional $5 \cdot 10^3$ episodes to obtain longer trajectories of length $H = 20, 40$ and $80$.  
We report the experimental results in Figure \ref{fig:psr_clustering}. 

We first observe how the level of separation $\delta$ between MDPs impacts trajectory separation, i.e., belief state vs true label (left). Recall that this separation property is the key for clustering trajectories. We then examine the performance of Algorithm \ref{algo:clustering_trajectories_pseudo} (see full Algorithm \ref{algo:clustering_trajectories}) for various levels of separation. 
Empirically, it succeeds to get a good initialization of an LMDP model when we have sufficient separation. As the separation level decreases, the algorithm fails to recover good enough LMDP parameters (Right). There are two possible sources of the failure: (1) the belief state is far from the true context, and (2) the similarity between MDPs drops the $M^{th}$ singular value of $P_{\mT, \mH}$ (Middle). We can compensate for (1) if we have a longer time-horizon to infer true contexts, as in the leftmost graph. For (2), if the $M^{th}$ singular value of $P_{\mT, \mH}$ drops, we require more samples for the estimation of PSR parameters. In our experiments, as we decreased $\delta$ we found that failure in the spectral learning step was the more significant of the two.

\section{Conclusion and Future Work}
We establish the first theoretical results in RL with latent contexts. We first have established a lower bound for general LMDPs, showing that necessary number of episodes can be exponential in the number of contexts. Then, we find that a sample-efficient RL is possible when true contexts of interacting MDPs are revealed in hindsight. Building off on this observation, we proposed a sample-efficient algorithm for a class well-separated LMDP instances with additional technical assumptions. We also evaluated the proposed algorithm on synthetic data. The proposed EM and L-UCRL algorithm performed very well once initialized in random instances, whereas the spectral learning and clustering method was sensitive to the amount of separation between different contexts. 

There are several interesting research venues in continuation of this work. An interesting direction is to study RL algorithms for LMDPs with no underlying assumptions. Although our lower bound suggests such an algorithm necessarily suffers an exponential dependence in the number of contexts, if this number is small, such dependence might be acceptable on an algorithm designer. Specifically, we conjecture the following:
\begin{question}[Upper Bound]
    \label{question:upper_bound}
    Can we learn the $\epsilon$-optimal policy for LMDPs with sample complexity at most $poly \left((HSA)^M,\epsilon^{-1} \right)$ without any assumptions?
\end{question}
In Appendix \ref{appendix:upper_bound}, we show that the exponential dependence in $M$ is sufficient when MDPs are fully deterministic. The case for general LMDPs is an interesting open question.
Furthermore, a needed empirical advancement is to design efficient ways to learn the set of sufficient histories/tests for learning predictive state representation of LMDPs. This can dramatically improve the performance of our algorithms when a sufficiently good initial model needs to be learned.

\bibliographystyle{abbrv}
\bibliography{main}

\begin{appendices}

\section{Guarantees for Latent Deterministic MDPs}
In this section, we provide lower and upper bound for LMDP instances with deterministic MDPs. The lower bound for latent deterministic MDPs implies the lower bound for general instances of LMDPs, proving Theorem \ref{theorem:lower_bound}. The upper bound for latent deterministic MDPs supports our conjecture on the sample complexity of learning general LMDPs (Open Question \ref{question:upper_bound}), and can be of independent interest. 

\subsection{Lower Bound (Theorem \ref{theorem:lower_bound})}
\label{appendix:lower_bound}
\begin{figure}[t]
    \centering
    \includegraphics[width=0.7\textwidth]{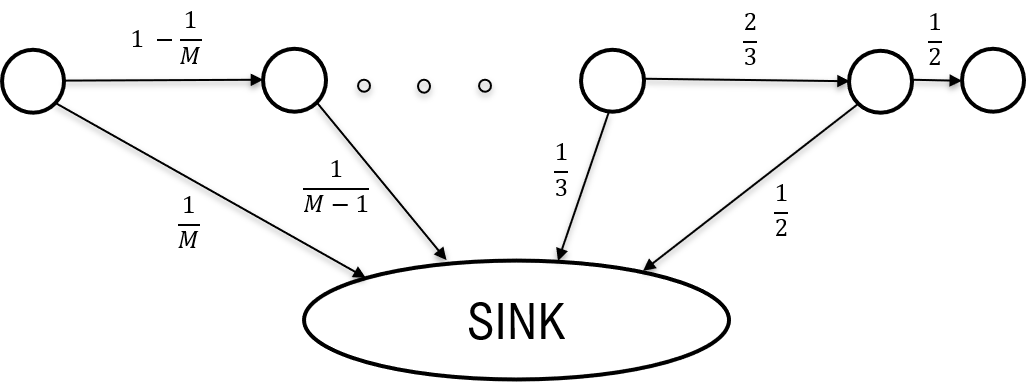}
    \caption{External view of the system dynamics with wrong action sequences without context information. Arrows indicate transition probabilities of a surrogate Markov chain that represents the external view.}
    \label{fig:lower_bound}
\end{figure}

We consider the following constructions with $M$ deterministic MDPs and $H = M$, $S = M + 1$ with $A$ actions: 
\begin{enumerate}
    \item At the start of episode, one of $M$-MDPs in $\{\mM_1, \mM_2, ..., \mM_M\}$ are chosen with probability $1/M$. 
    \item At each time step, each MDP either goes to the next state or go to the SINK-state depending on the action chosen at the time step. Once we fall into the SINK state, we keep staying in the SINK state throughout the episode without any rewards. 
    
    \item Rewards of all state-action pairs are all 0 except at time step $t = M$ with the right action choice $a_M$ only in the first MDP $\mM_1$.
    
    \item At time step $t = 1$, there are three state-transition possibilities:
    \begin{itemize}
        \item $\mM_1$: For all actions $a \in \mA$ except $a_1$, we go to the SINK state. For the action $a_1$, we go to the next state.
        \item $\mM_M$: For all actions $a \in \mA$ except $a_1$, we go to the {\it next} state. For the action $a_1$, we go to the SINK state.
        \item $\mM_2$, ..., $\mM_{M-1}$: For all actions $a \in \mA$, we go to the next state.
    \end{itemize}
    
    \item At time step $t = 2$, we again have three cases but now $\mM_1$ and $\mM_K$ would look the same:
    \begin{itemize}
        \item $\mM_1, \mM_M$: For all actions $a \in \mA$ except $a_2$, we go to the SINK state. For the action $a_2$, we go to the next state.
        \item $\mM_{M-1}$: For all actions $a \in \mA$ except $a_2$, we go to the next state. For the action $a_2$, we go to the SINK state.
        \item $\mM_2$, ..., $\mM_{M-2}$: For all actions $a \in \mA$, we go to the next state.
    \end{itemize}
    
    ...
    
    \item At time step $t = M-1$, 
    \begin{itemize}
        \item $\mM_1$, $\mM_3, ..., \mM_M$: For all actions $a \in \mA$ except $a_{M-1}$, we go to the SINK state. For the action $a_{M-1}$, we go to the next state.
        \item $\mM_2$: For all actions $a \in \mA$ except $a_{M-1}$, we go to the next state. For the action $a_{M-1}$, we go to the SINK state.
    \end{itemize}
    
    \item At time step $t = M$, there are two possibilities of getting rewards:
    \begin{itemize}
        \item $\mM_1$: For the action $a_{M} \in \mA$, we get reward 1. For all other actions, we get no reward.
        \item $\mM_2, ..., \mM_M$: For all actions $a \in \mA$, we get no reward.
    \end{itemize}
\end{enumerate}

Note that the right action sequence is $a^* = (a_1, a_2, ..., a_M)$. However, without the information on true contexts, the system dynamics with any wrong action sequence among the $A^M-1$ wrong sequences, is exactly viewed as Figure \ref{fig:lower_bound} with zero rewards, i.e.,
\begin{align*}
    \PP(s_{1:H}, r_{1:H} \| \textbf{do } a^{(1)}_{1:H}) = \PP(s_{1:H}, r_{1:H} \| \textbf{do } a^{(2)}_{1:H}),
\end{align*}
for any two wrong action sequences $a^{(1)} = a^{(1)}_{1:H}$, $ a^{(2)} = a^{(2)}_{1:H}$ such that $a^{(1)}, a^{(2)} \neq a^*$. The probability distribution of observation sequences with any wrong action sequence is the same as the distribution of sequences generated by the surrogate Markov chain in Figure \ref{fig:lower_bound}. Therefore, we cannot gain any information from executing wrong action sequences besides of eliminating this wrong action sequence. Note that there are $A^M$ possible choice of action sequences. Hence the problem is reduced to find one specific sequence among $A^M$ possibilities without any other information on the correct action sequence. It leads to the conclusion that before we play most of $A^M$ action sequences, we cannot find the correct one.

\begin{figure}[h]
    \centering
    \includegraphics[width=0.7\textwidth]{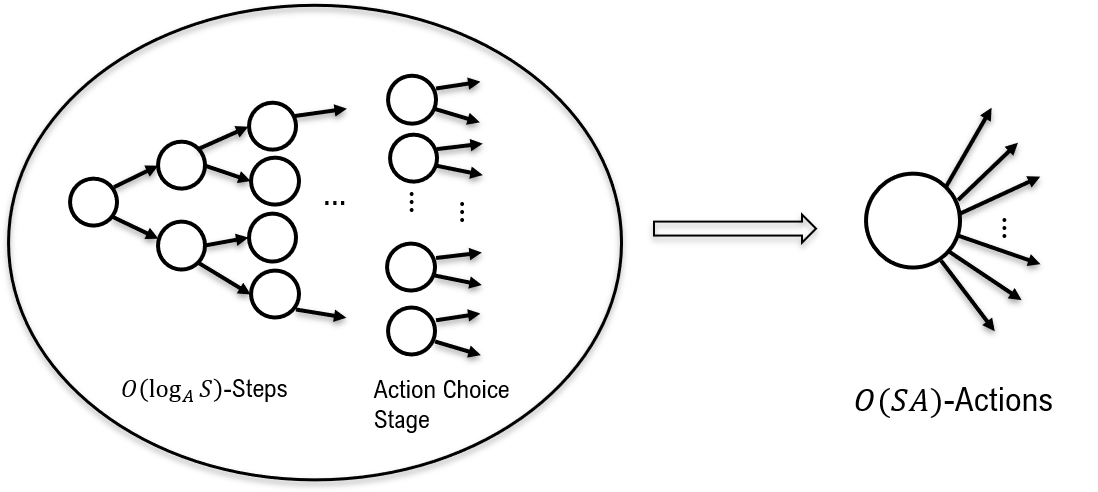}
    \caption{Effectively amplifying the number of actions}
    \label{fig:SA_lower_bound}
\end{figure}

Now the above argument can be easily extended to get a lower bound of $\Omega\left( \left(\frac{SA}{M} \right)^M \right)$, by effectively amplifying the number of actions up to $O \left(\frac{SA}{M} \right)$. That is, we can amplify the effective number of actions by considering a big state consisting of a tree of states with $O(\log_A S)$-depth (see Figure \ref{fig:SA_lower_bound}). Since we have such $M$ big states, total number of states is $O(MS)$ in our lower bound example with amplified number of actions, or conversely if the number of total states is $S$, then effective number of actions is $O (SA/M)$ per each big state. This gives the $\Omega \left((SA/M)^M \right)$ lower bounds.

Finally, we can easily include $\epsilon$-additive approximation factor to get the lower bound $\Omega(A^M / \epsilon^{2})$ by properly adjusting the reward distribution at the last time step $t=H$ as the following: 
\begin{itemize}
    \item $\mM_1$: For the action $a_{M} \in \mA$ at time step $H=M$, we get a reward from Bernoulli distribution $Ber(1/2 + \epsilon)$. For all other actions, we get a reward from Bernoulli distribution $Ber(1/2)$.
    \item $\mM_2, ..., \mM_M$: For all actions $a \in \mA$, we get a reward from Bernoulli distribution $Ber(1/2)$.
\end{itemize}
Then the distribution of the final reward with $a^*$ is $1/M \cdot Ber(1/2 + \epsilon) + (M-1)/M \cdot Ber(1/2)$, whereas the distribution of all other action sequences is $Ber(1/2)$. Similarly to the above, for any wrong action sequence, the probability of observations is identical. Hence, identifying the optimal action sequence $a^*$ among all $A^M$ action sequences requires $\Omega( A^M/\epsilon^{2})$ trials, i.e., to identify an $\epsilon$ optimal arm among $A^M$ actions.

\subsection{Upper Bound for Learning Deterministic LMDPs}
\label{appendix:upper_bound}
Although the lower bound is exponential in $M$, it would not be a disaster for instances with small number of contexts. In this appendix, we briefly discuss whether the exponential dependence in $M$ is sufficient for learning deterministic MDPs with latent contexts. If that is the case, then we can exclude the possibility of $\Omega(A^H)$ lower bound for deterministic LMDP instances. Intuitively, the exponential dependence in time-horizon is unlikely in LMDPs for the following reason: under certain regularity assumptions, if the time-horizon is extremely long $H \gg S^2 A$ such that every state-action pair can be visited sufficiently many times, then each trajectory can be easily clustered and the recovery of the model is straight-forward. The following theorem shows that we do {\it not} suffer from $\Omega(A^H)$ sample complexity for deterministic LMDPs:
\begin{theorem}[Upper Bound for Deterministic LMDPs]
\label{theorem:upper_bound}
    For any LMDP with a set of deterministic MDPs, there exists an algorithm such that it finds the optimal policy after at most $O \left(H (SA)^{M} \cdot poly(M) \log M \right)$ episodes.
\end{theorem}
The algorithm for the upper bound is implemented in Algorithm \ref{algo:deterministic_mdp}. While the upper bound for deterministic LMDPs does not imply the upper bound for general stochastic LMDPs, we have shown that the exponential lower bound higher than $O((SA)^M)$ cannot be obtained via deterministic examples. We leave it as future work to study the fundamental limits of general instances of LMDPs, and in particular, whether the problem is learnable with $\tilde{O}((SA)^M)$ sample complexity, which can be promising when the number of contexts is small enough ({\it e.g.,} $M = 2, 3$).

\begin{algorithm}[t]
    \caption{Exploring Deterministic MDPs with Latent Contexts}
    \label{algo:deterministic_mdp}
    
    {\bf Initialization:} For $O(M \log M)$ episodes, observe the possible initial states (discard all other information). If there is only one initial state $s_1$ for $O(M \log M)$ episodes, then set $C_1 = \{(\phi, s_1)\}$. Otherwise, let $C_1$ be a set of all observed initial states with $\{ (\{(init, s_1)\}, s_1), \forall \text{$s_1$ observed}\}$.

    \begin{algorithmic}[1]
        \FOR{$t = 1, ..., H$}
            \STATE Let $C_{t+1} = \{\}$.
            \FOR{each $(C,s) \in C_{t}$}
                \FOR {each $a \in \mA$}
                    \STATE Let $O = \{\}$.
                    \STATE 1. Find any action sequence $a_1, ..., a_{t-1}$ that can result in a state $s$ with distinguishing observations $C$.
                    \STATE 2. For $O(M \log M)$ episodes, run the action sequence $a_1, ..., a_{t-1}$ (execute any policy for the remaining time steps).
                    \STATE 2.1. If we reached the state $s$ with distinguishing observations $C$, then run action $a$, and get a new observation of next state and reward $(s',r)$.
                    \STATE 2.2. Update $O \leftarrow O \cup \{(s', r)\}$ and record the probability $p$ of observing $(s', r)$ conditioned on $C$ and $s$.
                    
                    \IF{$|O| = 1$} 
                        \STATE 3.1. With the only element $(s', r) \in O$, update $C_{t+1} \leftarrow C_{t+1} \cup \{(C, s')\}$.
                        \STATE Record that there is a path from $(C, s)$ to $(C, s')$ by taking action $a$ with a reward $r$.
                    \ELSE
                        \STATE 3.2. For all $(s', r) \in O$, let $C' = C \cup \{(s,a,s',r)\}$. Update $C_{t+1} \leftarrow C_{t+1} \cup \{(C', s')\}$.
                        \STATE Record that there is a path from $(C, s)$ to $(C', s')$ by taking action $a$ with a reward $r$ and a recorded probability $p$.
                    \ENDIF
                \ENDFOR
            \ENDFOR
        \ENDFOR
    \end{algorithmic}
\end{algorithm}

\subsubsection{Proof of Theorem \ref{theorem:upper_bound}}
Algorithm \ref{algo:deterministic_mdp} is essentially a pure exploration algorithm which searches over all possible states. After the pure exploration phase, we model the entire system as one large MDP with $poly(M) \cdot (SA)^M$ states. The optimal policy can be found by solving this large MDP. 

The core idea behind the algorithm is that since the system is deterministic, whenever there exist more than one possibility of observations (a pair of reward and next state) from the same state and action, it implies that at least one MDP shows a different behavior from other MDPs for the state-action pair. Therefore, each observation can be considered as a new distinguishing observation that can separate at least one MDPs from other MDPs. Afterwards, we can consider a sub-problem of exploration in the remaining time-steps given the distinguishing observation in history and the current state. The argument can be similarly applied in sub-problems, which leads to the concept of conditioning on a set of distinguishing observations and the current state. 

On the other hand, if an action results in the same observation for all MDPs given a set of distinguishing observations and a state, then we would only see one possibility. In this case, this state-action pair does not reveal any information on the context, and can be ignored for future decision making processes. 

Algorithm \ref{algo:deterministic_mdp} implements the above principles: for each time step $t$, we construct a set of all reachable states with a set of distinguishing observations in histories. In order to find out all possibilities, for each observation set, state, and action we first find the action sequence by which we can reach the desired state (with target distinguishing observation set). Since all MDPs are deterministic, the existence of path means at least one MDP always results in the desired state with the action sequence. The sequence can be found by the induction hypothesis that we are given all possible transitions and observations in previous time steps $1, ..., t-1$. By the coupon-collecting argument, if we try the same action sequence for $O(M \log M)$ episodes, we can see all different transitions that all different MDPs resulting in the target observation set and state can give. By doing this for all reachable states and observation sets, we can find out all possibilities that can happen at the time step $t$. The procedure repeats until $t = H$ and eventually we can find all possible outcomes from all action sequences. 

An important question is how many different possibilities we would encounter in the procedure. Note that as we find out a new distinguishing observation, we cut out the possibility of at least one MDP conditioning on that new observation. Since there are only $M$ possible MDPs, the size of distinguishable observation sets cannot be larger than $M - 1$. Based on this observation, we can see that the number of all possible combinations of the observation set and state is less than ${MSA \choose M-1} \cdot S$. Note that the $MSA$ is the total number of possible state-action-observation $(s,a,s',r)$ pairs. Hence in each time step, the iteration complexity does not exceed ${MSA \choose K-1} \cdot S A$ times the number of episodes for each possible state and observation set. Since we loop this procedure for $H$ steps, the total number of episodes is bounded by $O \left(HSA {MSA \choose M-1} \cdot M \log M \right)$, which results in the sample complexity of $O \left(H (SA)^M \cdot poly(M) \log M \right)$.

\section{Analysis of L-UCRL when True Contexts are Revealed}
In this section, we prove the optimism lemma (Lemma \ref{lemma:optimistic_model}) and regret guarantee (Theorem \ref{theorem:scenario1_regret_bound}) achieved by Algorithm \ref{algorithm:lucrl} when true contexts are given in hindsight.

\subsection{Analysis of Optimism in Alpha-Vectors}
\label{appendix:description_optimistic_alpha}
We start with an important observation that the upper confidence bound (UCB) type algorithm can be implemented in the belief-state space. Even though the exact planning in a belief-state space is not implementable, we can still discuss how the value iteration is performed in partially observable domains. Let $h$ be an entire history at time $t$, and denote $b(h)$ be a belief state over $M$ MDPs corresponding to a history $h$. The value iteration with a (history-dependent) policy $\pi$ is given as
\begin{align*}
    Q_t^\pi (h,s,a) = b(h)^\top \bar{R} (s,a) + \Exs_{s',r | h,s,a} [V_{t+1}^\pi (h', s')],
\end{align*}
for $t = 1,..., H$, where $h' = ha(rs')$ is a concatenated history. Here $Q_t^\pi (h,s,a)$ and $V_t^\pi (h,s)$ are state-action-value and state-value function at time step $t$ respectively given a history $h$ and a policy $\pi$. $\bar{R}(s,a) \in \mathbb{R}^M$ is a vector where value of $m^{th}$ coordinate $\bar{R}_m(s,a)$ is an expected immediate reward at $(s,a)$ in $m^{th}$ MDP, {\it i.e.,} $\bar{R}_m(s,a) = \Exs_{r \sim R_m (r | s,a)} [r]$. In case there exists a hidden reward $R_m^{hid} (s,a)$, we define $\bar{R}_m(s,a) = \Exs_{r \sim R_m (r | s,a)} [r] + R_m^{hid} (s,a)$. At the end of episode, we set $V_{H+1}^\pi = 0$. We first need the following lemma on the policy evaluation procedure of a POMDP.
\begin{lemma}
    For any history $h$ at time $t$, the value function for a policy $\pi$ can be written as
    \begin{align}
        V_t^\pi (h, s) = b(h)^\top \alpha_{t, s}^{h, \pi},
    \end{align}
    for some $\alpha_{t,s}^{h,\pi} \in \mR^M$ uniquely decided by $t,s, h$ and $\pi$. 
\end{lemma}
\begin{proof}
    We will show that the value of $\alpha_{t,s}^{h,\pi}$ is decided only by a history and policy, and is not affected by the history to belief-state mapping. On the other hand, the Bayesian update for $h'$ is given by
    \begin{align*}
        b_m (ha(rs')) = \frac{b_m (h) T_m (s'|s,a) R_m(r|s,a)}{\sum_m b_m(h) T_m (s'|s,a) R_m(r|s,a)} = \frac{b_m (h) \PP_m (s',r|s,a) }{\PP(s',r | h,s,a)}.
    \end{align*}
    Thus, the value iteration for policy evaluation in LMDPs can be written as:
    \begin{align}
        Q_t^\pi(h,s,a) &= b(h)^\top \bar{R}(s, a) + \sum_{(s',r)} \sum_{m} b_m(h) \alpha_{t+1, s'}^{ha(rs'),\pi}(m) \PP_m(s',r|s,a), \nonumber \\
        V_t^\pi (h,s) &= \sum_{a} \pi(a|h) Q_t^\pi(h,s, a).
        \label{eq:belief_q_update}
    \end{align}
    
    Let us explain how the alpha vectors~\cite{smallwood1973optimal} can be constructed recursively from the time step $H+1$. Note that $V_{H+1} (h,s) = 0$ for any $h$ and $s$, therefore $\alpha_{H+1,s}^{h,\pi} = 0$. Then we can define the set of alpha vectors recursively such that 
    \begin{align}
        \alpha_{t, s}^{h, a, *, \pi} (m) &= \bar{R}_m(s,a), \nonumber \\
        \alpha_{t, s}^{h, a, (s',r), \pi} (m) &= \PP_m(s',r | s,a) \alpha_{t+1, s'}^{ha(s'r), \pi} (m) \qquad \forall (s,a,r,s'),
        \label{eq:alpha_update_normal}
    \end{align}
    Finally, the alpha vector for the value with respect to $h$ is constructed as
    \begin{align*}
        \alpha_{t, s}^{h, \pi} (m) = \sum_{a} \pi(a|h) \left( \alpha_{t, s}^{h, a, *, \pi} (m) + \sum_{s',r} \alpha_{t,s}^{h,a,(s',r), \pi} (m) \right).
    \end{align*}
    Note that in the construction of alpha vectors, the mapping from history to belief-state is not involved, and the value function can be represented as $V_t^\pi(h,s) = b(h)^\top \alpha_{t,s}^{h,\pi}$.
\end{proof}

Now consider the optimistic model defined in Lemma \ref{lemma:optimistic_model}. For the optimistic model, the intermediate alpha vectors are constructed with the following recursive equation:
\begin{align}
    \tilde{\alpha}_{t, s}^{h, a, *, \pi} (m) &= \Exs_{r \sim \tilde{R}^{obs}_m (r|s,a) } [r] + \tilde{R}_m^{hid} (s,a), \nonumber \\
    \tilde{\alpha}_{t, s}^{h, a, (s',r), \pi} (m) &= \tilde{T}_m(s' | s,a) \tilde{R}^{obs}_m (r | s,a) \tilde{\alpha}_{t+1, s'}^{ha(s'r), \pi} (m) \qquad \forall (s,a,r,s'),
    \label{eq:alpha_update_optimistic}
\end{align}
From the constructions of alpha vectors above, we can show the optimism in alpha vectors: 
\begin{lemma}
    \label{lemma:optimism_alpha}
    Let $\alpha_{t,s}^{h,\pi}$ and $\tilde{\alpha}_{t,s}^{h,\pi}$ be alpha vectors constructed with $\mM^*$ and $\tmM$ respectively. Then for all $t,s, h, \pi$, we have
    \begin{align*}
        \tilde{\alpha}_{t,s}^{h,\pi} \succeq \alpha_{t,s}^{h,\pi}.
    \end{align*}
\end{lemma}
The lemma implies that if the history is mapped to the same belief states in both models, then we also have the optimism in value functions. Note that in general, different models will lead each history to different belief states. At the initial time-step, however, we start from similar belief states, and we can claim Lemma \ref{lemma:optimistic_model}. The remaining proof of Lemma \ref{lemma:optimistic_model} is given in Section \ref{appendix:lemma:optimistic_model}.

\subsection{Proof of Lemma \ref{lemma:optimism_alpha}}
We show this by mathematical induction moving reverse in time from $t = H$. The inequality is trivial when $t = H+1$ since all $\alpha_{H+1, s}^{h,\pi} = \tilde{\alpha}_{H+1, s}^{h,\pi} = 0$ for any $h, \pi, s$. 
Now we investigate $\alpha_{t,s}^{h,\pi} (m)$. It is sufficient to show that for all $a \in \mA$,
\begin{align*}
    \alpha_{t,s}^{h,a,*,\pi} (m) + \sum_{s',r} \alpha_{t,s}^{h,a,(s',r),\pi} (m) \le \tilde{\alpha}_{t,s}^{h,a,*,\pi} (m) + \sum_{s',r} \tilde{\alpha}_{t,s}^{h,a,(s',r),\pi} (m).
\end{align*}

Recall equations for alpha vectors \eqref{eq:alpha_update_normal}, \eqref{eq:alpha_update_optimistic}. 
\begin{align*}
    \tilde{\alpha}_{t,s}^{h,a,*,\pi} (m) &+ \sum_{s',r} \tilde{\alpha}_{t,s}^{h,a,(s',r),\pi} (m) - \alpha_{t,s}^{h,a,*,\pi} (m) - \sum_{s',r} \alpha_{t,s}^{h,a,(s',r),\pi} (m) \\
    &\ge \left( \Exs_{r \sim \tilde{R}_m^{obs} (r|s,a)}[r] - \Exs_{r \sim R_m^{obs} (r|s,a)}[r] \right) + \tilde{R}_m^{hid}(s,a) \\
    &\quad + \sum_{s',r} \left(\tilde{T}_m(s'|s,a) \tilde{R}_m^{obs} (r|s,a) \tilde{\alpha}_{t+1, s'}^{ha(s',r),\pi} (m) - T_m(s'|s,a) R_m(r|s,a) \alpha_{t+1,s'}^{ha(s',r),\pi} (m)\right) \\
    &\ge \left( \Exs_{r \sim \tilde{R}_m^{obs} (r|s,a)}[r] - \Exs_{r \sim R_m (r|s,a)}[r] \right) + H \min\left(1, \sqrt{5 (c_R + c_T)/N_m(s,a)} \right) \\
    &\quad + \sum_{s',r} \left( \tilde{T}_m(s'|s,a) \tilde{R}_m^{obs} (r|s,a) \alpha_{t+1, s'}^{ha(s',r),\pi} (m) - T_m(s'|s,a) R_m(r|s,a) \alpha_{t+1,s'}^{ha(s',r),\pi} (m) \right),
\end{align*}
where the last inequality comes from the induction hypothesis. On the other hand, note that $\tilde{R}_m^{obs}$ and $\tilde{T}_m$ are simply empirical estimates after visiting the state-action pair $N_m(s,a)$ times. Thus, it is easy to see that with high probability,
\begin{align*}
    &\left| \Exs_{r \sim \tilde{R}_m^{obs} (r|s,a)}[r] - \Exs_{r \sim R_m^{obs} (r|s,a)}[r] \right| \le \|(\hat{R}_m - R_m) (r|s,a)\|_1 \le \sqrt{c_R / N_m(s,a)},
\end{align*}
\begin{align*}
    \sum_{s',r} &\left| \hat{T}_m(s'|s,a) \hat{R}_m(r|s,a) \alpha_{t+1,s'}^{h,a,(s',r),\pi}(m) - T_m(s'|s,a) R_m(r|s,a) \alpha_{t+1,s'}^{h,a,(s',r),\pi}(m) \right| \\
    &\le H \sum_{s',r} \left| \hat{T}_m(s'|s,a) \hat{R}_m(r|s,a) - T_m(s'|s,a) R_m(r|s,a) \right| \\
    &\le H \left( \|(\hat{T}_m - T_m) (s'|s,a)\|_1 + \| (\hat{R}_m - R_m) (r|s,a) \|_1 \right) \le H \left( \sqrt{c_R / N_m(s,a)} + \sqrt{c_T / N_m(s,a)} \right),
\end{align*}
where we used that all alpha vectors in the original system satisfies $\|\alpha_{t,s}^{h,\pi}\|_{\infty} \le H$ for all $t,s, h, \pi$. This completes the proof of Lemma \ref{lemma:optimism_alpha}.

\subsection{Proof of Lemma \ref{lemma:optimistic_model}}
\label{appendix:lemma:optimistic_model}
The remaining step is to show the optimism at the initial time. When $t = 1$, history $h$ is simply the initial state $s$. The belief state after observing the initial state is given by 
\begin{align*}
    b_m(s) &= \frac{w_m \nu_m (s)}{\sum_{s'} w_m \nu_m (s')}, \quad \tilde{b}_m (s) = \frac{w_m \tilde{\nu}_m (s)}{\sum_{s'} w_m \tilde{\nu}_m (s')}.
\end{align*}
The expected long-term reward with $\pi$ for each model is therefore
\begin{align*}
    V_{\mM^*}^\pi &= \sum_{s} \PP(s_1 = s) V(s) = \sum_{s} \PP(s_1 = s) b(s)^\top \alpha_{1,s}^{s,\pi} 
    = \sum_{s} \sum_m w_m \nu_m(s) \alpha_{1,s}^{s,\pi} (m), \\ 
    V_{\tmM}^\pi &= \sum_{s} \sum_m w_m \tilde{\nu}_m(s) \tilde{\alpha}_{1,s}^{s,\pi} (m). 
\end{align*}
Following the similar arguments, we have
\begin{align*}
    V_{\tmM}^\pi - V_{\mM^*}^\pi \ge H \sum_m w_m \sqrt{c_\nu / N(m)} - H \sum_m w_m \sum_s |\nu_m(s) - \tilde{\nu}_m (s)| \ge 0,
\end{align*}
which proves the claim of Lemma \ref{lemma:optimistic_model}.

\subsection{Proof of Theorem \ref{theorem:scenario1_regret_bound}}
\label{appendix:lucrl_regret_bound}

Let us define a few notations. Suppose $\mM = (\mS, \mA, T_m, R_m, \nu_m)$ a LMDP and a context $m$ is randomly chosen at the start of an episode following a probability distribution $(w_1, w_2, ..., w_M)$. Let $\bar{R}_m(s,a) = \Exs_{r \sim R_m(r|s,a)} [r]$ be an expected (observable) reward of taking action $a$ at $s$ in $m^{th}$ MDP. With a slight abuse in notation, we use $\Exs_{\pi, \mM} [\cdot]$ to simplify $\Exs_{m \sim (w_1, ..., w_M)} \left[ \Exs_{\pi, T_m, R_m, \nu_m} [\cdot] \Big| m \right] = \sum_{m=1}^M w_m \Exs_m^\pi [\cdot]$.

We start with the following lemma on the difference in values in terms of difference in parameters.
\begin{lemma}
    \label{lemma:value_difference_lemma}
    Let $\mM_1 = (\mS, \mA, T_m^1, R_m^1, \nu_m^1)$ and $\mM_2 = (\mS, \mA, T_m^2, R_m^2, \nu_m^2)$ be two latent MDPs with different transition, reward and initial distributions. Then for any history-dependent policy $\pi$,
    \begin{align}
        |V_{\mM_1}^\pi - V_{\mM_2}^\pi| &\le H \cdot \Exs_{\pi, \mM_2} \left[ \|(\mu_m^1 - \mu_m^2)(s)\|_1 \right] + \sum_{t=1}^H \Exs_{\pi, \mM_2} \left[ |\bar{R}_m^1(s_t,a_t) - \bar{R}_m^2(s_t,a_t)| \right] \nonumber \\
        &\quad + H \cdot \sum_{t=1}^H \Exs_{\pi, \mM_2} \left[ \|(\PP_m^1 - \PP_m^2) (s', r | s_t, a_t)\|_1 \right]. \label{eq:value_difference_eq}
    \end{align}
\end{lemma}
The proof of Lemma \ref{lemma:value_difference_lemma} is proven in \ref{appendix:subsub_value_diff_lemma}.
    
    Equipped with Lemma \ref{lemma:optimistic_model} and \ref{lemma:value_difference_lemma}, we now can prove the main theorem. We first define a few new notations. Let $\#_k (m, s,a)$ be a count of visiting $(s,a)$ in the $m^{th}$ MDP by running a policy $\pi_k$ chosen at the $k^{th}$ episode. Let $N_m^k(s,a)$ be the total number of visit at $(s,a)$ in the $m^{th}$ MDP before the beginning of $k^{th}$ episode, {\it i.e.,} $N_m^k (s,a) = \sum_{k'=1}^{k-1} \#_{k'} (m,s,a)$. Let $\mF_{k}$ be the filteration of events after running $k$ episodes. Let $\tilde{V}_k^{\pi}$ the value of the optimistic model chosen at the $k^{th}$ episode with a policy $\pi$. Let $\pi^*$ be the optimal policy for the true LMDP $\mM^*$. Finally, let us denote $(\cdot)^k$ for the model parameter in the optimistic model at $k^{th}$ episode. 
    
    The expected reward $\widetilde{\bar{R}}_m (s,a)$ in optimistic model is equivalent to $\tilde{R}_m^{hid}(s,a) + \Exs_{r \sim \tilde{R}_m^{obs}(\cdot | s,a)} [r]$. Using the Lemma \ref{lemma:value_difference_lemma}, the total regret can be rephrased as the following:
    \begin{align*}
        \sum_{k=1}^K V_{\mM^*}^{\pi^*}- V_{\mM^*}^{\pi_k} &\le \sum_{k=1}^K V_{\tmM_k}^{\pi^*} - V_{\mM^*}^{\pi_k} \le \sum_{k=1}^K V_{\tmM_k}^{\pi_k} - V_{\mM^*}^{\pi_k} \\
        &\le \sum_{k=1}^K \sum_{(m, s,a)} \Exs_{\pi_k, \mM^*} [\#_k(m,s,a) ] \cdot \Bigg(H \cdot \|(\tPP_m^k - \PP_m) (s',r | s,a)\|_1 \\
        &\qquad \qquad + \left|\tilde{R}_m^{hid, k} (s,a)\right| + \left|\Exs_{r \sim \tilde{R}_m^{obs, k} (r|s,a)} [r] - \Exs_{r \sim R_m (r|s,a)} [r] \right| \Bigg) \\
        &\ + H \cdot \sum_{k=1}^K \sum_{m} \left( \|(\tilde{\mu}_m^k - \mu_m^*) (s)\|_1 \Exs_{\mM^*} [\#_k(m)] + \sqrt{c_\nu / N^k(m)} \Exs_{\mM^*} [\#_k(m)] \right). 
    \end{align*}
    Note that $\tilde{R}_m^{hid, k}(s,a) = H \min \left(1, 5\sqrt{(c_T+c_R) / N_m^k(s,a)} \right) \ge H \|(\tPP_m^k - \PP_m) (s',r | s,a)\|_1$, and this is the dominating term. Therefore, the upper bounding equation can be reduced to
    \begin{align*}
        \sum_{k=1}^K V_{\tmM_k}^{\pi_k} - V_{\mM^*}^{\pi_k} &\le 3 H \sum_{k=1}^K \sum_{(m, s,a)} \Bigg( 5\sqrt{(c_T+c_R) / N_m^k (s,a)} \Exs_{\pi_k, \mM^*} \left [\#_k(m,s,a) \right] \Bigg) \\
        &\qquad + 2H \sum_{k=1}^K \sum_{m} \left( \sqrt{c_\nu / N^k(m)} \Exs_{\mM^*} [\#_k(m)] \right). 
    \end{align*}
        
    Observe that the expected value of $N_m^k (s,a)$ is $\sum_{k'=1}^{k-1} \Exs_{\pi_{k'}, \mM^*} [\#_{k'} (m, s,a)]$. Let this quantity $\Exs [N_m^k]$. We can check that 
    \begin{align*}
        Var \left(\#_{k} (m,s,a) | \mF_{k-1} \right) \le H \Exs_{\pi_k, \mM^*} [\#_{k} (m,s,a)]. 
    \end{align*}
    From the Bernstein's inequality for martingales, for any $(s,a)$ (ignoring constants),
    \begin{align*}
        N_m^k (s,a) \ge \Exs [N_m^k (s,a)] - c_1 \sqrt{H \Exs [N_m^k (s,a)] \log (MSAK/\eta)} - c_2 H \log (MSAK/\eta),
    \end{align*}
    for some absolute constants $c_1, c_2 > 0$ and for all $k$ and $(m,s,a)$, with probability at least $1 - \eta$. From this, we can show that
    \begin{align*}
        H \sum_{k=1}^K \sum_{(k,s,a)} &\sqrt{(c_T+c_R) / N_m^k(s,a)} \Exs [\#_k (m,s,a)] \\
        &\le H \sum_{(k,s,a)} \left(\sum_{k=1}^{k_0} \Exs [\#_k (m,s,a)] + \sum_{k=k_0+1}^K \sqrt{(c_T+c_R) / N_m^k(s,a)} \Exs[\#_k (m,s,a)] \right) \\
        &\lesssim H \sum_{(m,s,a)} \left(H\log(MSAK/\eta) + 2 \sum_{k=k_0+1}^K \sqrt{(c_T+c_R) / \Exs[N_m^k(s,a)]} \Exs[\#_k (m,s,a)] \right),
    \end{align*}
    where $k_0$ is a threshold point where the expected number of visit at $(m,s,a)$ exceeds $4H\log(MSAK/\eta)$. Note that after this point we can assume, with high probability, that $N_m^k(s,a) \ge \Exs[N_m^k(s,a)]/4$. To bound the summation of the remaining term, for a fixed $(m,s,a)$, we denote $X_k = \Exs[N_m^k(s,a)]/H$ and $x_k = \Exs [\#_k(m,s,a)] / H$. Note that $X_{k+1} = X_k + x_k$ and $x_k \le 1$. Then,
    \begin{align*}
        \sum_{k=k_0+1}^K \sqrt{1/X_k} x_k &\le \int_{X_{k_0}}^{X_K} \sqrt{\frac{1}{x - 1}} dx \le 2 \sqrt{X_K}.
    \end{align*}
    Plugging this equation, we bound the remaining terms:
    \begin{align*}
        H \sum_{(m,s,a)} &\left(H\log(MSAK/\eta) + 2 \sum_{k=k_0+1}^K \sqrt{(c_T+c_R) / \Exs[N_m^k(s,a)]} \Exs[\#_k (m,s,a)] \right), \\
        &\le H^2MSA \log(MSAK/\eta) + 4 H \sum_{(m,s,a)} \sqrt{(c_T+c_R) N_m^K(s,a)} \\
        &\le H^2MSA \log(MSAK/\eta) + 4 H \sqrt{(c_T+c_R) HMSAK}, 
    \end{align*}
    where in the last step, we used Cauchy-Schwartz inequality with $\sum_{(m,s,a)} N_m^K(s,a) = HK$. Similarly, we can show that
    \begin{align*}
        H \sum_{k=1}^K \sum_{m} &\sqrt{c_\nu / N^k(m)} \Exs [\#_k (m)] \lesssim H M \log(MSAK/\eta) + 4H \sqrt{c_\nu MK}. 
    \end{align*}
    
    Our choice of confidence parameters $c_T$ for a transition probability is $c_T = O(S \log(MSAK / \eta))$, and this is the dominating factor. Thus, the total regret is dominated by
    \begin{align*}
        H \sqrt{c_T HMSAK} \lesssim HS \sqrt{MAN \log(MSAK/\eta)},
    \end{align*}
    which in turn gives a total regret bound of $O\left( HS \sqrt{MAN \log(MSAN / \eta)} \right)$ where $N = HK$.

\subsubsection{Proof of Lemma \ref{lemma:value_difference_lemma}}
\label{appendix:subsub_value_diff_lemma}

\begin{proof}
    We first observe that
    \begin{align*}
        V_{\mM_1}^{\pi} - V_{\mM^2}^{\pi} &= \sum_{m=1}^M w_m  \left(\Exs_{m}^{1, \pi} \left[ \sum_{t=1}^H r_t \right] - \Exs_{m}^{2, \pi} \left[ \sum_{t=1}^H r_t \right] \right) \\
        &= \sum_{m=1}^M w_m \sum_{t=1}^H \left( \sum_{(s_1, a_1, r_1, ..., s_t, a_t, r_t)} r_t \PP_m^{1, \pi} (s_1, ..., r_t) - r_t \PP_m^{2, \pi}(s_1, ..., r_t) \right), 
    \end{align*}
    where $\PP_m^{1, \pi}(s_1, a_1, r_1, ..., r_{t-1}, s_t) := \nu_m^p(s_1) \Pi_{i=1}^{t-1} \pi (a_i | s_1, ..., r_{i-1}, s_i) T_m(s_{i+1}|s_i, a_i) R_m(r_{i}|s_i, a_i)$. We decompose the main difference as
    \begin{align*}
        \sum_{(s,a,r)_{1:t}} & r_t (\PP_m^{1,\pi}((s,a,r)_{1:t}) - \PP_m^{2,\pi}((s,a,r)_{1:t})) \\
        &= \sum_{((s,a,r)_{1:t})} r_t (R_m^1 (r_t | s_t, a_t) - R_m^2 (r_t | s_t, a_t)) \PP_m^{1, \pi}((s,a,r)_{1:t-1}, s_t, a_t) \\
        &\quad + \sum_{(s,a,r)_{1:t}} r_t R_m(r_t | s_t, a_t) (\PP_m^{1, \pi} - \PP_m^{2, \pi}) ((s,a,r)_{1:t-1}, s_t, a_t) \\
        &\le \sum_{s_t, a_t} \left|\Exs_{r_t \sim R_m^1 (\cdot|s_t,a_t)} [r_t] - \Exs_{r_t \sim R_m^2 (\cdot|s_t,a_t)} [r_t] \right| \PP_m^{2,\pi} (s_t, a_t) \\
        &\quad +  \| (\PP_m^{1,\pi} - \PP_m^{2,\pi}) ((s,a,r)_{1:t-1}, s_t, a_t) \|_1.
    \end{align*}
    Now we bound the total variation distance of the length $t$ histories. For notational convenience, let us denote $|\PP_1 - \PP_2| (\cdot) = |\PP_1(\cdot) - \PP_2(\cdot)|$ for any probability measures $\PP_1, \PP_2$. Then,
    \begin{align*}
        \sum_{(s,a,r)_{1:t-1}, s_t, a_t} &|\PP_m^{1,\pi} - \PP_m^{2,\pi}| ((s,a,r)_{1:t-1}, s_t, a_t) \\
        &= \sum_{(s,a,r)_{1:t-1}, s_t} |\PP_m^{1,\pi} - \PP_m^{2,\pi}| ((s,a,r)_{1:t-1}, s_t) \sum_{a_t} \pi (a_t | (s,a,r)_{1:t-1}, s_t) \\
        &= \sum_{(s,a,r)_{1:t-1}, s_t} |\PP_m^{1, \pi} - \PP_m^{2, \pi}| ((s,a,r)_{1,t-1}, s_t) \\
        &\le \sum_{(s,a,r)_{1:t-1}, s_t} |\PP_m^{1, \pi} - \PP_m^{2,\pi}| ((s,a,r)_{1:t-2}, s_{t-1}, a_{t-1}) \PP_m^1 (s_t, r_{t-1} | s_{t-1}, a_{t-1}) \\
        &\quad + \sum_{(s,a,r)_{1:t-1}, s_t} \PP_m^{2, \pi} ((s,a,r)_{1:t-2}, s_{t-1}, a_{t-1}) |\PP_m^1 - \PP_m^2| (s_{t}, r_{t-1} | s_{t-1}, a_{t-1}) \\
        &\le \sum_{(s,a,r)_{1:t-2}, s_{t-1}, a_{t-1}} |\PP_m^{1,\pi} - \PP_m^{2,\pi}| ((s,a,r)_{1:t-2}, s_{t-1}, a_{t-1}) \\
        &\quad + \sum_{(s,a,r)_{1:t-2}, s_{t-1}, a_{t-1}} \|(\PP_m^{1,\pi} - \PP_m^{2,\pi}) (s_t, r_{t-1} | s_{t-1}, a_{t-1})\|_1 \PP_m^{2,\pi} ((s,a,r)_{1:t-2}, s_{t-1}, a_{t-1}) \\
        &= \|(\PP_m^{1,\pi} - \PP_m^{2,\pi}) ((s,a,r)_{1:t-2}, s_{t-1}, a_{t-1}) \|_1 \\
        &\quad + \sum_{s_{t-1}, a_{t-1}} \|(\PP_m^{1,\pi} - \PP_m^{2,\pi}) (s_t, r_{t-1} | s_{t-1}, a_{t-1})\|_1 \PP_m^{2, \pi} (s_{t-1}, a_{t-1}).
    \end{align*}
    We can apply the same expansion recursively to bound total variation for length $t-1$ histories. Now plug this relation to the regret bound, we have
    \begin{align*}
        |V_{\mM_1}^{\pi} - V_{\mM_2}^{\pi}| &\le \sum_{m=1}^M w_m \sum_{(s,a)} \sum_{t=1}^H \left|\Exs_{r \sim R_m^1 (\cdot|s,a)} [r] - \Exs_{r \sim R_m^2 (\cdot|s,a)} [r] \right| \PP_m^{2, \pi} (s_t = s, a_t = a) \\
        &+ \sum_{m=1}^M w_m \sum_{t=1}^H \Bigg( \sum_{s} |\nu_m^{1} (s) - \nu_m^{2} (s)| \PP_m^2 (s_1 = s) \\
        &\qquad \qquad + \sum_{(s,a)} \sum_{t'=1}^t \|(\PP_m^{1, \pi} - \PP_m^{2, \pi}) (s', r | s, a) \|_1 \PP_m^{2, \pi} (s_{t'} = s, a_{t'} = a) \Bigg) \\
        &\le \sum_{m=1}^M w_m \sum_{t=1}^H \left(\Exs_{m}^{2, \pi} \left[|\bar{R}_m^1 (s_t,a_t) - \bar{R}_m^2 (s_t,a_t)| \right] \right) \\
        &\quad + H \cdot \sum_{m=1}^M w_m \left( \| (\mu_m^1 - \mu_m^2) (s) \|_1 + \sum_{t=1}^H \Exs_{m}^{2, \pi} \left[ \| (\PP_m^{1} - \PP_m^{2}) (s',r|s_t,a_t)\|_1 \right] \right),
    \end{align*}
    giving the equation \eqref{eq:value_difference_eq} as claimed.
\end{proof}

\section{Learning with Separation and Good Initialization}
\subsection{Well-Separated Condition for MDPs}
\label{appendix:well_separated_mdps}
In this subsection, we formalize a condition for {\it clusterable} mixtures of MDPs: the overlap of trajectories from different MDPs should be small in order to correctly infer the true contexts from sampled trajectories. We call the underlying MDPs {\it well-separated} if they satisfy the following separation condition:
\begin{condition}[Well-Separated MDPs]
    \label{condition:well_separated_cond}
    If a trajectory $\tau$ of length $H$ is sampled from MDP $M_{m^*}$ by running any policy $\pi \in \Pi$, we have
    \begin{align}
        \label{eq:well_separated_cond}
        \PP_{\tau \sim \mM_{m^*}, \pi} \left( \frac{\PP_{\tau \sim \mM_{m}, \pi}(\tau)} {\PP_{\tau \sim \mM_{m^*}, \pi} (\tau)} > (\epsilon_p / M)^{c_1} \right) < (\epsilon_p / M)^{c_2} \qquad \forall m \neq m^*.
    \end{align}
    for a target failure probability $\epsilon_p > 0$ where $c_1, c_2 \ge 4$ are some universal constants.
\end{condition}
Here, $\PP_{\tau \sim \mM_{m}, \pi}$ is a probability of getting a trajectory from the context $m$ with policy $\pi$. One {\it sufficient} condition that ensures the well-separated condition~\eqref{eq:well_separated_cond} is Assumption \ref{assump:delta_separation} as guaranteed by the following lemma:
\begin{lemma}
    \label{lemma:separation_lemma}
    Under the Assumption \ref{assump:delta_separation} with a constant $\delta = \Theta(1)$, if the time horizon is sufficiently long such that $H > C \cdot \delta^{-4} \log^2(1/\alpha) \log (M / \epsilon_p)$ for some absolute constant $C > 0$ and $\alpha = \delta^2 / (200 S)$, then the well-separated condition~\eqref{eq:well_separated_cond} holds true with $c_1, c_2 \ge 4$. 
\end{lemma}
Proof of Lemma~\ref{lemma:separation_lemma} is given in Appendix~\ref{appendix:separation_condition_lemma}. We remark here that we have not optimized the requirement on the time horizon $H$ to satisfy Condition \ref{condition:well_separated_cond}, and we conjecture it can be improved. We also mention here that the required time-horizon can be much shorter if the KL-divergence between distributions is larger, even though the $l_1$ distance remains the same. Finally, we remark that Assumption \ref{assump:delta_separation} is only a sufficient condition, and can be relaxed as long as Condition \ref{condition:well_separated_cond} is satisfied.

\subsection{Proof of Lemma \ref{lemma:separation_lemma}}
\label{appendix:separation_condition_lemma}
In this proof, we assume all probabilistic event is taken with true context $m^*$: unless specified, we assume $\PP(\cdot)$ and $\Exs[\cdot]$ are measured with context $m^*$.

Suppose a trajectory $\tau$ is obtained from MDP $M_{k^*}$. Let us denote the probability of getting $\tau$ from $m^{th}$ MDP by running policy $\pi$ as $\PP_{\tau \sim M_{m}, \pi}(\tau) = \PP_{m} (\tau)$. It is enough to show that 
\begin{align*}
    \ln \left( \frac{\PP_{m^*} (\tau)}{\PP_{m} (\tau)} \right) > c \ln(M / \epsilon_p), \quad \forall m \neq m^*,
\end{align*}
with probability $1 - (\epsilon_p/M)^{4}$. Note that for any history-dependent policy $\pi$, 
\begin{align*}
    \ln \left( \frac{\PP_{m^*} (\tau)}{\PP_{m} (\tau)} \right) &= \sum_{t=1}^H \ln \left( \frac{\PP_{m^*} (s_{t+1}, r_t |s_t,a_t)}{\PP_{m} (s_{t+1}, r_t |s_t,a_t)} \right).
\end{align*}
For simplicity, let us compactly denote $(s',r)$ as $o$, and $(s_{t+1}, r_t)$ as $o_t$. Note that in general, $\ln \left( \frac{\PP_{m^*} (\tau)}{\PP_{m} (\tau)} \right)$ can be unbounded due to zero probability assignments. Thus we consider a relaxed MDP that assigns non-zero probability to all observations. Let $\alpha > 0$ be sufficiently small such that $\alpha \ln(1/\alpha) < \delta^2 / (200S)$. We define similar probability distributions such that $\hat{\PP}_m$
\begin{align*}
    \hat{\PP}_m (o | s,a) = \alpha + (1 - 2\alpha S) \PP_m (o | s,a).
\end{align*}
We split the original target into three terms and bound each of them:
\begin{align*}
    \ln \left( \frac{\PP_{m^*}(\tau)}{\PP_m(\tau)} \right) = \ln \left( \frac{\PP_{m^*}(\tau)}{\hat{\PP}_m(\tau)} \right) + \ln \left( \frac{\hat{\PP}_{m}(\tau)}{\PP_m(\tau)} \right).
\end{align*}
Note that $\|\PP_m - \hat{\PP}_m (o|s,a)\|_1 \le 4 S \alpha$. For the first term, we investigate the expectation of this quantity first:
\begin{align*}
    \Exs \left[ \sum_{t=1}^H \ln \left( \frac{\PP_{m^*} (o_t | s_t,a_t)}{\hat{\PP}_{m} (o_t|s_t,a_t)} \right) \right] &= \Exs \left[ \sum_{t=1}^H \Exs \left[ \ln\left( \frac{\PP_{m^*} (o_t|s_t,a_t)}{\hat{\PP}_{m} (o_t|s_t,a_t)} \right) \Bigg| s_1, a_1, r_1, s_2, ..., r_{t-1}, s_t, a_t \right] \right] \\
    &= \Exs \left[ \sum_{t=1}^H \sum_{o_t} \PP_{m^*}(o_t | s_t, a_t) \ln \left( \frac{\PP_{m^*} (o_t|s_t,a_t)}{\hat{\PP}_{m} (o_t | s_t, a_t)} \right) \right] \\
    &= \Exs \left[ \sum_{t=1}^H D_{KL} (\PP_{m^*}(o_t | s_t, a_t), \hat{\PP}_{m}(o_t | s_t, a_t)) \right] \ge H \delta^2,
\end{align*}
where in the last step we applied Pinsker's inequality. 

Now we want to apply Chernoff-type concentration inequalities for martingales. We need the following lemma on a sub-exponential property of $\PP(X)$ on a general random variable $X$:
\begin{lemma}
    Suppose $X$ is arbitrary discrete random variable on a finite support $\mX$. Then, $\ln(1/\PP(X))$ is a sub-exponential random variable \cite{vershynin2010introduction} with Orcliz norm $\|\ln(1/\PP(X))\|_{\psi_1} = 1/e$.
\end{lemma}
\begin{proof}
    Following the definition of sub-exponential norm \cite{vershynin2010introduction}, we find $\|\ln(1/\PP(X))\|_{\psi_1} = O(1)$:
    \begin{align*}
        \|\ln(1/\PP(X))\|_{\psi_1} &= \sup_{q \ge 1} q^{-1} \Exs_X [\ln^q (1/\PP(X))]^{1/q} \\
        &= \sup_{q \ge 1} q^{-1} \left(\sum_{X \in \mX} \PP(X) \ln^q(1/\PP(X))\right)^{1/q}.
    \end{align*}
    For any $q \ge 1$, let us first find maximum value of $p \ln^q (1/p)$ for $0 \le p \le 1$. Taking a log and finding a derivative with respect to $p$ yields
    \begin{align*}
        \frac{1}{p} + q \frac{(-1/p)}{\ln(1/p)} = \frac{1}{p} (1 - q / \ln(1/p)).
    \end{align*}
    Hence $p \ln^q (1/p)$ takes a maximum at $p = e^{-q}$ with value $(q/e)^{q}$. This gives a bound for sub-exponential norm:
    \begin{align*}
        \|\ln(1/\PP(X))\|_{\psi_1} &= \sup_{q \ge 1} q^{-1} \left(\sum_{X \in \mX} \PP(X) \ln^q(1/\PP(X))\right)^{1/q} \\
        &\le \sup_{q \ge 1} q^{-1} (q/e) = 1/e.
    \end{align*}
\end{proof}

With the above Lemma and the sum of sub-exponential martingales, it is easy to verify (see Proposition 5.16 in \cite{vershynin2010introduction}) that
\begin{align*}
    \PP \left( \ln \left( \PP_{m^*}(\tau) \right) \le \Exs \left[ \ln \left( \PP_{m^*}(\tau) \right) \right] - H \epsilon_1 \right) \le \exp\left( - c \cdot \min(\epsilon_1, \epsilon_1^2) H \right),
\end{align*}
where $c > 0$ is some absolute constant, since $\ln(\PP_{m^*}(\tau)) = \sum_{t=1}^H \ln(\PP_{m^*}(o_t | s_t, a_t))$ is a sum of $H$ sub-exponential martingales. We can also apply Azuma-Hoeffeding's inequality to control the statistical deviation in $\ln(\hat{\PP}_m (\tau))$:
\begin{align*}
    \PP \left( \ln \left( \hat{\PP}_{m}(\tau) \right) \ge \Exs \left[ \ln \left( \hat{\PP}_{m}(\tau) \right) \right] + H \epsilon_2 \right) \le \exp\left( -\frac{H \epsilon_2^2}{2 \log^2 (1/\alpha)} \right),
\end{align*}
since $\hat{\PP}_m (\tau)$ is bounded by $\ln(1/\alpha)$. 

Now let $\epsilon = \epsilon_1 + \epsilon_2 = c_2 \cdot \log(1/\alpha) \sqrt{2 \log(M/\epsilon_p) / H}$ for some absolute constant $c_2 > 0$. If the time horizon $H \ge C_0 \delta^{-4} \log^2(1/\alpha) \log(M/\epsilon_p)$ for some sufficiently large constant $C_0 > 0$, then a simple algebra shows that
\begin{align*}
    \ln \left( \frac{\PP_{m^*}(\tau)}{\hat{\PP}_{m}(\tau)} \right) \ge H\delta^2 - H \epsilon \ge H\delta^2 / 2,
\end{align*}
with probability at least $1 - (\epsilon_p/M)^5$. 

Finally, we bound extra terms caused by using approximated probabilities. We note that
\begin{align*}
    \ln \left( \frac{\hat{\PP}_m(o|s,a)}{\PP_m(o|s,a)} \right) \ge -4\alpha S, \qquad \forall (o,s,a),
\end{align*}
given $2 \alpha S$ is sufficiently small. Therefore for any trajectory, we have $\ln\left(\hat{\PP}_m(\tau) / \PP_m(\tau)\right) \ge -4\alpha SH \ge -H \delta^2 / 4$. Thus we have $\ln \left(\PP_{m^*}(\tau) / \PP_m(\tau) \right) \ge H\delta^2 / 4 \ge 4 \log(M/\epsilon_p)$ with probability at least $1 - (\epsilon_p/M)^5$, which satisfies Condition \ref{condition:well_separated_cond}.

\subsection{Proof of Theorem \ref{theorem:scenario2_regret_bounds}}
\label{appendix:ucrl_em_regret}
The key component is the following lemma on the correct estimation of belief in contexts. 
\begin{lemma}
    \label{lemma:scenario2_belief_lemma}
    Let a trajectory is sampled from ${m^*}^{th}$ MDP. Under the Assumption \ref{assump:delta_separation} with good initialization $\epsilon_{init} < \delta^2 / (200 \ln(1/\alpha))$ in \eqref{eq:initialization_condition} and $H > C \cdot \delta^{-4} \log^2(1/\alpha) \log(N/\eta)$ for some universal constant $C > 0$, we have
    \begin{align*}
        \hat{b}(m^*) \ge 1 - (N/\eta)^{-4},
    \end{align*}
    with probability at least $1 - (N/\eta)^{-4}$.
\end{lemma}

Since we have estimated belief is almost approximately correct for $O(N)$ episodes with $\epsilon_p = O(1/N)$, we now have the confidence intervals for transition matrices and rewards: 
\begin{corollary}
    With probability at least $1 - 1/N$, for all round of episodes, we have
    \begin{align*}
        \| (\hat{T}_m - T_m^*) (s' | s,a)\|_1 &\le \sqrt{c_T / N_m(s,a)} + 1/N^3, \\
        \| (\hat{R}_m - R_m^*) (r| s,a) \|_1 &\le \sqrt{c_R / N_m(s,a)} + 1/N^3, \\
        \| (\hat{\nu}_m - \nu_m^*) (s) \|_1 &\le \sqrt{c_\nu / N_m(s)} + 1/N^3.
    \end{align*}
    for all $s, a, r, s'$. 
\end{corollary}
The corollary is straight-forward since the estimation error accumulated from errors in beliefs throughout $K$ episodes is at most $1/N^{3}$. If we build an optimistic model with the estimated parameters as in Lemma \ref{lemma:optimistic_model}, the optimistic value with any policy for the model satisfies
\begin{align}
    V_{\tmM}^\pi \ge V_{\mM^*}^\pi - H^2 /N^2. \label{eq:sensitivity_value}
\end{align}
Equation \eqref{eq:sensitivity_value} is a consequence of Lemma \ref{lemma:optimistic_model} and LMDP version of sensitivity analysis in partially observable environments \cite{ross2009sensitivity}, which can also be inferred from \ref{lemma:value_difference_lemma}. Following the same argument in the proof of Theorem \ref{theorem:scenario1_regret_bound}, we can also show that the estimated visit counts at $(s,a)$ is at least 
\begin{align*}
    N_m(s,a) \ge \Exs [N_m (s,a)] - c_1 \sqrt{H \Exs [N_m (s,a)] \log (MSAK/\eta)} - c_2 H \log (MSAK/\eta) - 1/N^2,
\end{align*} 
for some absolute constants $c_1, c_2 > 0$ for all $(s,a)$, with probability at least $1 - \eta$. The additional regret caused by small errors in belief estimates is therefore bounded by
\begin{align*}
    SH^2 / N^2 * N + H^2 MSA /N^2 \le 1/N,
\end{align*}
assuming $N = HK \gg H^2 S^2 M A$. The remaining steps are equivalent to the proof of Theorem \ref{theorem:scenario1_regret_bound}. 

We note here that the convergence guarantee for the online EM might be extended to allow some small probability of wrong inference of contexts. Such scenario can happen if $H$ does not scale logarithmically with total number of episodes $K$. It would be an analogous to the local convergence guarantee in a mixture of well-separated Gaussian distributions \cite{kwon2020converges, kwon2020algorithm}. The situation is even more complicated since we may run a possibly different policy in each episode. It would be an interesting question whether the online EM implementation would eventually get some good converged policy and model parameters in more general settings.

\subsection{Proof of Lemma \ref{lemma:scenario2_belief_lemma}}
\label{appendix:lemma_belief_separation}
\begin{proof}
    The proof for Lemma \ref{lemma:scenario2_belief_lemma} is an easy replication of the proof for Lemma \ref{lemma:separation_lemma}. We show that 
    \begin{equation}
        \label{eq:sum_belief_bound}
        \sum_{t=1}^H \ln \left( \frac{\alpha + (1-2\alpha S) \hat{P}_{m^*} (o_t | s_t,a_t)}{\alpha + (1-2\alpha S) \hat{P}_m (o_t | s_t,a_t)} \right) \ge 8 \log(N/\eta), 
    \end{equation}
    with probability at least $1 - (N/\eta)^{-4}$ for all $m^* \neq m$. 
    
    Let $Q_m = \alpha + (1-2\alpha S) \hat{P}_{m}$ for all $m$. Note that $\|Q_m - Q_{m^*}\|_1 \ge \delta/2$ due to the initialization condition. Furthermore, $|\ln(Q_m(o|s,a)/Q_{m^*}(o|s,a))| \le \ln(1/\alpha)$. Hence we can apply Azuma-Hoeffeding's inequality to get
    \begin{align*}
        \sum_{t=1}^H \ln \left( \frac{Q_{m^*} (o_t | s_t,a_t)}{Q_m (o_t | s_t,a_t)} \right) &\ge \Exs \left[ \sum_{t=1}^H \ln \left( \frac{Q_{m^*} (o_t | s_t,a_t)}{Q_m (o_t | s_t,a_t)} \right) \right] - \ln(1/\alpha) \sqrt{H \log(N/\eta)}
    \end{align*}
    with probability at least $1 - (MN)^{-4}$. To lower bound the expectation, we can proceed as before:
    \begin{align*}
        \Exs \Bigg[ \sum_{t=1}^H &\ln \left( \frac{Q_{m^*} (o_t | s_t,a_t)}{Q_m (o_t | s_t,a_t)} \right) \Bigg] = \Exs \left[ \sum_{t=1}^H \sum_{o_t} P_{m^*} (o_t | s_t, a_t) \ln \left( \frac{Q_{m^*} (o_t | s_t,a_t)}{Q_m (o_t | s_t,a_t)} \right) \right] \\
        &= \Exs \left[ \sum_{t=1}^H \sum_{o_t} Q_{m^*} (o_t | s_t, a_t) \ln \left( \frac{Q_{m^*} (o_t | s_t,a_t)}{Q_m (o_t | s_t,a_t)} \right)  \right] \\
        &\quad + \Exs \left[ \sum_{t=1}^H \sum_{o_t} (P_{m^*} - Q_{m^*}) (o_t | s_t, a_t) \ln \left( \frac{Q_{m^*} (o_t | s_t,a_t)}{Q_m (o_t | s_t,a_t)} \right)  \right] \\
        &\ge \Exs \left[ \sum_{t=1}^H D_{KL} (Q_{m^*} (o_t|s_t,a_t), Q_m (o_t|s_t,a_t)) \right] - \Exs \left[ \sum_{t=1}^H \|P_{m^*} - Q_{m^*}\|_1 \right] \ln(1/\alpha)\\
        &\ge H\delta^2/4 - H (2\alpha S + \epsilon_{init}) \ln(1/\alpha).
    \end{align*}
    As long as $2 \alpha S \ln(1/\alpha) \le \delta^2 / 200$ and $\epsilon_{init} \ln(1/\alpha) \le \delta^2 / 200$, we have 
    \begin{align*}
        \Exs \Bigg[ \sum_{t=1}^H &\ln \left( \frac{Q_{m^*} (o_t | s_t,a_t)}{Q_m (o_t | s_t,a_t)} \right) \Bigg] \ge H\delta^2 / 8.
    \end{align*}
    If $H \ge C \delta^{-4} \ln(1/\alpha)^2 \log(N/\eta)$ for sufficiently large constant $C > 0$, \eqref{eq:sum_belief_bound} holds with probability at least $1 - (N/\eta)^{-4}$. The implication of lemma is:
    \begin{align*}
        \hat{b}(k^*) \ge 1 - (N/\eta)^{-8} \cdot M \ge 1 - (N/\eta)^{-4},
    \end{align*}
    which proves the claimed lemma.
\end{proof}

\section{Algorithm Details for Initialization}

\label{appendix:unsupervised_algorithm}
This section provides a detailed algorithm for efficient initialization which is deferred from Section \ref{subsection:scenario3}. 

\subsection{Spectral Learning of PSRs}
\label{appendix:spectral_learning_psr}
In this subsection, we implement a spectral algorithm to learn PSR in detail. Recall that we define $P_{\mT, \mH_s} = L_s H_s$ in Condition \ref{condition:rank_test}, \ref{condition:rank_history} such that
\begin{align*}
    (P_{\mT, \mH_s})_{i,j} = \PP^\pi (\tau_i, h_{s, j}) = (L_s)_{i, :} (H_s)_{(:, j)}. 
\end{align*}
where $P_{\mT, \mH_s} \in \mR^{|\mT| \times |\mH_s|}$ is a matrix of joint probabilities of tests and histories ending with $s$. Let the top-$k$ left and right singular vectors of $P_{\mT, \mH_s}$ be $U_s$ and $V_s$ respectively. Note that with the rank conditions, $U_s^\top P_{\mT_s, \mH_s} V_s$ is invertible. We also consider a matrix of joint probabilities of histories, intermediate action-reward-next-state pairs, and tests $P_{\mT, (s',r) a, \mH_s} = L_{s'} D_{(s',r),a,s} H_s$, where $D_{(s',r),a,s} = diag(\PP_1(s',r|a,s), ..., \PP_M(s',r|a,s))$. For the simplicity in notations, we occasionally replace $(s',r)$ by a single letter $o$. The transformed PSR parameters of the LMDP can be computed by
\begin{align*}
    B_{o,a,s} = U_{s'}^\top P_{\mT, oa, \mH_s} V_s (U_s P_{\mT, \mH_s} V_s)^{-1} = (U_{s'}^\top L_{s'}) D_{o,a,s} (U_{s}^\top L_{s})^{-1}. 
\end{align*}
The initial and normalization parameters can be computed as
\begin{align*}
    b_{1,s} &= U_s^\top \PP(\mT,s_1 = s) = U_s^\top \PP(\mT | s) (w \cdot \nu) (s) =  (U_s^\top L_s) (w \cdot \nu) (s), \\
    b_{\infty,s}^{\top} &= P_{\mH_s}^\top V_s (U_s^\top P_{\mT, \mH_s} V_s)^{-1},
\end{align*}
where $P_{\mH_s} \in \mR^{|\mH_s|}$ is a vector of probability of sampling a history in $\mH_s$, and $(w\cdot \nu)(s)$ is $M$ dimensional vector with each $m^{th}$ entry $w_m \nu_m(s)$. For the normalization factor, note that $P_{\mH_s}^\top = 1^\top H_s$, therefore
\begin{align*}
    b_{\infty,s}^{\top} &= 1^\top H_s V_s (U_s^\top P_{\mT, \mH_s} V_s)^{-1} = 1^\top (U_s^\top L_s)^{-1} (U_s^\top L_s H_s V_s) (U_s^\top P_{\mT, \mH_s} V_s)^{-1} = 1^\top (U_s^\top L_s)^{-1}.
\end{align*}
It is easy to verify that 
$$\PP((s,a,r)_{1:t}, s_t) = b_{s_t, \infty}^\top B_{o_{t-1}, a_{t-1}, s_{t-1}} ... B_{o_1, a_1, s_1} b_{s_1, 1} = 1^\top D_{o_{t-1}, a_{t-1},s_{t-1}} ... D_{o_1, a_1, s_1} (w \cdot \nu) (s_1).$$ 

With Assumption \ref{assumption:sufficient_set}, we assume that a set of histories and tests $\mH, \mT$ contain all possible observations of a fixed length $l$. Furthermore, we assume that the short trajectories are collected such that each history is sampled from the sampling policy $\pi$ and then the intervening action sequence for test is uniformly randomly selected. We estimate the joint probability matrices with $N$ short trajectories such that 
\begin{align*}
    (\hat{P}_{\mH_s})_{i} = \frac{1}{N} \#(h_{s,i}), \ \  &(\hat{P}_{\mT, \mH_s})_{i,j} = \frac{A^l}{N} \#(\tau_i, h_{s,j}), \ \ (\hat{P}_{\mT, oa, \mH_s})_{i,j} = \frac{A^{l+1}}{N} \#(\tau_i, oa, h_{s,j}),
\end{align*}
where $\#$ means the number of occurrence of the event when we sample histories from the sampling policy $\pi$. For instance, $\#(\tau_i, h_{s,j})$ means the number of occurrence of $j^{th}$ history in $\mH_s$ and test resulting in $i^{th}$ test in $\mT$. Factors $A^l$ and $A^{l+1}$ are importance sampling weights for intervening actions. The initial PSR states are estimated separately: $(\hat{\PP}_{\mT, s_1=s})_{i} = \frac{A^l}{N} \#(\tau_i, s_1 = s)$, assuming we get $N$ sample trajectories from the beginning of each episode.

Now let $\hat{U}_s, \hat{V}_s$ be left and right singular vectors of $\hat{P}_{\mT,\mH_s}$. Then the spectral learning algorithm outputs parameters for PSR:
\begin{align}
    \hat{B}_{o,a,s} &= \hat{U}_{s'}^\top \hat{P}_{\mT, oa, \mH_s} \hat{V}_s (\hat{U}_{s}^\top \hat{P}_{\mT, \mH_s} \hat{V}_s)^{-1}, \nonumber \\
    \hat{b}_{\infty, s}^\top &=  \hat{P}_{\mH_s}^\top \hat{V}_s (\hat{U}_{s}^\top \hat{P}_{\mT, \mH_s} \hat{V}_s)^{-1}, \nonumber \\
    \hat{b}_{1, s} &= \hat{U}_s^\top \hat{\PP}(\mT, s_1 = s). \label{eq:PSR_construction}
\end{align}
Then, the estimated probability of a sequence with any history-dependent policy $\pi$ is given by 
\begin{equation}
    \hat{\PP}^{\pi} ((s,a,r)_{1:t-1}, s_t) = \Pi_{i=1}^{t-1} \pi(a_i | (s,a,r)_{1:i-1}, s_i) \cdot \hat{b}_{\infty, s_t}^\top \hat{B}_{o_{t-1}, a_{t-1}, s_{t-1}} ... \hat{B}_{o_1, a_1, s_1} \hat{b}_{1, s_1}. \label{eq:estimate_probability_psr}
\end{equation} 
The update of PSR states and the prediction of next observation is given as the following:
\begin{align}
    \hat{b}_1 = \hat{b}_{1, s_1}, \quad \hat{b}_{t} = \frac{\hat{B}_{o_{t-1},a_{t-1},s_{t-1}} \hat{b}_{t-1}}{\hat{b}_{\infty, s_{t}}^\top \hat{B}_{o_{t-1},a_{t-1},s_{t-1}} \hat{b}_{t-1}}, \label{eq:psr_state_update} \\
    \hat{\PP}(s',r | (s,a,r)_{1:t-1}, s_{t} || \ \bm{do} \  a) = \hat{b}_{s', \infty}^\top \hat{B}_{(s',r),a,s_{t}} \hat{b}_{t}. \label{eq:estimate_cond_prob_psr}
\end{align}

From the above procedure, we can establish a formal guarantee on the estimation of probabilities of length $t > 0$ trajectories obtained with {\it any} history-dependent policies: 
\begin{theorem}
    \label{theorem:psr_main_tv_bound}
    Suppose the LMDP and a set of histories $\mH$ and tests $\mT$ satisfies Assumption \ref{assumption:sufficient_set}. If the number of short trajectories $N = n_0$ satisfies
    \begin{align*}
        n_0 \ge C \cdot \frac{M A^{2l+1}}{p_{\pi} \sigma_{\tau}^2 \sigma_{h}^2} \frac{t^2}{\epsilon_t^2} \left( S + \frac{A^{l}}{\sigma_{h}^2} \right) \log(SA/\eta),
    \end{align*}
    where $C > 0$ is an universal constant, and $p_{\pi} = \min_s \PP^\pi (\text{end state} = s)$, then for any (history dependent) policy $\pi$, with probability at least $1 - \eta$, 
    \begin{align*}
        \| (\PP^{\pi} - \hat{\PP}^{\pi}) ((s,a,r)_{1:t-1}, s_t) \|_1 \le \epsilon_t.
    \end{align*}
\end{theorem}
We mention that the formal finite-sample guarantee of PSR learning only exists for hidden Markov models \cite{hsu2012spectral}, an extension to LMDPs requires re-derivation of the proof to include the effect of arbitrary decision making policies. For completeness, we provide the proof of Theorem \ref{theorem:psr_main_tv_bound} in Appendix \ref{appendix:proof_psr_main_tv}.

As a result of spectral learning of PSR (see a detailed procedure in Appendix \ref{appendix:spectral_learning_psr}), we can provide a key ingredient to cluster longer trajectories to recover the original LMDP model, as we show in the next subsection.
\begin{theorem}
    \label{corollary:psr_conditional_error}
    Suppose we have successfully estimated PSR parameters from the spectral learning procedure in Section \ref{appendix:spectral_learning_psr}, such that we have the following guarantee on estimated probabilities of trajectories with any history-dependent policy $\pi$:
    \begin{align*}
        \| (\PP^{\pi} - \hat{\PP}^{\pi}) ((s,a,r)_{1:t-1}, s_t) \|_1 \le \epsilon_t,
    \end{align*}
    for sufficiently small $\epsilon_t > 0$. Suppose we will execute a policy $\pi$ for $t$ time steps, observe a history $((s,a,r)_{1:t-1}, s_t)$, and then estimate probabilities of all possible future observations (or tests $o_{t:t+l-1}$) with intervening action sequence $a_{t:t+l-1}^\tau$. Then we have the following guarantee on conditional probabilities with target accuracy $\epsilon_c > 0$:
    \begin{align*}
        \| (\PP^{\pi} - \hat{\PP}^{\pi}) (o_{t:t+l-1} | (s,a,r)_{1:t-1}, s_t || \bm{do} \ a_{t:t+l-1}^\tau) \|_1 \le 4 \epsilon_{c},
    \end{align*}
    with probability at least $1 - \epsilon_{t} / \epsilon_{c}$. 
\end{theorem}

\subsection{Clustering with PSR Parameters and Separation}
\label{Appendix:clustering_with_psrs}
\begin{algorithm}[t]
    \caption{Recovery of LMDP parameters}
    \label{algo:clustering_trajectories}
    {\bf Input:} A set of short histories $\mH$ and tests $\mT$ for learning PSR, and tests $\mT'$ for clustering 

    \begin{algorithmic}[1]
        \STATE {\textcolor{gray}{ // Learn PSR parameters up to precision $o(\delta)$}}
        \STATE{Estimate PSR parameters $\{\hat{b}_{1, s}, \hat{b}_{\infty,s}, \hat{B}_{o,a,s}, \ \forall o,a,s \}$ following \eqref{eq:PSR_construction} in Appendix \ref{appendix:spectral_learning_psr} up to precision $o(\delta)$}
        \STATE {\textcolor{gray}{ // Get clusters $\{\hat{T}_m(\cdot|s,a), \hat{R}_m(\cdot|s,a)\}_{(s,a) \in \mS \times \mA, m\in [M]}$ with learned PSR parameters}}
        \STATE Initialize $V_s = \{\}$ for all $s \in \mS$
        \FOR{$n_1/3$ episodes}
            \STATE Play exploration policy $\pi$ and get a trajectory $h = (s_1, a_1, r_1, ..., s_H, a_H, r_H)$
            \STATE Get PSR state $\hat{b}_{H-1}$ at time step $H-1$ using equation \eqref{eq:psr_state_update}
            \STATE Compute $p_{H-1}(\mT^{'}) = \hat{\PP}(\mT^{'} |(s,a,r)_{1:H-2}, s_{H-1})$ using equation \eqref{eq:estimate_cond_prob_psr}
            \STATE Add $p_{H-1}$ in $V_{s_{H-1}}$
        \ENDFOR
        \FOR{all $s \in \mS$}
            \STATE Find $M$-cluster centers $C_s$ that cover all points in $V_s$ ({\it e.g.,} with $k$-means++ \cite{arthur2007k})
        \ENDFOR
        \STATE {\textcolor{gray}{ // Build each MDP model by correctly assigning contexts to estimated transition and reward probabilities}}
        \FOR{$n_1/3$ episodes}
            \STATE Play exploration policy $\pi$ until time-step $H-1$ and get a PSR state $\hat{b}_{H-1}$ at time step $H-1$
            \STATE Play an uniformly sampled action $a$ and get a PSR state $\hat{b}_{H}$ at time step $H$
            \STATE Compute $p_{H-1}(\mT^{'}), p_{H}(\mT^{'})$
            \STATE Find centers (labels) $c_{H-1} \in C_{s_{H-1}}$ and $c_{H} \in C_{s_{H}}$ such that $c_{H-1}$ and $c_H$ are the closest to $p_{H-1}$ and $p_H$ respectively.
            \STATE If $s_{H-1}$ and $s_H$ are different, let two centers $c_{H-1}$, $c_H$ be in the same context
        \ENDFOR
        \STATE If reordering of contexts are inconsistent, return FAIL
        \STATE Otherwise, construct $\hat{T}_m$ and $\hat{R}_m$ from cluster centers $\{C_s\}_{s \in \mS}$
        \FOR{$n_1/3$ episodes}
            \STATE Play exploration policy $\pi$ and get a PSR state $\hat{b}_H$ at time step $H$
            \STATE Compute $p_{H}(\mT^{'})$ and find centers $c_H \in C_{s_H}$ that is closest to $p_{H}(\mT^{'})$
            \STATE Get the context $m$ where $c_H$ belongs to, and update initial state distribution $\hat{\nu}_m$ of $m^{th}$ MDP
        \ENDFOR
    \end{algorithmic}
\end{algorithm}

We begin with the high-level idea of the algorithm that works as the following: suppose we have a new trajectory of length $H$ and the last two states are $s_{H-1}, s_H$ from unknown context $m^*$. We first consider true conditional probability given a history of $h = (s, a, o)_{1:H-2}$. Here $H > C_0 \cdot \delta^{-4} \log^2(1/\alpha) \log(N/\eta)$ is the length of episodes which satisfies the required condition for $H$ to infer the context (see Lemma \ref{lemma:separation_lemma}). $N$ is total number of episodes to be run with L-UCRL (Algorithm \ref{algorithm:lucrl}). Under Condition \ref{condition:well_separated_cond} with a failure probability $\epsilon_p = O(1/N)$, the true belief state over contexts $b$ at time step $O(H)$ satisfies
\begin{align*}
    b (m^*) \ge 1 - (\eta/N)^4. 
\end{align*}
With PSR parameters, we can estimate prediction probabilities at time step $H-1$ for any given histories. This in turn implies that for any intervening actions $a_1^\tau, ..., a_{l'}^\tau$ of length $l'$, the prediction probability given the history of length $H-1$ is nearly close to the prediction in the ${m^*}^{th}$ MDP:
\begin{align*}
    \| (\PP - \PP_{m^*}) (o_1^\tau...o_{l'}^\tau | h || \bm{do} \ a_1^\tau ... a_{l'}^\tau) \|_1 \le (\eta/N)^4,
\end{align*}
with probability at least $1 - (\eta/N)^4$. On the other hand, note that in the ${m^*}^{th}$ MDP, 
$$\PP_{m^*}(o_1^\tau...o_{l'}^\tau|h || \bm{do} \ a_1^\tau...a_{l'}^\tau) = \PP_{m^*} (o_1^\tau...o_{l'}^\tau | s_{H-1} || \bm{do} \ a_1^\tau ... a_{l'}^\tau).$$ 
Therefore, combining with Theorem \ref{corollary:psr_conditional_error}, we have that
\begin{align*}
    \| (\PP - \hat{\PP}) (o_1^\tau...o_{l'}^\tau | h || \bm{do} \ a_1^\tau ... a_{l'}^\tau) \|_1 \le 4\epsilon_c, \qquad \forall a_1^\tau ... a_{l'}^\tau \in \mA^{l'},
\end{align*}
with probability at least $1 - A^{l'} \epsilon_t / \epsilon_c$. In other words, the prediction probability estimated with PSR parameters are almost correct within error $(\eta/N)^{4} + 4\epsilon_c$ with probability at least $1 - A^{l'} \epsilon_t / \epsilon_c$.

In a slightly more general context, let $\mT^{'}$ be a set of all tests of length $l'$ with all possible intervening $A^{l'}$ action sequences where $1 \le l' \le l$. The core idea of clustering is to have the error in prediction probability $\epsilon_c$ smaller than the separation of prediction probabilities between different MDPs. Let $\delta_{psr}$ be the average $l_1$ distance between predictions of all length ${l'}$ tests such that:
\begin{align}
    \label{eq:separation_psr}
    \sum_{a_1^\tau ... a_{l'}^\tau \in \mA^{l'}} \| (\PP_{m_1} - \PP_{m_2}) (o_1^\tau...o_{l'}^\tau | s || \bm{do} \ a_1^\tau ... a_{l'}^\tau) \|_1 \ge A^{l'} \cdot \delta_{psr}, \qquad \forall s \in \mS, \ \forall m_1 \neq m_2 \in [M]. 
\end{align}
For instance, Assumption \ref{assumption:sufficient_set} alone gives that the equation \eqref{eq:separation_psr} holds with $l' = l$ and $A^{l'} \cdot \delta_{psr} \ge \sigma_\tau$, since
\begin{align*}
    \sum_{a_1^\tau ... a_{l'}^\tau \in \mA^{l'}} \| (\PP_{m_1} - \PP_{m_2}) (o_1^\tau...o_{l'}^\tau | s || \bm{do} \ a_1^\tau ... a_{l'}^\tau) \|_1 &\ge \|L_s (e_{m_1} - e_{m_2}) \|_1 \ge \|L_s (e_{m_1} - e_{m_2}) \|_2 \ge \sqrt{2} \sigma_\tau,
\end{align*}
where $e_m$ is a standard basis vector in $\mathbb{R}^M$ with $1$ at the $m^{th}$ position. If MDPs satisfy the Assumption \ref{assump:delta_separation}, then equation \eqref{eq:separation_psr} holds with $l' = 1$ and $\delta_{psr} = \delta$. The discussion in Section \ref{subsection:scenario3} applies to this case.

\begin{figure}[t]
    \centering
    \includegraphics[width=0.5\textwidth]{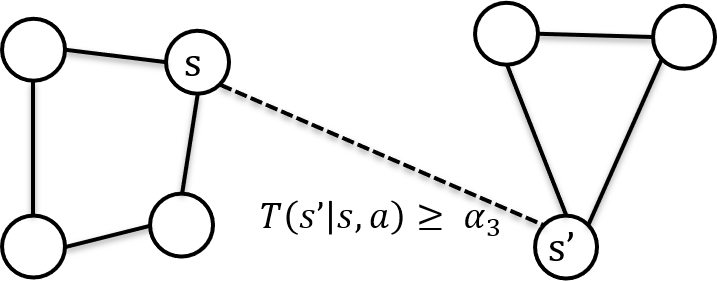}
    \caption{Connected graph constructed from an MDP with Assumption \ref{assumption:connectivity}}
    \label{fig:mdp_connectivity}
\end{figure}
Once the equation \eqref{eq:separation_psr} is given true with some $\delta_{psr} = \Theta(1)$, with high probability, we can identify the context by grouping trajectories with same ending state and similar $l'$-step predictions at time-step $H - 1$. Hence a prediction at the $(H-1)^{th}$ time step serves as a label for each trajectory. 

We are then left with recovering the full LMDP models. Even though we can cluster trajectories according to predictions conditioning on length $H-1$ histories, if we have two trajectories landed in two different states at $(H-1)^{th}$ time-step, we have no means to combine them even if they are still from the same context. In order to resolve this, our approach requires the following assumption: 
\begin{assumption}
    \label{assumption:connectivity}
    For all $m \in [M]$, let $\mG_m$ be an undirected graph where each node in $\mG_m$ corresponds to each state $s \in \mS$. Suppose we connect $(s, s')$ in $\mG_m$ (assign an edge between $s, s'$) for $s \neq s'$ if there exists at least one action $a \in \mA$ such that $T_m(s'|s,a) \ge \alpha_3$ for some $\alpha_3 > 0$. Then, $\mG_m$ is connected, {\it i.e.,} from any states there exists a path to any other states on $\mG_m$.
\end{assumption}
The high-level idea of Assumption \ref{assumption:connectivity} is to consider a graph between states as in Figure \ref{fig:mdp_connectivity}. We want to recover edges between different states $s, s'$ in $\mG_m$ so that we can assign same labels resulted from the same context but ended at different states. 

With Assumption \ref{assumption:connectivity}, if we have a trajectory that ends with last two states $(s_{H-1}, s_H) = (s,s')$ where $s \neq s'$, then we can find labels of this trajectory according to two different labeling rules at state $s$ and $s'$. Hence, we can associate labels assigned by predictions at two different states $s, s'$. Afterwards, even if we have two trajectories ending at different states from the same context, we can assign the same label to two trajectories if we have seen a connection between $(s,s')$. In other words, this step connects labels according to the same context in different states $s, s'$. Note that even if there is no direct connection, we can infer the identical context if we have a path in a graph by crossing over states that have direct connections.
\begin{remark}
    \label{remark:connectivity}
    Assumption \ref{assumption:connectivity} is satisfied if, for instance, each MDP has a finite diameter $D > 0$ \cite{jaksch2010near} where
    \begin{align*}
        D = \min_{\pi} \max_{m, s \neq s'} \Exs_m^\pi [\text{$\#$ of steps} ( s \rightarrow s')],
    \end{align*}
    $D$ is the minimum required number of expected steps in any MDP (with some deterministic memoryless policy $\pi$) to move from any state $s$ to any other states $s'$. In this case, each $\mG_m$ is connected with $\alpha_3 \ge 1/D$, since if we have some disconnected groups of states in $\mG_m$, then the diameter cannot be smaller than $1/\alpha_3$ (see also Figure \ref{fig:mdp_connectivity}). Note that in general, we only need $\alpha_3$ to be bounded below to make each graph $\mG_m$ connected for all states. With the connectivity of $\mG_m$, we can associate labels in all different states in a consistent way to resolve ambiguity in the ordering of contexts. 
\end{remark}

As we get more trajectories that end with various $s_{H-1}$ and $s_H$, whenever $s_{H-1} \neq s_H$, we can associate labels across more different states, and recover more connections (edges in $\mG_m$). Then, once every node in $\mG_m$ is connected in each context $m$, we can recover full transition and reward models for the context $m$ since we resolved the ambiguity in the ordering of labels of all different states. After we recover transition and reward models, we recover initial distribution of each MDP with a few more length $H$ trajectories. The full clustering procedure is summarized in Algorithm \ref{algo:clustering_trajectories}.

To reliably estimate the parameters with Algorithm \ref{algo:clustering_trajectories} to serve as a good initialization for Algorithm \ref{algorithm:lucrl}, we require
\begin{align*}
    \epsilon_c \le \frac{1}{4} \cdot \min(\epsilon_{init}, \delta_{psr}),  \ 
    (\eta/N)^{4} + A^{l'} \epsilon_t/\epsilon_c \le 0.01 /n_1,\end{align*}
which in turn implies the desired accuracy in total variation distance between full length $t$ trajectories: $\epsilon_t \ll A^{-l'} \epsilon_c / n_1$. In summary, total sample complexity we need for the initialization to be
\begin{align*}
    n_0 \ge C_0 \cdot \frac{H^2 M n_1^2}{\epsilon_c^2} \cdot \frac{A^{2l+2l'+1}}{p_{s}\sigma_\tau^2 \sigma_h^2} \left( S + \frac{A^l}{\sigma_h^2} \right) poly \log(N /\eta),
\end{align*}
for sufficiently large absolute constant $C_0 > 0$.

\subsubsection{Proof of Theorem \ref{theorem:final_result}}
\label{appendix:proof_clustering_with_psr}
\begin{proof}
    Let $n_1 \ge C_1 \cdot  \log(n_1) MA / (\alpha_2 \alpha_3)$, $\epsilon_c = c \cdot \min(\delta_{psr}, \epsilon_{init})$ for some sufficiently large constant $C_1 > 0$ and sufficiently small constant $c > 0$. Let $\epsilon_t = \epsilon_c / (10 n_1 A^{l'})$. Plugging this to the Theorem \ref{theorem:psr_main_tv_bound} and Theorem \ref{corollary:psr_conditional_error}, if we use $n_0$ short trajectories for learning PSR where 
    \begin{align*}
        n_0 \ge C_0 \cdot \frac{H^2 M^3}{\epsilon_c^2 \alpha_2^2 \alpha_3^2} \cdot \frac{A^{2l+2l'+3}}{p_{s}\sigma_\tau^2 \sigma_h^2} \left( S + \frac{A^l}{\sigma_h^2} \right) poly \log(N /\eta),
    \end{align*}
    then the error of the estimated conditional probability given a trajectory and a test is less than $\epsilon_c$ with probability at least $9/10$ for all $n_1$ trajectories (over the randomness of new trajectories).
    
    With Assumption \ref{assumption:reachability}, with $n_1 \gg M/\alpha_2 \log(MS)$, we can visit all states in all MDPs at least once at time step $H-1$ after $n_1/3$ episodes with probability larger than 9/10. Furthermore, for all $n_1/3$ trajectories $h_1, ..., h_{n_1/3}$ up to $H-1$ time step, we have
    \begin{align*}
        \| (\PP^\pi - \hat{\PP}^\pi) (\mT^{'} | h_i) \|_1 \le A^{l'} \epsilon_c, \qquad \forall i \in [n_1],
    \end{align*}
    with probability at least $9/10$ by union bound. Let $k_i$ and $s_i$ be the true context and ending state of $h_i$. With Assumption \ref{assump:delta_separation} and the separation Lemma \ref{lemma:separation_lemma}, we also have with probability at least $1 - \eta$ that
    \begin{align*}
        \| \PP^\pi(\mT^{'} | h_i) - \PP_{k_i}^\pi(\mT^{'} | s_i) \|_1 \le A^{l'} \cdot (\eta/N)^4, \qquad \forall i \in [n_1],
    \end{align*}
    where $N \gg n_1$ is the number of episodes to be run after initialization with Algorithm \ref{algorithm:lucrl}. Note that the prediction probabilities are $\delta_{psr}$-separated, Theorem \ref{corollary:psr_conditional_error} ensures that all possible sets of $l'$-step predictions are within error $\epsilon_c \ll \delta_{psr}$. Thus, we are guaranteed that all $h_i$s whose estimated $\hat{\PP}^\pi (\mT^{'} | h_i)$ are within $A^{l'} \epsilon_c$-error are generated from the same context. Note that with Assumption \ref{assump:delta_separation}, we have $l' = 1$ and $\delta_{psr} = \delta$. 
    
    Suppose now that we have Assumption \ref{assump:delta_separation}. In this case, we set $T'$ be a set of all possible observations of length $1$. Now we are remained with the recovery of full transition and reward models for each context. Note that same guarantees in the previous paragraph hold for predictions at the time step $H$ with probability $9/10$. With Assumption \ref{assumption:connectivity} (see also Remark \ref{remark:connectivity}), we build a connection graph for each context. That is, with $n_1 = O(MA \log(MS) / (\alpha_2 \alpha_3))$ episodes (since we need to see at least one occurrence of all edges in all contexts, {\it i.e.,} all $(m,s)$ with edges to neighborhood states $s'$ via action $a$), we have pairs of $(s_{H-1}, s_H)$ in the same trajectory where $s_{H-1}$ and $s_H$ are sufficient to recover all edges in all graphs $\mG_m$. Note that each edge occurs with probability at least $O(\alpha_2 \alpha_3 / (MA))$ and there are at most $MS^2$ edges, which gives a desired number of trajectories for clustering. 
    
    More specifically, by associating 1-step predictions at time steps $H-1$ and $H$ in the same trajectory, we can connect labels found at $s_{H-1}$ with estimated quantity $\PP_m(\cdot | s_{H-1}, a)$ and the same one found at $s_H$, as we confirm these labels are in the same true context $m$ of the trajectory. We can aggregate more sample trajectories until we recover all edges in the connection graph $\mG_m$. As long as this association results in a consistent reordering of contexts in all states, we can recover the full transition models (as well as rewards and initial distributions) for all contexts. 
    
    Now we visit every state $s$ with probability at least $\alpha_2$ at time step $H-1$ by Assumption \ref{assumption:reachability}. Then, by taking uniform action $a$ at time step $H-1$, with probability at least $\alpha_3/A$, we reveal the connection from $s$ to some other state $s'$ (which is essential for the consistent reordering of contexts) at time step $H$ by Assumption \ref{assumption:connectivity}. If we repeat this process for $n_1 = C_1 \cdot MA \log(MS) / (\alpha_2 \alpha_3)$ episodes, we can collect all necessary information for the reordering of contexts in all different states. In conclusion, Algorithm \ref{algo:clustering_trajectories} recovers $T_m$ and $R_m$ up to $\epsilon_c$-accuracy for all $m,s,a$ (not necessarily in the same order in $m$). Initial state distributions for all contexts can be similarly recovered. The entire process succeeds with probability at least $2/3$. 
\end{proof}

\section{Proofs for Spectral Learning of PSR}
In this section, we provide deferred proofs for the Lemmas used in Appendix \ref{appendix:spectral_learning_psr}. If the norm $\| \cdot \|$ is used without subscript, we mean $l_2$-norm for vectors and operator norm for matrices.

\subsection{Proof of Theorem \ref{theorem:psr_main_tv_bound}}
\label{appendix:proof_psr_main_tv}

Let us define a few notations before we get into the detail. Let us denote $p_s = 1^\top P_{H_s} = \PP(\text{end state} = s)$, and empirical counterpart $\hat{p}_s = 1^\top \hat{P}_{H_s}$, for the (empirical) probability of sampling a history ending with $s$. First, we normalize joint probability matrices:
\begin{align*}
    P_{\mT, \mH | s} = \frac{P_{\mT, \mH_s}}{p_s}, \quad P_{\mT, oa, \mH | s} = \frac{P_{\mT, oa, \mH_s}}{p_{s}}, \quad \hat{P}_{\mT, \mH | s} = \frac{\hat{P}_{\mT, \mH_s}}{p_{s}}, \quad \hat{P}_{\mT, oa, \mH | s} = \frac{\hat{P}_{\mT, oa, \mH_s}}{p_{s}}.
\end{align*}
We occasionally express unnormalized PSR states with PSR parameters $\{ (b_{\infty, s}, B_{o,a,s}, b_{1, s}) \}$ as given a history $(s,a,o)_{1:t-1}$ as
\begin{align*}
    {b}_{t, s_1} = {B}_{(o,a,s)_{t-1}} {B}_{(o,a,s)_{t-2}} ... {B}_{(o,a,s)_{1}} {b}_{1,s_1} = {B}_{(o,a,s)_{t-1:1}} {b}_{1,s_1}.
\end{align*}
The empirical counterpart will be defined similarly with $\hat{\cdot}$ on the top. We often concisely use $h_{t}$ instead of $(s,a,o)_{1:t-1} = (s,a,r)_{1:t-1} s_t$. We represent the probability of choosing actions $a_1, ..., a_{t-1}$ when the history is $h_{t-1}$ as
\begin{align*}
    \pi(a_{1:t-1}|h_{t-1}) = \pi(a_1|h_1) \pi(a_2|h_{2}) ... \pi(a_{t-1}|h_{t-1}).
\end{align*}

Now suppose that empirical estimates of probability matrices satisfy the following:
\begin{align*}
    \|P_{\mH_s} - \hat{P}_{\mH_s}\|_2 &\le \epsilon_{0,s} \\  
    \|\PP(\mT, s_1 = s) - \hat{\PP}(\mT, s_1 = s)\|_2 &\le \epsilon_{1,s} \\
    \|P_{\mT, \mH |s} - \hat{P}_{\mT, \mH | s} \|_2 &\le \epsilon_{2,s} \\
    \|P_{\mT, o a, \mH |s} - \hat{P}_{\mT, o a, \mH |s} \|_2 &\le \epsilon_{3,oas},
\end{align*}
for all $s, a, o$. The following lemma shows how the error in estimated matrices affects the accuracy of PSR parameters.
\begin{lemma}
    \label{lemma:appendix_PSR_l2_error}
    Let the true transformed PSR parameters with $\hat{U_s}, \hat{V}_s$ be
    \begin{align*}
        \tilde{B}_{o,a,s} &= \hat{U}_{s'}^\top P_{\mT, oas, \mH_s} \hat{V}_s (\hat{U}_{s}^\top P_{\mT, \mH_s} \hat{V}_s)^{-1}, \\
        \tilde{b}_{\infty, s} &= P_{\mH_s}^\top \hat{V}_s (\hat{U}_{s}^\top P_{\mT, \mH_s} \hat{V}_s)^{-1}, \\
        \tilde{b}_{1, s} &= \hat{U}_s^\top \PP(\mT, s_1 = s),
    \end{align*}
    for all $s, a, o$. Let $\sigma_{M,s}$ be the minimum ($M^{th}$) singular value of $\sigma_M(\hat{U}_s^\top P_{\mT, \mH|s} \hat{V}_s)$. Then, we have that
    \begin{align*}
        \|\tilde{B}_{o,a,s} - \hat{B}_{o,a,s}\|_2 &\le \frac{\epsilon_{3,oas}}{\sigma_{M,s}} + \sqrt{A^l} \PP^\pi (o | s || \bm{do} \ a) \frac{2 \epsilon_{2,s}}{\sigma_{M,s}^2}, \\
        \|\tilde{b}_{s,\infty} - \hat{b}_{s, \infty}\|_2 &\le \frac{\epsilon_{0,s}}{p_s \sigma_{M,s} } + \frac{2 \epsilon_{2,s}}{\sigma_{M,s}^2}, \\
        \|\tilde{b}_{s, 1} - \hat{b}_{s,1}\|_2 &\le \epsilon_{1,s},
    \end{align*}
    where $\PP^\pi (\cdot)$ is the probability of events when we sample histories with the exploration policy $\pi$. 
\end{lemma}

The proofs of helping lemmas will be proved at the last of this subsection. We define the following quantities with error bounds similarly as in \cite{hsu2012spectral}:
\begin{align}
    \delta_{\infty, s} &= \|L_s^\top \hat{U}_s (\tilde{b}_{\infty,s} - \hat{b}_{\infty, s})\|_{\infty} \le \|L_s^\top\|_{\infty,2} \| \tilde{b}_{\infty,s} - \hat{b}_{\infty, s}\|_2 \le \sqrt{A^l} \| \tilde{b}_{\infty,s} - \hat{b}_{\infty, s}\|_2, \nonumber \\
    \delta_{1, s} &= \|(\hat{U}_s^\top L_s)^{-1} (\tilde{b}_{1,s} - \hat{b}_{1, s})\|_{1} \le \sqrt{M} \|\tilde{b}_{1,s} - \hat{b}_{1,s}\|_2 / \sigma_M(\hat{U_s}^\top L_s), \nonumber \\
    \Delta_{o,a,s} &= \|(\hat{U}_s^\top L_s)^{-1} (\tilde{B}_{o,a,s} - \hat{B}_{o,a,s}) (\hat{U}_s^\top L_s) \|_1 \le \sqrt{M} \|\tilde{B}_{o,a,s} - \hat{B}_{o,a,s}\|_{2} / \sigma_M(\hat{U}_s^\top L_s), \nonumber \\
    \Delta &= \max_{a,s} \sum_o \Delta_{o,a,s}, \ \delta_{\infty} = \max_s \delta_{\infty, s}, \ \delta_1 = \max_s \delta_{1,s}. \label{eq:psr_delta_bound}
\end{align}
We let $\epsilon_t = \delta_{\infty} + (1+\delta_{\infty}) ((1+\Delta)^t \delta_1 + (1+\Delta)^t - 1)$. We first note that for any fixed action sequence $a_{1:t-1}$, it holds that
\begin{align*}
    \sum_{(s,o)_{1:t-1}} &|\tilde{b}_{\infty,s}^\top \tilde{B}_{(o,a,s)_{t-1:1}} \tilde{b}_{1,s_1} - \hat{b}_{\infty,s}^\top \hat{B}_{(o,a,s)_{t-1:1}} \hat{b}_{1,s_1}| \\
    &\le \delta_{\infty} + (1 + \delta_{\infty}) ((1+\Delta)^t \delta_{1} + (1+\Delta)^t - 1).
\end{align*}
This equation is a direct consequence of the Lemma 12 in \cite{hsu2012spectral}. However, here we aim to get the bound for {\it all} history dependent policy, hence we need to establish the theorem by re-deriving the induction hypothesis with considering the policy. We now bound the original equation. Observe first that
\begin{align*}
    &\sum_{(s,a,r)_{1:t-1}, s_t} |\PP^{\pi} ((s,a,r)_{1:t-1}, s_t) - \hat{\PP}^{\pi} ((s,a,r)_{1:t-1}, s_t) | \\
    &= \sum_{(s,a,r)_{1:t-1}, s_t} \pi(a_{1:t-1} | h_{t-1})  |\tilde{b}_{\infty,s}^\top \tilde{B}_{(o,a,s)_{t-1:1}} \tilde{b}_{1,s_1} - \hat{b}_{\infty,s}^\top \hat{B}_{(o,a,s)_{t-1:1}} \hat{b}_{1,s_1}|.
\end{align*}
Following the steps in \cite{hsu2012spectral}, for each $s_1$, we will prove the following Lemma: 
\begin{lemma}
    \label{lemma:appendix_psr_sum_main}
    For any $t$, we have
    \begin{align}
        \sum_{(s,a,o)_{1:t-1}} \pi(a_{1:t-1} | h_{t-1}) \|(\hat{U}_{s_t}^\top L_{s_t})^{-1} (\tilde{B}_{(o,a,s)_{t-1:1}} \tilde{b}_{1,s_1} - & \hat{B}_{(o,a,s)_{t-1:1}} \hat{b}_{1,s_1}) \|_1 \nonumber \\
        &\qquad \le (1+\Delta)^{t} \delta_1 + (1 + \Delta)^{t} - 1. \label{eq:l1_sum_UL_inv}
    \end{align}
\end{lemma}
We are now ready to prove the original claim. Let us denote $\tilde{b}_{t,s_1} = \tilde{B}_{(o,a,s)_{t-1:1}} \tilde{b}_{1,s_1}$ and $\hat{b}_{t,s_1} = \hat{B}_{(o,a,s)_{t-1:1}} \hat{b}_{1,s_1}$. The remaining step is to involve the effect of error in $\hat{b}_{\infty, s_t}$. Following the similar steps, we decompose the summation as:
\begin{align*}
    \sum_{(s,a,o)_{1:t-1}} &\pi(a_{1:t-1} | h_{t-1}) |\tilde{b}_{\infty,s_t}^\top \tilde{B}_{(o,a,s)_{t-1:1}} \tilde{b}_{1,s_1} - \hat{b}_{\infty,s_t}^\top \hat{B}_{(o,a,s)_{t-1:1}} \hat{b}_{1,s_1}| \\
    &= \sum_{(s,a,o)_{1:t-1}} \pi(a_{1:t-1} | h_{t-1}) |(\tilde{b}_{\infty,s_t} - \hat{b}_{\infty,s_t})^\top (\hat{U}_{s_t}^\top L_{s_t}) (\hat{U}_{s_t}^\top L_{s_t})^{-1} \tilde{b}_{t,s_1}| \\
    &\ + \sum_{(s,a,o)_{1:t-1}} \pi(a_{1:t-1} | h_{t-1}) |(\tilde{b}_{\infty,s_t} - \hat{b}_{\infty,s_t})^\top (\hat{U}_{s_t}^\top L_{s_t}) (\hat{U}_{s_t}^\top L_{s_t})^{-1} (\tilde{b}_{t,s_1} - \hat{b}_{t,s_1})| \\
    &\ + \sum_{(s,a,o)_{1:t-1}} \pi(a_{1:t-1} | h_{t-1}) |\tilde{b}_{\infty,s_t}^\top (\hat{U}_{s_t}^\top L_{s_t}) (\hat{U}_{s_t}^\top L_{s_t})^{-1} (\tilde{b}_{t,s_1} - \hat{b}_{t,s_1})|.
\end{align*}
For the first term, 
\begin{align*}
    \sum_{(s,a,o)_{1:t-1}} &\pi(a_{1:t-1} | h_{t-1}) |(\tilde{b}_{\infty,s_t} - \hat{b}_{\infty,s_t})^\top (\hat{U}_{s_t}^\top L_{s_t}) (\hat{U}_{s_t}^\top L_{s_t})^{-1} \tilde{b}_{t,s_1}| \\
    &\le \sum_{(s,a,o)_{1:t-1}} \pi(a_{1:t-1} | h_{t-1}) \|(\tilde{b}_{\infty,s_t} - \hat{b}_{\infty,s_t})^\top (\hat{U}_{s_t}^\top L_{s_t})\|_{\infty} \|(\hat{U}_{s_t}^\top L_{s_t})^{-1} \tilde{b}_{t,s_1}\|_1 \\
    &\le \sum_{(s,a,o)_{1:t-1}} \pi(a_{1:t-1} | h_{t-1}) \delta_{\infty, s_t} \|(\hat{U}_{s_t}^\top L_{s_t})^{-1} \tilde{b}_{t,s_1}\|_1 \\
    &\le \delta_{\infty} \sum_{(s,a,o)_{1:t-1}} \pi(a_{1:t-1} | h_{t-1}) \PP(h_{t} | a_{1:t-1}) \le \delta_{\infty}.
\end{align*}
Following the similar step, the second term is bounded by $\delta_{\infty} ((1+\Delta)^{t} \delta_1 + (1 + \Delta)^{t} - 1)$. For the last term, note that $\tilde{b}_{\infty,s_t}^\top (\hat{U}_{s_t}^\top L_{s_t}) = 1^\top$. Therefore, 
\begin{align*}
    \sum_{(s,a,o)_{1:t-1}} &\pi(a_{1:t-1} | h_{t-1}) |\tilde{b}_{\infty,s_t}^\top (\hat{U}_{s_t}^\top L_{s_t}) (\hat{U}_{s_t}^\top L_{s_t})^{-1} (\tilde{b}_{t,s_1} - \hat{b}_{t,s_1})| \\
    &\le \sum_{(s,a,o)_{1:t-1}} \pi(a_{1:t-1} | h_{t-1}) \|(\hat{U}_{s_t}^\top L_{s_t})^{-1} (\tilde{b}_{t,s_1} - \hat{b}_{t,s_1})\|_1 \\
    &\le ((1+\Delta)^{t} \delta_1 + (1 + \Delta)^{t} - 1). 
\end{align*}
Therefore, we conclude that
\begin{align*}
    \sum_{(s,a,o)_{1:t-1}} |\PP^{\pi} ((s,a,o)_{1:t-1}) - \hat{\PP}^{\pi} ((s,a,o)_{1:t-1}) | \le \delta_{\infty} + (1 + \delta_{\infty}) ((1+\Delta)^t \delta_{1} + (1+\Delta)^t - 1).
\end{align*}
Finally, in other to make the error term smaller than $\epsilon_t$, we want the followings:
\begin{align*}
    \delta_{\infty} \le \epsilon_t / 8, 
    \Delta \le \epsilon_t / 4t,
    \delta_1 \le \epsilon_t / 4.
\end{align*}
We need the following lemma on finite-sample error in estimated probability matrices:
\begin{lemma}
    \label{lemma:psr_finite_sample_error}
    For a sufficiently large constant $C > 0$, the errors in empirical estimates of the probability matrices are bounded by
    \begin{align*}
        \|P_{\mH_s} - \hat{P}_{\mH_s} \|_2 &\le C \sqrt{\frac{p_s}{N} \log(SA/\eta)}, \\
        \|P_{\mT,\mH_s} - \hat{P}_{\mT,\mH_s} \|_2 &\le C A^l \sqrt{\frac{p_s}{N} \log(SA/\eta)}, \\
        \|\PP(\mT,s_1 = s) - \hat{\PP} (\mT, s_1 = s) \|_2 &\le C A^l \sqrt{\frac{\PP(s_1 = s)}{N} \log(SA/\eta)}, \\
        \|P_{\mT,oa, \mH_s} - \hat{P}_{\mT, oa, \mH_s} \|_2 &\le C A^{l+1} \left( \sqrt{\frac{\PP^\pi (o|s || \bm{do} \ a) p_s}{NA} \log(SA/\eta)} + \frac{\log(SA/\eta)}{N}\right), 
    \end{align*}
    for all $s, a, o$ with probability at least $1 - \eta$. 
\end{lemma}
This lemma follows the same concentration argument to Proposition 19 in \cite{hsu2012spectral} using McDiarmid's inequality. The proofs of three lemmas are given at the end of this subsection. With Lemma \ref{lemma:appendix_PSR_l2_error}, \ref{lemma:psr_finite_sample_error} and equation \eqref{eq:psr_delta_bound}, we now decide the sample size. For $\Delta$, 
\begin{align*}
    \Delta &\le \sum_{o} \Delta_{o,a,s} \\
    &\le \frac{\sqrt{M}}{\sigma_M(\hat{U}_s^\top L_s)} \left( \sum_{o} \frac{\epsilon_{3,oas}}{\sigma_{M,s}} + \sqrt{A^l} \PP^\pi (o | s || \bm{do} \ a) \frac{2\epsilon_{2,s}}{\sigma^2_{M,s}} \right) \\
    &\le \frac{\sqrt{M}}{\sigma_M(\hat{U}_s^\top L_s)} \left( \frac{\sum_{o} \epsilon_{3,oas}}{\sigma_{M,s}} + \sqrt{A^l} \frac{2\epsilon_{2,s}}{\sigma^2_{M,s}} \right). \\
\end{align*}
The summation of $\epsilon_{3,oas}$ is bounded by
\begin{align*}
    \sum_{o} \epsilon_{3,oas} &\le C A^{l+1} \sum_o \left(\sqrt{\frac{\PP^\pi (o|s || \bm{do} \ a)}{NA p_s} \log(SA/\eta)} + \frac{\log(SA/\eta)}{Np_s} \right) \\
    &\le C A^{l+1} \left(\sqrt{\frac{2S}{NA p_s} \log(SA/\eta)} + \frac{2S\log(SA/\eta)}{N p_s} \right).
\end{align*}
Also note that 
\begin{align*}
    \epsilon_{2,s} &\le C A^l \sqrt{\frac{\log(SA/\eta)}{N p_s}}.
\end{align*}
In order to have $\Delta < \epsilon_t / (4t)$, the sample size should be at least 
\begin{align*}
    N \ge C' \cdot K \frac{t^2}{\epsilon_t^2} \left( \frac{A^{2l+1} S}{p_s \sigma_M^2(\hat{U}_s^\top L_s) \sigma_{M,s}^2} + \frac{A^{3l+1} }{p_s \sigma_K^2(\hat{U}_s^\top L_s) \sigma_{M,s}^4} \right) \log(SA/\eta),
\end{align*}
for some large constant $C' > 0$. 

Finally, $\sigma_{M,s} = \sigma_M (\hat{U}_s^\top P_{\mT, \mH|s} \hat{V}_s) \ge (1 - \epsilon_0) \sigma_M(P_{\mT, \mH|s})$ where $\epsilon_0 = \epsilon_{2,s}^2 / ((1-\epsilon_t) \sigma_M(P_{\mT, \mH|s}))^2$ by applying Lemma \ref{lemma:theorem_singular_value} twice. Hence, as long as $N \gg 1/\sigma_M(P_{\mT, \mH|s})$, it holds that $\sigma_{M,s} \ge \sigma_M(P_{\mT, \mH|s}) / 2$. Similarly, we have $\sigma_M(\hat{U}_s^\top L_s) \ge \sigma_M(L_s)/2$. Plugging this inequality in the sample complexity completes the Theorem \ref{theorem:psr_main_tv_bound}.

\subsection{Proof of Theorem \ref{corollary:psr_conditional_error}}
\begin{proof}
    We first define an extended policy $\pi'$ which runs the given policy $\pi$ for $t$ times and play intervening action sequences $a_{t}...a_{t+l-1}$. Let us denote $o = (r,s')$ to represent a pair of reward and next state compactly. A simple corollary of Theorem \ref{theorem:psr_main_tv_bound} is the following lemma:
    \begin{lemma}
        \label{lemma:psr_cond_helper_sum}
        With the estimated PSR parameters in Theorem \ref{theorem:psr_main_tv_bound}, for any given trajectory $(s,a,o)_{1:t-1}$, the following holds:
        \begin{align*}
            \sum_{a_t, r_{t}, ..., s_{t+l}} & \pi(a_{1:t-1} | h_{t-1}) | b_{\infty, s_{t+l}}^\top B_{(o,a,s)_{t+l-1:1}} b_{1, s_1} - \hat{b}_{\infty, s_{t+l}}^\top \hat{B}_{(o,a,s)_{t+l-1:1}} \hat{b}_{1, s_1} | \\
            &\le \epsilon_l \PP^\pi ((s,a,o)_{1:t-1}) + 2 \pi(a_{1:t-1} | h_{t-1}) \|(\hat{U}_{s_{t}}^\top  L_{s_{t}})^{-1} (\tilde{b}_{t,s_1} - \hat{b}_{t,s_1}) \|_1.
        \end{align*}
    \end{lemma}
    On top of this lemma, we also have the following lemma that bounds the probability of bad events in which the error in estimated probability can be arbitrarily large:
    \begin{lemma}
        \label{lemma:psr_cond_good_event}
        For any history-dependent policy $\pi$, with the PSR parameters guaranteed in Theorem \ref{theorem:psr_main_tv_bound}, we have
        \begin{align*}
            |\hat{\PP}^{\pi'} ((s,a,o)_{1:t-1}) - \PP^{\pi} ((s,a,o)_{1:t-1})| \le \epsilon_c \PP^\pi ((s,a,o)_{1:t-1} ), \\
            \pi(a_{1:t-1} | h_{t-1}) \|(\hat{U}_{s_{t}}^\top  L_{s_{t}})^{-1} (\tilde{b}_{t,s_1} - \hat{b}_{t,s_1}) \|_1 \le \epsilon_c \PP^\pi ((s,a,o)_{1:t-1}),
        \end{align*}
        with probability at least $1 - \epsilon_t / \epsilon_c$.
    \end{lemma}
    
    By the definition of conditional test probability, note that
    \begin{align*}
        | \hat{\PP}^{\pi'} (\tau| (s,a,o)_{1:t-1} ) - \PP^{\pi'} (\tau| (s,a,o)_{1:t-1}) | &= \left| \frac{\hat{\PP}^{\pi'} (\tau, (s,a,o)_{1:t-1})}{\hat{\PP}^{\pi} ((s,a,o)_{1:t-1}) } - \frac{\PP^{\pi'} (\tau, (s,a,o)_{1:t-1})}{\PP^{\pi} ((s,a,o)_{1:t-1})} \right|,
    \end{align*}
    which is less than 
    \begin{align*}
        \frac{(1 + \epsilon_c)}{\PP^{\pi} ((s,a,o)_{1:t-1})} \left| \hat{\PP}^{\pi'} (\tau, (s,a,o)_{1:t-1}) - \PP^{\pi'} (\tau, (s,a,o)_{1:t-1}) \right| + \epsilon_c \frac{\PP^{\pi'} (\tau, (s,a,o)_{1:t-1})}{\PP^{\pi} ((s,a,o)_{1:t-1})},
    \end{align*}
    with probability at least $1 - \epsilon_t / \epsilon_c$. Now we sum over all possible trajectories $\tau$ in $\mO$ with intervening actions $a_1 ... a_l$ after observing $(s,a,o)_{1:t-1}$. Under the good event guaranteed in Lemma \ref{lemma:psr_cond_good_event}, the summation over all possible future trajectories is less than
    \begin{align*}
        \frac{(1 + \epsilon_c)}{\PP^{\pi} ((s,a,o)_{1:t-1})} \left((\epsilon_t + 2\epsilon_c) \PP^{\pi} ((s,a,o)_{1:t-1}) + \epsilon_c \PP^{\pi} ((s,a,o)_{1:t-1}) \right) \le 4\epsilon_c,
    \end{align*}
    from Lemma \ref{lemma:psr_cond_helper_sum} and \ref{lemma:psr_cond_good_event}. Therefore, for a fixed intervening action sequences $a_{t}, ... , a_{t+l-1}$, we can conclude that
    \begin{align*}
        \| \PP^{\pi} (\mO | (s,a,o)_{1:t-1} ||\bm{do} \ a_{t:t+l-1}) - \hat{\PP}^{\pi} (\mO | (s,a,o)_{1:t-1} ||\bm{do}\  a_{t:t+l-1}) \|_1 \le 4 \epsilon_{c},
    \end{align*}
    with probability at least $1 - \epsilon_{t} / \epsilon_c$.
\end{proof}

\subsection{Proof of Lemma \ref{lemma:appendix_PSR_l2_error}}
\begin{proof}
    The proof of the lemma can be done by unfolding expressions:
    \begin{align*}
        \|\tilde{B}_{o,a,s} - \hat{B}_{o,a,s}\| &= \|\hat{U}_{s'}^\top P_{\mT, oa, \mH|s} \hat{V}_s (\hat{U}_{s}^\top P_{\mT, \mH|s} \hat{V}_s)^{-1} - \hat{U}_{s'}^\top P_{\mT, oa, \mH|s} \hat{V}_s (\hat{U}_{s}^\top \hat{P}_{\mT, \mH|s} \hat{V}_s)^{-1} \| \\ 
        &\le \|(\hat{U}_{s'}^\top (P_{\mT, oa, \mH|s} - \hat{P}_{\mT, oa, \mH|s}) \hat{V}_s) (\hat{U}_{s}^\top P_{\mT, \mH|s} \hat{V}_s)^{-1} \| \\
        &\quad + \|(\hat{U}_{s'}^\top P_{\mT, oa, \mH|s} \hat{V}_s) ((\hat{U}_{s}^\top P_{\mT, \mH|s} \hat{V}_s)^{-1} - (\hat{U}_{s}^\top \hat{P}_{\mT, \mH|s} \hat{V}_s)^{-1} ) \| \\
        &\le \frac{\epsilon_{3,oas}}{\sigma_{M,s}} + \| P_{\mT, oa, \mH|s} \|_2 \frac{2 \epsilon_{2,s}}{\sigma_{M,s}^2} \le \frac{\epsilon_{3,oas}}{\sigma_{M,s}} + \PP^\pi (o | s || \bm{do} \ a) \sqrt{A^l} \frac{2 \epsilon_{2,s}}{\sigma_{M,s}^2},
    \end{align*}
    where we used Lemma \ref{lemma:theorem_matrix_perturb} from matrix perturbation theory for the second inequality, and 
    \begin{align*}
        \|P_{\mT, ao, \mH_s}\|_2 &\le \sqrt{\sum_{\tau \in \mT, h \in \mH_s} \PP^\pi (o o_1^\tau o_2^\tau ... o_l^\tau | h || \bm{do} \ a a_1^\tau ... a_l^\tau)^2  \PP^\pi(h)^2 } \\
        &\le \sqrt{ \sum_{a_1,a_2,...,a_l} \sum_{o_1,...,o_l} \sum_{h \in \mH_s} \PP^\pi (oo_1...a_l | h || \bm{do} \  aa_1...a_l)^2 \PP^\pi (h)^2 } \\
        &\le \sqrt{ \sum_{a_1,a_2,...,a_l} \PP^\pi(o|h||\bm{do} \ a)^2 \sum_{o_1,...,o_l} \sum_{h \in \mH_s} \PP^\pi (o_1...o_l | hao || \bm{do} \  a_1...a_l)^2 \PP^\pi (h)^2 } \\
        &\le \PP^\pi(o|h||\bm{do} \ a) \sqrt{ \sum_{a_1,a_2,...,a_l} \sum_{h \in \mH_s} \PP^\pi(h)^2 } = \PP^\pi(o|h||\bm{do} \ a) \sqrt{ A^l} \sqrt{\sum_{h \in \mH_s} \PP^\pi(h)^2 } \\
        &\le \PP^\pi(o|h||\bm{do} \ a) \sqrt{ A^l} \sum_{h \in \mH_s} \PP^\pi(h) = \PP^\pi(o|h||\bm{do} \ a) \sqrt{ A^l} p_s,
    \end{align*}
    therefore $\|P_{\mT, oa, \mH|s}\| \le \PP^\pi(o|h||\bm{do} \ a) \sqrt{A^l}$ for the last inequality. For initial and normalization parameters,
    \begin{align*}
        \|\tilde{b}_{\infty,s} - \hat{b}_{\infty,s}\| &\le \|(P_{\mH_s} - \hat{P}_{\mH_s})^\top \hat{V}_s (\hat{U}_{s}^\top P_{\mT, \mH_s} \hat{V}_s)^{-1} \| + \|P_{\mH_s}^\top \hat{V}_s ((\hat{U}_{s}^\top P_{\mT, \mH_s} \hat{V}_s)^{-1} - (\hat{U}_{s}^\top \hat{P}_{\mT, \mH_s} \hat{V}_s)^{-1} ) \| \\
        &\le \frac{\epsilon_{0,s}}{\sigma_M (P_{\mT, \mH_s})} + \| P_{\mH_s}/\hat{p}_s \|_2 \frac{2 \epsilon_{2,s}}{\sigma_M (P_{\mT, \mH|s})^2} \le \frac{\epsilon_{0,s}}{p_s \sigma_M (P_{\mT, \mH|s})} + \frac{2 \epsilon_{2,s}}{\sigma_M (P_{\mT, \mH|s})^2}.
    \end{align*}
    \begin{align*}
        \|\tilde{b}_{s,1} - \hat{b}_{s,1}\| &\le \epsilon_{1,s}.
    \end{align*}
\end{proof}

\subsection{Proof of Lemma \ref{lemma:appendix_psr_sum_main}}
\begin{proof}
    We show this lemma by induction on $t$. For $t = 1$, we bound $\|(\hat{U}_{s_1} L_{s_1})^{-1} (\tilde{b}_{1,s_1} - \hat{b}_{1,s_1})\|_1 \le \delta_{1,s_1}$ by definition. Now assume it holds for $t-1$ and check the induction hypothesis. 
    \begin{align*}
        &\sum_{(s,a,o)_{1:t-1}} \pi(a_{1:t-1} | h_{t-1}) \|(\hat{U}_{s_t}^\top L_{s_t})^{-1}  (\tilde{B}_{(o,a,s)_{t-1:1}} \tilde{b}_{1,s_1} - \hat{B}_{(o,a,s)_{t-1:1}} \hat{b}_{1,s_1}) \|_1 \\
        &= \sum_{(s,a,o)_{1:t-2}} \sum_{a_{t-1}, o_{t-1}} \pi(a_{t-1} | h_{t-1}) \pi(a_{1:t-2} | h_{t-2}) \|(\hat{U}_{s_t}^\top L_{s_t})^{-1}  (\tilde{B}_{(o,a,s)_{t-1:1}} \tilde{b}_{1,s_1} - \hat{B}_{(o,a,s)_{t-1:1}} \hat{b}_{1,s_1}) \|_1 \\
        &= \sum_{a_{t-1}} \pi(a_{t-1} | h_{t-1}) \sum_{o_{t-1}} \sum_{(s,a,o)_{1:t-2}} \pi(a_{1:t-2} | h_{t-2}) \|(\hat{U}_{s_t}^\top L_{s_t})^{-1}  (\tilde{B}_{(o,a,s)_{t-1:1}} \tilde{b}_{1,s_1} - \hat{B}_{(o,a,s)_{t-1:1}} \hat{b}_{1,s_1}) \|_1 \\
        &= \sum_{a_{t-1}} \pi(a_{t-1} | h_{t-1}) \sum_{o_{t-1}} \sum_{(s,a,o)_{1:t-2}} \pi(a_{1:t-2} | h_{t-2}) \|(\hat{U}_{s_t}^\top L_{s_t})^{-1}  (\tilde{b}_{t,s_1} -   \hat{b}_{t,s_1}) \|_1.
    \end{align*}
    We investigate the inside sum by decomposing $\|(\hat{U}_{s_t}^\top L_{s_t})^{-1}  (\tilde{b}_{t,s_1} -   \hat{b}_{t,s_1}) \|_1$ as
    \begin{align*}
        \|(\hat{U}_{s_t}^\top L_{s_t})^{-1} (\tilde{b}_{t,s_1} - &\hat{b}_{t,s_1}) \|_1 = \|(\hat{U}_{s_t}^\top L_{s_t})^{-1}  (\tilde{B}_{(o,a, s)_{t-1}} - \hat{B}_{(o,a, s)_{t-1}}) (\hat{U}_{s_{t-1}}^\top L_{s_{t-1}})\|_1 \| (\hat{U}_{s_{t-1}}^\top L_{s_{t-1}})^{-1} \tilde{b}_{t-1,s_1} \|_1 \\
        &+ \|(\hat{U}_{s_t}^\top L_{s_t})^{-1}  (\tilde{B}_{(o,a, s)_{t-1}} - \hat{B}_{(o,a, s)_{t-1}}) (\hat{U}_{s_{t-1}}^\top L_{s_{t-1}})\|_1 \| (\hat{U}_{s_{t-1}}^\top L_{s_{t-1}})^{-1} (\tilde{b}_{t-1,s_1} - \hat{b}_{t-1, s_1}) \|_1 \\
        &+ \|(\hat{U}_{s_t}^\top L_{s_t})^{-1}  \tilde{B}_{(o,a, s)_{t-1}} (\hat{U}_{s_{t-1}}^\top L_{s_{t-1}}) (\hat{U}_{s_{t-1}}^\top L_{s_{t-1}})^{-1} (\tilde{b}_{t-1,s_1} - \hat{b}_{t-1, s_1}) \|_1.
    \end{align*}
    For the first term, 
    \begin{align*}
        \sum_{o_{t-1}} &\sum_{(s,a,o)_{1:t-2}} \pi(a_{1:t-2} | h_{t-2}) \|(\hat{U}_{s_t}^\top L_{s_t})^{-1}  (\tilde{B}_{(o,a, s)_{t-1}} - \hat{B}_{(o,a, s)_{t-1}}) (\hat{U}_{s_{t-1}}^\top L_{s_{t-1}})\|_1 \| (\hat{U}_{s_{t-1}}^\top L_{s_{t-1}})^{-1} \tilde{b}_{t-1,s_1} \|_1 \\
        &= \sum_{o_{t-1}} \sum_{(s,a,o)_{1:t-2}} \pi(a_{1:t-2} | h_{t-2}) \Delta_{(o,a,s)_{t-1}} \| (\hat{U}_{s_{t-1}}^\top L_{s_{t-1}})^{-1} \tilde{b}_{t-1,s_1} \|_1 \\
        &\le \Delta \sum_{(s,a,o)_{1:t-2}}  \pi(a_{1:t-2} | h_{t-2}) \| (\hat{U}_{s_{t-1}}^\top L_{s_{t-1}})^{-1} \tilde{b}_{t-1,s_1} \|_1 \\
        &= \Delta \sum_{(s,a,o)_{1:t-2}} \pi(a_{1:t-2} | h_{t-2}) \PP(h_{t-2}) = \Delta,
    \end{align*}
    where we used the definition of $\tilde{b}_{t-1,s_1} = \hat{U}_{s_{t-1}}^\top L_{s_{t-1}} \PP((s,a,o)_{1:t-2})$. For the second term, by the induction hypothesis,
    \begin{align*}
        \sum_{o_{t-1}} &\sum_{(s,a,o)_{1:t-2}} \pi(a_{1:t-2} | h_{t-2}) \Delta_{(o,a,s)_{t-1}} \| (\hat{U}_{s_{t-1}}^\top L_{s_{t-1}})^{-1} (\tilde{b}_{t-1,s_1} - \hat{b}_{t-1,s_1}) \|_1 \\
        &\le \Delta \sum_{(s,a,o)_{1:t-2}} \pi(a_{1:t-2} | h_{t-2}) \| (\hat{U}_{s_{t-1}}^\top L_{s_{t-1}})^{-1} (\tilde{b}_{t-1,s_1} - \hat{b}_{t-1,s_1}) \|_1 \\
        &= \Delta ((1+\Delta)^{t-1} \delta_1 + (1 + \Delta)^{t-1} - 1).
    \end{align*}
    The last term can also be derived following the same argument in \cite{hsu2012spectral}. It gives
    \begin{align*}
        \sum_{o_{t-1}} &\sum_{(s,a,o)_{1:t-2}}  \pi(a_{1:t-2} | h_{t-2}) \|(\hat{U}_{s_t}^\top L_{s_t})^{-1}  \tilde{B}_{(o,a, s)_{t-1}} (\hat{U}_{s_{t-1}}^\top L_{s_{t-1}}) (\hat{U}_{s_{t-1}}^\top L_{s_{t-1}})^{-1} (\tilde{b}_{t-1,s_1} - \hat{b}_{t-1, s_1}) \|_1 \\
        &\le \sum_{o_{t-1}} \sum_{(s,a,o)_{1:t-2}}  \pi(a_{1:t-2} | h_{t-2}) \|D_{(o,a, s)_{t-1}} (\hat{U}_{s_{t-1}}^\top L_{s_{t-1}})^{-1} (\tilde{b}_{t-1,s_1} - \hat{b}_{t-1, s_1}) \|_1 \\
        &\le \sum_{o_{t-1}} \sum_{(s,a,o)_{1:t-2}}  \pi(a_{1:t-2} | h_{t-2}) \|D_{(o,a,s)_{t-1}}\|_1 \|(\hat{U}_{s_{t-1}}^\top L_{s_{t-1}})^{-1} (\tilde{b}_{t-1,s_1} - \hat{b}_{t-1, s_1}) \|_1 \\
        &\le (1+\Delta)^{t-1} \delta_1 + (1 + \Delta)^{t-1} - 1.
    \end{align*}
    Now combining these three bounds, we get
    \begin{align*}
        \sum_{a_{t-1}} &\pi(a_{t-1} | h_{1:t-1}) \sum_{o_{t-1}} \sum_{(a,r,s')_{1:t-2}}  \pi(a_{1:t-2} | h_{1:t-2}) \|(\hat{U}_{s_t}^\top L_{s_t})^{-1}  (\tilde{b}_{t,s_1} -   \hat{b}_{t,s_1}) \|_1 \\
        &\le \sum_{a_{t-1}} \pi(a_{t-1} | h_{1:t-1}) (\Delta + (1 + \Delta) ((1+\Delta)^{t-1} \delta_1 + (1 + \Delta)^{t-1} - 1)) \\
        &\le \sum_{a_{t-1}} \pi(a_{t-1} | h_{1:t-1}) ((1+\Delta)^{t} \delta_1 + (1 + \Delta)^{t} - 1) = (1+\Delta)^{t} \delta_1 + (1 + \Delta)^{t} - 1. 
    \end{align*}
\end{proof}

\subsection{Proof of Lemma \ref{lemma:psr_finite_sample_error}}
\begin{proof}
    For the first inequality, we note that
    \begin{align*}
        \|P_{\mH_s} - \hat{P}_{\mH_s} \|_2 &\le p_s \|P_{\mH|s} - \hat{P}_{\mH|s} \|_2 + |p_s - \hat{p}_s| \| P_{\mH|s} \|_2.
    \end{align*}
    Let $N_s = \hat{p}_s N$. For the first term, we can use McDiarmid's inequality since a change at one sample among $N_s$ samples (conditioned on starting test from $s$) causes only $\sqrt{2/N_s}$ difference:
    \begin{align*}
        \|P_{\mH|s} - \hat{P}_{\mH|s}\|_2 - \Exs[\|P_{\mH|s} - \hat{P}_{\mH|s}\|_2]  \lesssim \sqrt{\frac{1}{N_s} \ln(1/\eta)},
    \end{align*}
    with probability at least $1 - \eta / 100$. Let $\#(h_{s,i})$ be a count of a history $h_{s,i}$ after seeing $N_s$ histories that end with $s$. Also,
    \begin{align*}
        \Exs[\|P_{\mH|s} - \hat{P}_{\mH|s}\|_2] &\le \sqrt{\Exs[\|P_{\mH|s} - \hat{P}_{\mH|s}\|_2^2}] \le \sqrt{\sum_{i=1}^{|H_s|} Var \left(\frac{1}{N_s} \#(h_{s,i}) \Big| s \right)} \\
        &\le \sqrt{\frac{1}{N_s} \sum_{i=1}^{|H_s|} \PP(h_{s,i} | s) } \le \sqrt{\frac{1}{N_s}}.
    \end{align*}
    Therefore, we can conclude that $\|P_{\mH|s} - \hat{P}_{\mH|s}\|_2 \lesssim \sqrt{\frac{1}{N_s} \ln(1/\eta)}$. On the other hand, we can show that $|p_s - \hat{p}_s| \lesssim \sqrt{\frac{p_s}{N} \log(1/\eta)} + \frac{\log(1/\eta)}{N}$ via a simple application of Bernstein's inequality. Note that our sample complexity guarantees $N \gg \log(1/\eta)/p_s$. Hence,
    \begin{align*}
        \|P_{\mH_s} - \hat{P}_{\mH_s} \|_2 &\le p_s \|P_{\mH|s} - \hat{P}_{\mH|s} \|_2 + |p_s - \hat{p}_s| \| P_{\mH|s} \|_2 \\
        &\lesssim p_s \sqrt{\frac{1}{N_s} \log(1/\eta)} + \sqrt{\frac{p_s}{N}  \log(1/\eta)} \| P_{\mH|s} \|_2 \lesssim \sqrt{\frac{p_s}{N} \log(1/\eta)}.
    \end{align*}
    
    Similarly, we can show that
    \begin{align*}
        \Exs[\|P_{\mT, \mH|s} - \hat{P}_{\mT, \mH|s}\|_2] &\le \sqrt{\Exs[\|P_{\mT, \mH|s} - \hat{P}_{\mT, \mH|s}\|_F^2}] \le \sqrt{\sum_{j \in [\mT], i \in \left[|\mH_s| \right]} Var \left(\frac{A^l}{N_s} \#(\tau_j, h_{s,i}) \Big| s \right)} \\
        &\le A^l \sqrt{\frac{1}{N_s} \sum_{j \in [\mT], i \in \left[|\mH_s| \right]} \PP(\tau_j, h_{s,i} | s) } \le A^l \sqrt{\frac{1}{N_s}}.
    \end{align*}
    Following the same argument with McDiarmid's inequality, we get the second inequality. The remaining inequalities can be shown through similar arguments. Taking over union bounds over all $s, a, o$ gives the Lemma. 
\end{proof}

\subsection{Proof of Lemma \ref{lemma:psr_cond_helper_sum}}
\begin{proof}
    As in the proof in Theorem \ref{theorem:psr_main_tv_bound}, let us denote $\tilde{b}_{t,s_1} = \tilde{B}_{(o,a,s)_{t-1:1}} \tilde{b}_{1,s_1}$. Then, we can decompose the terms as before:
    \begin{align*}
        &\sum_{r_{t},s_{t+1}, ..., s_{t+l}} | b_{\infty, s_{t+l}}^\top B_{(o,a,s)_{t+l-1:1}} b_{1, s_1} - \hat{b}_{\infty, s_{t+l}}^\top \hat{B}_{(o,a,s)_{t+l-1:1}} \hat{b}_{1, s_1} | \\
        &= \sum_{r_{t}, ..., s_{t+l}} | \tilde{b}_{\infty, s_{t+l}}^\top \tilde{B}_{(o,a,s)_{t+l-1:t}} \tilde{b}_{t, s_1} - \hat{b}_{\infty, s_{t+l}}^\top \hat{B}_{(o,a,s)_{t+l-1:t}} \hat{b}_{t, s_1} | \\
        &= \sum_{r_{t}, ..., s_{t+l}} | \tilde{b}_{\infty, s_{t+l}}^\top \tilde{B}_{(o,a,s)_{t+l-1:t}} (\hat{U}_{s_t}^\top L_{s_t}) (\hat{U}_{s_t}^\top L_{s_t})^{-1} \tilde{b}_{t, s_1} - \hat{b}_{\infty, s_{t+l}}^\top \hat{B}_{(o,a,s)_{t+l-1:t}} (\hat{U}_{s_t}^\top L_{s_t}) (\hat{U}_{s_t}^\top L_{s_t})^{-1} \hat{b}_{t, s_1} | \\
        &\le \sum_{r_{t}, ..., s_{t+l}} \| (\tilde{b}_{\infty, s_{t+l}}^\top \tilde{B}_{(o,a,s)_{t+l-1:t}} - \hat{b}_{\infty, s_{t+l}} \hat{B}_{(o,a,s)_{t+l-1:t}}) (\hat{U}_{s_t}^\top L_{s_t})\|_1 \|(\hat{U}_{s_t}^\top L_{s_t})^{-1} \tilde{b}_{t,s_1} \|_1 \\
        &\ + \sum_{r_{t}, ..., s_{t+l}} \| (\tilde{b}_{\infty, s_{t+l}}^\top \tilde{B}_{(o,a,s)_{t+l-1:t}} - \hat{b}_{\infty, s_{t+l}} \hat{B}_{(o,a,s)_{t+l-1:t}}) (\hat{U}_{s_t}^\top L_{s_t}) \|_1 \| (\hat{U}_{s_t}^\top L_{s_t})^{-1} (\tilde{b}_{t,s_1} - \hat{b}_{t,s_1}) \|_1  \\
        &\ + \sum_{r_{t}, ..., s_{t+l}} \| \tilde{b}_{\infty, s_{t+l}}^\top \tilde{B}_{(o,a,s)_{t+l-1:t}} (\hat{U}_{s_t}^\top L_{s_t}) \|_1 \| (\hat{U}_{s_t}^\top L_{s_t})^{-1} (\tilde{b}_{t,s_1} - \hat{b}_{t,s_1}) \|_1.
    \end{align*}
    We follow the same induction procedure starting with showing the following equation:
    \begin{align*}
        \sum_{s_{t}, r_{t}, ..., s_{t+l-1}} \| (\hat{U}_{s_{t+l}}^\top L_{s_{t+l}})^{-1} (\tilde{B}_{(o,a,s)_{t+l-1:t}} - \hat{B}_{(o,a,s)_{t+l-1:t}}) (\hat{U}_{s_t}^\top L_{s_t})\|_1 \le (1+\Delta)^l - 1.
    \end{align*}
    We show this equation by induction as in the previous proof. If $l = 1$, then
    $$\|(\hat{U}_{s_{t+1}} L_{s_{t+1}})^{-1} (\tilde{B}_{(o,a,s)_{t}} - \hat{B}_{(o,a,s)_{t}}) (\hat{U}_{s_t}^\top L_{s_t})\|_1 \le \sum_{o_t} \Delta_{o_t, a_t, s_t} \le \Delta,$$
    by the definition of $\Delta$. Now we assume it holds for sequences of length less than $l$, and prove the induction hypothesis for $l$. We again split the term into three terms:
    \begin{align*}
        \sum_{r_{t}, ..., s_{t+l}} &\| (\hat{U}_{s_{t+l}}^\top L_{s_{t+l}})^{-1} (\tilde{B}_{(o,a,s)_{t+l-1:t}} - \hat{B}_{(o,a,s)_{t+l-1:t}}) (\hat{U}_{s_{t+l-1}}^\top L_{s_{t+l-1}})\|_1 \\
        &\le \sum_{s_{t}, r_{t}, ..., s_{t+l}} \| (\hat{U}_{s_{t+l}}^\top L_{s_{t+l}})^{-1} (\tilde{B}_{(o,a,s)_{t+l-1}} - \hat{B}_{(o,a,s)_{t+l-1}}) (\hat{U}_{s_{t+l-1}}^\top L_{s_{t+l-1}})\|_1 \times \\ &\qquad \qquad \qquad \qquad \|(\hat{U}_{s_{t+l-1}}^\top L_{s_{t+l-1}})^{-1} \tilde{B}_{(o,a,s)_{t+l-2:t}} (\hat{U}_{s_{t}}^\top L_{s_{t}})\|_1 \\
        &+ \sum_{r_{t}, ..., s_{t+l}} \|(\hat{U}_{s_{t+l}}^\top L_{s_{t+l}})^{-1} (\tilde{B}_{(o,a,s)_{t+l-1}} - \hat{B}_{(o,a,s)_{t+l-1}})  (\hat{U}_{s_{t+l-1}}^\top L_{s_{t+l-1}})\|_1 \times \\
        &\qquad \qquad \qquad \qquad \|(\hat{U}_{s_{t+l-1}}^\top L_{s_{t+l-1}})^{-1} (\tilde{B}_{(o,a,s)_{t+l-2:t}} - \hat{B}_{(o,a,s)_{t+l-2:t}}) (\hat{U}_{s_{t}}^\top L_{s_{t}})\|_1 \\
        &+\sum_{r_{t}, ..., s_{t+l}} \| (\hat{U}_{s_{t+l}}^\top L_{s_{t+l}})^{-1} \tilde{B}_{(o,a,s)_{t+l-1}} (\hat{U}_{s_{t+l-1}}^\top L_{s_{t+l-1}})\|_1 \times \\
        &\qquad \qquad \qquad \qquad \|(\hat{U}_{s_{t+l-1}}^\top L_{s_{t+l-1}})^{-1} (\tilde{B}_{(o,a,s_{t+l-2:t}} - \hat{B}_{(o,a,s_{t+l-2:t}}) (\hat{U}_{s_{t}}^\top L_{s_{t}})\|_1 \\
        &\le \Delta + \Delta \sum_{r_{t}, ..., s_{t+l-1}} \|(\hat{U}_{s_{t+l}}^\top L_{s_{t+l}})^{-1} (\tilde{B}_{(o,a,s)_{t+l-1:t}} - \hat{B}_{(o,a,s)_{t+l-1:t}}) (\hat{U}_{s_{t}}^\top L_{s_{t}})\|_1 \\
        &+ \sum_{r_{t}, ..., s_{t+l-1}} \|(\hat{U}_{s_{t+l}}^\top L_{s_{t+l}})^{-1} (\tilde{B}_{(o,a,s)_{t+l-1:t}} - \hat{B}_{(o,a,s)_{t+l-1:t}}) (\hat{U}_{s_{t}}^\top L_{s_{t}})\|_1,
    \end{align*}
    where in the last step, we used $\sum_{r_{t}, ..., s_{t+l}} \|(\hat{U}_{s_{t+l}}^\top L_{s_{t+l}})^{-1} \tilde{B}_{(o,a,s)_{t+l-1:t}} (\hat{U}_{s_{t}}^\top L_{s_{t}})\|_1 = 1$ as well as $$\sum_{r_{t+l-1}, s_{t+l}} \|(\hat{U}_{s_{t+l}}^\top L_{s_{t+l}})^{-1} \tilde{B}_{(o,a,s)_{t+l-1}} (\hat{U}_{s_{t}}^\top L_{s_{t}})\|_1 = 1.$$ By induction hypothesis, it is bounded by
    \begin{align*}
        \Delta + \Delta ((1 + \Delta)^{l-1} - 1) + (1 + \Delta)^{l-1} - 1 = (1+\Delta)^l - 1.
    \end{align*}
    
    With the Lemma, we can verify that 
    \begin{align*}
        \sum_{r_{t}, ..., s_{t+l}} &\| (\tilde{b}_{s_{t+l}}^\top \tilde{B}_{(o,a,s)_{t+l-1:t}} - \hat{b}_{s_{t+l}}^\top \hat{B}_{(o,a,s)_{t+l-1:t}}) (\hat{U}_{s_t}^\top L_{s_t})\|_1 \\
        &\le \sum_{r_{t}, ..., s_{t+l}} \| (\tilde{b}_{s_{t+l}} - \hat{b}_{s_{t+l}}) (\hat{U}_{s_{t+l}}^\top  L_{s_{t+l}}) \|_{\infty} \| (\hat{U}_{s_{t+l}}^\top  L_{s_{t+l}})^{-1} (\tilde{B}_{(o,a,s)_{t+l-1:t}} -  \hat{B}_{(o,a,s)_{t+l-1:t}}) (\hat{U}_{s_t}^\top L_{s_t})\|_1 \\
        &+ \sum_{r_{t}, ..., s_{t+l}} \| \tilde{b}_{s_{t+l}} (\hat{U}_{s_{t+l}}^\top L_{s_{t+l}}) \|_{\infty} \| (\hat{U}_{s_{t+l}}^\top L_{s_{t+l}})^{-1} (\tilde{B}_{(o,a,s)_{t+l-1:t}} -  \hat{B}_{(o,a,s)_{t+l-1:t}}) (\hat{U}_{s_t}^\top L_{s_t})\|_1 \\
        &+ \sum_{r_{t}, ..., s_{t+l}} \| (\tilde{b}_{s_{t+l}} - \hat{b}_{s_{t+l}}) (\hat{U}_{s_{t+l}}^\top  L_{s_{t+l}}) \|_{\infty} \| (\hat{U}_{s_{t+l}}^\top  L_{s_{t+l}})^{-1} \tilde{B}_{(o,a,s)_{t+l-1:t}} (\hat{U}_{s_t}^\top L_{s_t})\|_1 \\
        &\le \delta_{\infty} + (\delta_{\infty} + 1) ((1+\Delta)^l - 1).
    \end{align*}
    Let $\epsilon_l := \delta_{\infty} + (\delta_{\infty} + 1) ((1+\Delta)^l - 1)$. From the above, we can conclude that
    \begin{align*}
        \sum_{a_t, r_{t}, ..., s_{t+l}} & \pi(a_{1:t} | h_t) | b_{\infty, s_{t+l}}^\top B_{(o,a,s)_{t+l:1}} b_{1, s_1} - \hat{b}_{\infty, s_{t+l}}^\top \hat{B}_{(o,a,s)_{t+l:1}} \hat{b}_{1, s_1} | \\
        &= \sum_{a_t, r_{t}, ..., s_{t+l}} | \pi(a_{1:t-1} | h_{t-1}) b_{\infty, s_{t+l}}^\top B_{(o,a,s)_{t+l-1:t}} \tilde{b}_{t, s_1} - \hat{b}_{\infty, s_{t+l}}^\top \hat{B}_{(o,a,s)_{t+l-1:t}} \hat{b}_{t, s_1} | \\
        &\le \epsilon_l \pi(a_{1:t-1} | h_{t-1}) \|(\hat{U}_{s_t}^\top L_{s_t})^{-1} \tilde{b}_{t,s_1} \|_1 + \epsilon_l \pi(a_{1:t-1} | h_{t-1}) \| (\hat{U}_{s_t}^\top L_{s_t})^{-1} (\tilde{b}_{t,s_1} - \hat{b}_{t,s_1}) \|_1  \\
        &+ \sum_{a_t, r_{t}, ..., s_{t+l}} \pi(a_{1:t-1} | h_{t-1}) \| \tilde{b}_{\infty, s_{t+l}}^\top \tilde{B}_{(o,a,s)_{t+l-1:t}} (\hat{U}_{s_t}^\top L_{s_t}) \|_1 \| (\hat{U}_{s_t}^\top L_{s_t})^{-1} (\tilde{b}_{t,s_1} - \hat{b}_{t,s_1}) \|_1 \\
        &\le \epsilon_l \PP^\pi ((s,a,o)_{1:t-1}) +  (1 + \epsilon_l) \pi(a_{1:t-1} | h_{t-1}) \| (\hat{U}_{s_t}^\top L_{s_t})^{-1} (\tilde{b}_{t,s_1} - \hat{b}_{t,s_1}) \|_1,
    \end{align*}
    where we used 
    \begin{align*}
        \pi(a_{1:t-1} | h_{t-1}) \|(\hat{U}_{s_t}^\top L_{s_t})^{-1} \tilde{b}_{t,s_1} \|_1 &= \pi(a_{1:t-1} | h_{t-1}) 1^\top (\hat{U}_{s_t}^\top L_{s_t})^{-1} \tilde{b}_{t,s_1} = \PP((o,a,s_{1:t} |s_1), 
    \end{align*}
    and 
    \begin{align*}
        \sum_{r_{t}, ..., s_{t+l}} \| \tilde{b}_{\infty, s_{t+l}}^\top \tilde{B}_{(o,a,s)_{t+l-1:t}} (\hat{U}_{s_t}^\top L_{s_t}) \|_1 = \sum_{s_{t}, r_{t}, ..., s_{t+l}} \| 1^\top \tilde{D}_{(o,a,s)_{t+l-1:t}} \|_1 = 1.
    \end{align*}
    Since $\epsilon_l < 1$, we get the Lemma.
\end{proof}

\subsection{Proof of Lemma \ref{lemma:psr_cond_good_event}}
\begin{proof}
    Note that from equation \eqref{eq:l1_sum_UL_inv} in Lemma \ref{lemma:appendix_psr_sum_main}, we have
    \begin{align*}
        \sum_{s_1, a_1, r_1, ..., r_{t-1}, s_t} \pi(a_{1:t-1} | h_{t-1}) \|(\hat{U}_{s_t}^\top L_{s_t})^{-1} (\tilde{b}_{t,s_1} - \hat{b}_{t,s_1})\|_1 \le \epsilon_t. 
    \end{align*}
    Let $\Eps_b$ be a bad event where for a sampled trajectory $s_1, a_1, r_1, ..., r_{t-1}, s_t$, the difference in estimated probability is larger than $\epsilon_c \PP^\pi (s_1, a_1, r_1, ..., r_{t-1}, s_t)$, i.e., 
    \begin{align*}
       \pi(a_{1:t-1} | h_{t-1}) \|(\hat{U}_{s_t}^\top L_{s_t})^{-1} (\tilde{b}_{t,s_1} - \hat{b}_{t,s_1})\|_1 \ge \epsilon_c \PP^\pi (s_1, a_1, r_1, ..., r_{t-1}, s_t).
    \end{align*}
    Note that $\PP^\pi (s_1, a_1, r_1, ..., r_{t-1}, s_t) = \pi(a_{1:t-1} | h_{t-1}) 1^\top (\hat{U}_{s_t}^\top L_{s_t})^{-1} \tilde{b}_{t,s_1} = \pi(a_{1:t-1} | h_{t-1}) \|(\hat{U}_{s_t}^\top L_{s_t})^{-1} \tilde{b}_{t,s_1}\|_1$. If $\PP^\pi (\Eps_b) > \epsilon_t / \epsilon_c$, then 
    \begin{align*}
        \sum_{(s,a,o)_{1:t-1}} \pi(a_{1:t-1} | h_{t-1}) \|(\hat{U}_{s_t}^\top L_{s_t})^{-1} (\tilde{b}_{t,s_1} - \hat{b}_{t,s_1})\|_1 &\ge \sum_{(s, a, o)_{1:t-1} \in \Eps_b} \pi(a_{1:t-1} | h_{t-1}) \|(\hat{U}_{s_t}^\top L_{s_t})^{-1} (\tilde{b}_{t,s_1} - \hat{b}_{t,s_1})\|_1 \\
        &\ge \sum_{(s, a, o)_{1:t-1} \in \Eps_b} \epsilon_c \pi(a_{1:t-1} | h_{t-1}) \|(\hat{U}_{s_t}^\top L_{s_t})^{-1} \tilde{b}_{t,s_1}\|_1 \\
        &\ge \epsilon_c \PP^\pi (\Eps_b) > \epsilon_t,
    \end{align*}
    which is a contradiction. Similarly, by Theorem \ref{theorem:psr_main_tv_bound}, we have
    \begin{align*}
        \sum_{(s,a,r)_{1:t-1}, s_t} |\PP^{\pi} ((s,a,r)_{1:t-1}, s_t) - \hat{\PP}^{\pi} ((s,a,r)_{1:t-1}, s_t) | \le \epsilon_t. 
    \end{align*}
    Following the same argument, we can show the contradiction if the Lemma \ref{lemma:psr_cond_good_event} does not hold. 
\end{proof}

\subsection{Auxiliary Lemmas for Spectral Learning}
For completeness of the paper, we include the following lemmas from Appendix B in \cite{hsu2012spectral} and \cite{stewart1990matrix}.
\begin{lemma}[Theorem 4.1 in \cite{stewart1990matrix}]
    \label{lemma:singular_perturb}
    Let $A \in \mathbb{R}^{m\times n}$ with $m \ge n$ and $\hat{A} = A + E$ for some $E \in \mathbb{R}^{m \times n}$. If singular values of $A$ and $\hat{A}$ are $\sigma_1 \ge \sigma_2 \ge ... \ge \sigma_n$ and $\hat{\sigma}_1 \ge \hat{\sigma}_2 \ge ... \ge \hat{\sigma}_n$ respectively, then
    \begin{align*}
        |\sigma_i - \hat{\sigma_i}| \le \|E\|_2, \qquad \forall i \in [n].
    \end{align*}
\end{lemma}
The following lemma is also called the Davis-Kahn's Sin($\Theta$) theorem.
\begin{lemma}[Theorem 4.4 in \cite{stewart1990matrix}]
    \label{lemma:sin_theta_theorem}
    Let $A \in \mathbb{R}^{m\times n}$ with $m \ge n$ with singular value decomposition (SVD) $(U_1, U_2, U_3, \Sigma_1, \Sigma_2, V_1, V_2)$ such that
    \begin{align*}
        A = \begin{bmatrix} U_1 & U_2 & U_3 \end{bmatrix} \begin{bmatrix} \Sigma_1 & 0 \\ 0 & \Sigma_2 \\ 0 & 0 \end{bmatrix} \begin{bmatrix} V_1^\top  \\ V_2^\top  \end{bmatrix}.
    \end{align*}
    Similarly, $\hat{A} = A + E$ for some $E \in \mathbb{R}^{m \times n}$ has a SVD $(\hat{U}_1, \hat{U}_2, \hat{U}_3, \hat{\Sigma}_1, \hat{\Sigma}_2, \hat{V}_1, \hat{V}_2)$. Let $\Phi$ be the matrix of canonical angles between $range(U_1)$ and $range(\hat{U}_1)$, and $\Theta$ be the matrix of canonical angles between $range(V_1)$ and $range(\hat{V}_1)$. If there exists $\alpha, \delta > 0$ such that $\sigma_{min}(\Sigma_1) \ge \alpha + \delta$ and $\sigma_{max} (\Sigma_2) < \alpha$, then
    \begin{align*}
        \max\{ \| \sin \Phi \|_2, \| \sin \Theta \|_2\} \le \|E\|_2 / \delta.
    \end{align*}
\end{lemma}

\begin{lemma}[Corollary 22 in \cite{hsu2012spectral}]
    \label{lemma:corollary22}
    Suppose $A \in \mathbb{R}^{m\times n}$ with rank $k \le n$ and $m \ge n$, and $\hat{A} = A + E$ with $E \in \mathbb{R}^{m \times n}$. Let $\sigma_k(A)$ be the $k^{th}$ singular value of $A$ and assume $\|E\|_2 \le \epsilon \cdot \sigma_k(A)$ for some small $\epsilon < 1$. Let $\hat{U}$ be top-$k$ left singular vectors of $\hat{A}$, and $\hat{U}_{\perp}$ be the remaining left singular vectors. Then,
    \begin{itemize}
        \item $\sigma_k (\hat{A}) \ge (1 - \epsilon) \sigma_k(A)$.
        \item $\| \hat{U}_{\perp}^\top U\|_2 \le \|E\|_2 / \sigma_k (\hat{A})$.
    \end{itemize}
\end{lemma}
\begin{proof}
    The first inequality follows from Lemma \ref{lemma:singular_perturb}, and the second inequality follows from Lemma \ref{lemma:sin_theta_theorem} by plugging $\alpha = 0$ and $\delta = \hat{\sigma}_k$. 
\end{proof}

\begin{lemma}[Lemma 9 in \cite{hsu2012spectral}]
    \label{lemma:theorem_singular_value}
    Suppose $A \in \mathbb{R}^{m\times n}$ with rank $k \le n$, and $\hat{A} = A + E$ with $E \in \mathbb{R}^{m \times n}$ for $\|E\|_2 \le \epsilon \cdot \sigma_k(A)$ with small $\epsilon < 1$. Let $\epsilon_0 = \epsilon^2 / (1 - \epsilon)^2$ and $\hat{U}$ be top-$k$ left singular vectors of $\hat{A}$. Then,
    \begin{itemize}
        \item $\sigma_k(\hat{U}^\top \hat{A}) \ge (1 - \epsilon) \cdot \sigma_k(A)$.
        \item $\sigma_k(\hat{U}^\top A) \ge \sqrt{1 - \epsilon_0} \cdot \sigma_k(A)$.
    \end{itemize}
\end{lemma}
\begin{proof}
    The first item is immediate since $\sigma_k(\hat{U}^\top \hat{A}) = \sigma_k(\hat{A})$. Let $U$ be top-$k$ left singular vectors of $A$. If the top-$k$ SVD of $A$ is $A = U \Sigma V^\top$, then 
    \begin{align*}
        \sigma_k (\hat{U}^\top U \Sigma V^\top) \ge \sigma_{min} (\hat{U}^\top U) \cdot \sigma_k (\Sigma) \ge \sqrt{1 - \|\hat{U}_{\perp}^\top U\|_2^2} \cdot \sigma_k (A) \ge \sqrt{1 - \epsilon_0} \cdot \sigma_k(A),
    \end{align*}
    where the first inequality holds since $V$ is orthonormal and $\hat{U}^\top U$ is full-rank, and the final inequality follows from Lemma \ref{lemma:corollary22}. 
\end{proof}

\begin{lemma}[Theorem 3.8 in \cite{stewart1990matrix}]
    \label{lemma:theorem_matrix_perturb}
    Let $A \in \mathbb{R}^{m\times n}$ with $m \ge n$, and let $\tilde{A} = A + E$ with $E \in \mathbb{R}^{m \times n}$. Then,
    \begin{align*}
        \| \tilde{A}^{\dagger} - A^{\dagger} \|_2 \le \frac{1 + \sqrt{5}}{2} \cdot \max \{\|A^{\dagger}\|_2^2, \| \tilde{A}^{\dagger}\|_2^2 \} \cdot \|E\|_2.
    \end{align*}
\end{lemma}

\end{appendices}

\end{document}